\documentclass[10pt,twocolumn,letterpaper]{article}

\usepackage{cvpr}
\usepackage{times}
\usepackage{epsfig}
\usepackage{graphicx}
\usepackage{amsmath}
\usepackage{amssymb}
\usepackage{multirow}
\usepackage{makecell}
\usepackage{amsthm}
\usepackage[vlined,ruled,linesnumbered]{algorithm2e}
\usepackage[noend]{algorithmic}
\usepackage{capt-of}
\usepackage{adjustbox}

\usepackage[breaklinks=true,bookmarks=false,pagebackref=true,colorlinks]{hyperref}

\cvprfinalcopy 

\def\cvprPaperID{4382} 

\ifcvprfinal\fi

\graphicspath{{./figures/}}

\usepackage{amsfonts}
\usepackage{amsmath}
\usepackage{amssymb}

\usepackage{graphicx,url}

\usepackage{xspace}
\usepackage{bm}
\usepackage[T1]{fontenc}
\usepackage{latexsym}
\usepackage{xstring}
\usepackage{relsize}

\usepackage[noend]{algorithmic}
\usepackage{multirow}
\usepackage{xcolor}

\newtheorem{theorem}{Theorem}
\newtheorem{problem}{Problem}

\newtheorem{definition}[theorem]{Definition}
\newtheorem{proposition}[theorem]{Proposition}
\newtheorem{remark}[theorem]{Remark}
\newtheorem{example}[theorem]{Example}

\renewcommand{\cf}{\emph{cf.}\xspace}

\newcommand{\bdmath}{\begin{dmath}}
\newcommand{\edmath}{\end{dmath}}
\newcommand{\beq}{\begin{equation}}
\newcommand{\eeq}{\end{equation}}
\newcommand{\bdm}{\begin{displaymath}}
\newcommand{\edm}{\end{displaymath}}
\newcommand{\bea}{\begin{eqnarray}}
\newcommand{\eea}{\end{eqnarray}}
\newcommand{\beal}{\beq \begin{array}{ll}}
\newcommand{\eeal}{\end{array} \eeq}
\newcommand{\beas}{\begin{eqnarray*}}
\newcommand{\eeas}{\end{eqnarray*}}
\newcommand{\ba}{\begin{array}}
\newcommand{\ea}{\end{array}}
\newcommand{\bit}{\begin{itemize}}
\newcommand{\eit}{\end{itemize}}
\newcommand{\ben}{\begin{enumerate}}
\newcommand{\een}{\end{enumerate}}

\newcommand{\calB}{{\cal B}}
\newcommand{\calC}{{\cal C}}

\newcommand{\calF}{{\cal F}}

\newcommand{\calL}{{\cal L}}

\newcommand{\calS}{{\cal S}}

\newcommand{\calX}{{\cal X}}

\renewcommand{\etal}{\emph{et~al.}\xspace}

\newcommand{\M}[1]{{\bm #1}} 

\newcommand{\hide}[1]{}
\renewcommand{\wrt}{w.r.t.\xspace}

\newcommand{\hiddenText}{{\color{gray} hidden text.}}
\newcommand{\hideWithText}[1]{\hiddenText}

\newcommand{\dist}{\mathbf{dist}}

\newcommand{\subject}{\text{ subject to }}

\DeclareMathOperator*{\argmin}{arg\,min}

\newcommand{\norm}[1]{\left\| #1 \right\|}

\newcommand{\tran}{^{\mathsf{T}}}

\newcommand{\inv}{^{-1}}

\newcommand{\Real}[1]{ { {\mathbb R}^{#1} } }

\newcommand{\SOthree}{\ensuremath{\mathrm{SO}(3)}\xspace}

\newcommand{\ME}{\M{E}}

\newcommand{\MK}{\M{K}}

\newcommand{\MR}{\M{R}}

\newcommand{\MF}{\M{F}}

\newcommand{\va}{\boldsymbol{a}} 
 
\newcommand{\vb}{\boldsymbol{b}}

\newcommand{\vp}{\boldsymbol{p}}
\newcommand{\vq}{\boldsymbol{q}}

\newcommand{\vt}{\boldsymbol{t}}
\newcommand{\vxx}{\boldsymbol{x}}

\newcommand{\vtheta}{\boldsymbol{\theta}}

\newcommand{\scenario}[1]{{\smaller \sf#1}\xspace}

\newcommand{\blue}[1]{{\color{blue}#1}}

\newcommand{\linkToPdf}[1]{\href{#1}{\blue{(pdf)}}}
\newcommand{\linkToPpt}[1]{\href{#1}{\blue{(ppt)}}}
\newcommand{\linkToCode}[1]{\href{#1}{\blue{(code)}}}
\newcommand{\linkToWeb}[1]{\href{#1}{\blue{(web)}}}
\newcommand{\linkToVideo}[1]{\href{#1}{\blue{(video)}}}
\newcommand{\award}[1]{\xspace} 

\newcommand{\aka}{\emph{a.k.a.}}

\newcommand{\gray}[1]{{\color{gray}#1}}

\newcommand{\teaser}{\scenario{TEASER++}}
\newcommand{\gnc}{\scenario{GNC}}
\newcommand{\fgr}{\scenario{FGR}}
\newcommand{\magsac}{\scenario{MAGSAC}}
\newcommand{\gcransac}{\scenario{GCRANSAC}}

\newcommand{\nameshort}{\scenario{SGP}}

\newcommand{\topic}{Geometric Perception\xspace}
\newcommand{\topiclc}{Geometric perception\xspace}
\newcommand{\topicllc}{geometric perception\xspace}

\newcommand{\threedmatch}{\scenario{3DMatch}}
\newcommand{\megadepth}{\scenario{MegaDepth}}
\newcommand{\fpfh}{\scenario{FPFH}}

\newcommand{\fcgf}{\scenario{FCGF}}
\newcommand{\sfcgf}{\scenario{S-FCGF}}
\newcommand{\caps}{\scenario{CAPS}}
\newcommand{\scaps}{\scenario{S-CAPS}}
\newcommand{\dthreefeat}{\scenario{D3Feat}}
\newcommand{\kpconv}{\scenario{KPConv}}
\newcommand{\ucn}{\scenario{UCN}}
\newcommand{\cgf}{\scenario{CGF}}
\newcommand{\ltwonet}{\scenario{L2-Net}}
\newcommand{\ransac}{\scenario{RANSAC}}
\newcommand{\ransactenk}{\scenario{RANSAC10K}}

\newcommand{\icp}{\scenario{ICP}}
\newcommand{\threedsmooth}{\scenario{3DSmoothNet}}
\newcommand{\tls}{\scenario{TLS}}
\newcommand{\usip}{\scenario{USIP}}
\newcommand{\iss}{\scenario{ISS3D}}
\newcommand{\pvnet}{\scenario{PVNet}}
\newcommand{\sfm}{\scenario{SfM}}
\newcommand{\slam}{\scenario{SLAM}}
\newcommand{\colmap}{\scenario{COLMAP}}
\newcommand{\openthreed}{\scenario{Open3D}}

\newcommand{\yolosixd}{\scenario{YOLO6D}}
\newcommand{\finetune}{\scenario{finetune}}
\newcommand{\pnp}{\scenario{PnP}}
\newcommand{\dpod}{\scenario{DPOD}}
\newcommand{\pthreep}{\scenario{P3P}}
\newcommand{\dgr}{\scenario{DGR}}
\newcommand{\ransaceightyk}{\scenario{RANSAC80K}}
\newcommand{\fcgfgt}{$\text{\fcgf}^{\circ}$}
\newcommand{\capsgt}{$\text{\caps}^{\circ}$}

\newcommand{\finalfcgf}{$\text{\sfcgf}^T$}
\newcommand{\starfcgf}{$\text{\sfcgf}^\star$}
\newcommand{\finalcaps}{$\text{\scaps}^T$}
\newcommand{\starcaps}{$\text{\scaps}^\star$}
\newcommand{\plsr}{\scenario{PLSR}}
\newcommand{\plir}{\scenario{PLIR}}
\newcommand{\bs}{\scenario{BS}}
\newcommand{\multiwayname}{\scenario{Stanford RGBD}}
\newcommand{\tum}{\scenario{TUM RGBD}}
\newcommand{\indoorlidarrgbd}{\scenario{Indoor LIDAR RGBD}}
\newcommand{\redwoodobjs}{\scenario{Redwood Objects}}
\newcommand{\opencv}{\scenario{OpenCV}}
\newcommand{\scannet}{\scenario{ScanNet}}
\newcommand{\horn}{\scenario{Horn}}

\newcommand{\parentheses}[1]{\left( #1 \right)}

\newcommand{\cbrace}[1]{\left\{#1\right\}}

\newcommand{\nrmeasurements}{M}
\newcommand{\measone}{\va}
\newcommand{\meastwo}{\vb}
\newcommand{\model}{\vxx}
\newcommand{\modelset}{\calX}
\newcommand{\modelgt}{\model^{\circ}}
\newcommand{\kpt}{\vp}
\newcommand{\kptone}{\kpt^a}
\newcommand{\kpttwo}{\kpt^b}
\newcommand{\dimone}{d_a}
\newcommand{\dimtwo}{d_b}

\newcommand{\nrkpt}{N}
\newcommand{\barcsq}{\bar{c}^2}
\newcommand{\barc}{\bar{c}}

\newcommand{\embedding}{\calC}

\renewcommand{\norm}[1]{\left\|#1\right\|}
\renewcommand{\subject}{s.t.}
\newcommand{\nnparam}{\vtheta}
\newcommand{\dimnn}{N_\embedding}
\newcommand{\dimdescriptor}{d_{\describe}}

\newcommand{\abs}[1]{\left|#1\right|}

\newcommand{\pmargin}{m_p}
\newcommand{\nmargin}{m_n}
\newcommand{\pnmargin}{m}
\newcommand{\pcoeff}{\lambda_p}
\newcommand{\ncoeff}{\lambda_n}
\newcommand{\pncoeff}{\lambda}
\newcommand{\alm}{\scenario{ALM}}
\newcommand{\pslack}{s_p}
\newcommand{\nslack}{s_n}
\newcommand{\pnslack}{s}
\newcommand{\teach}{\scenario{teach}}
\newcommand{\learn}{\scenario{learn}}
\newcommand{\initembedding}{\calB}
\newcommand{\estnnparam}{\hat{\nnparam}}
\newcommand{\estmodel}{\hat{\model}}
\newcommand{\filterset}{\calS}
\newcommand{\filter}{\scenario{verify}}
\newcommand{\sift}{\scenario{SIFT}}
\newcommand{\superpoint}{\scenario{SuperPoint}}
\newcommand{\usphere}[1]{\mathbb{S}^{#1}}
\newcommand{\nrmodels}{O}
\newcommand{\kptpredict}{\vq^b}
\newcommand{\kptpredicthomo}{\tilde{\vq}^b}
\newcommand{\kptonehomo}{\tilde{\kpt}}
\newcommand{\softmax}{\scenario{softmax}}
\newcommand{\describe}{\calF}
\newcommand{\MRgt}{\MR^{\circ}}
\newcommand{\vtgt}{\vt^{\circ}}
\newcommand{\MEgt}{\ME^{\circ}}
\newcommand{\MFgt}{\MF^{\circ}}
\newcommand{\hatmap}[1]{[#1]_{\times}}
\newcommand{\supp}{Supplementary Material}
\newcommand{\Kone}{\MK^a}
\newcommand{\Ktwo}{\MK^b}

\newcommand{\filterlabel}{\scenario{verifyLabel}}

\newcommand{\retrain}{\scenario{retrain}}
\newcommand{\true}{\scenario{True}}
\newcommand{\false}{\scenario{False}}
\newcommand{\randomparam}{\scenario{randParam}}
\newcommand{\namefilter}{verifier}
\newcommand{\Namefilter}{Verifier}

\begin{document}

\title{\vspace{-16mm} Self-supervised \topic \vspace{-2mm}}

\author{Heng Yang\thanks{Equal contribution. Work performed during internship at Intel Labs.}\\
MIT LIDS
\and
Wei Dong${}^*$\\
CMU RI
\and
Luca Carlone \\
MIT LIDS
\and
Vladlen Koltun\\
Intel Labs
\vspace{-2mm}}

\maketitle


\begin{abstract}
We present \emph{self-supervised \topicllc} (\nameshort), the first general framework to learn a feature descriptor for correspondence matching without any ground-truth geometric model labels (\eg,~camera poses, rigid transformations). Our first contribution is to formulate \topicllc as an optimization problem that \emph{jointly} optimizes the feature descriptor and the geometric models given a large corpus of visual measurements (\eg,~images, point clouds). Under this optimization formulation, we show that two important streams of research in vision, namely \emph{robust model fitting} and \emph{deep feature learning}, correspond to optimizing one block of the unknown variables while fixing the other block. This analysis naturally leads to our second contribution -- the \nameshort algorithm that performs \emph{alternating minimization} to solve the joint optimization. \nameshort iteratively executes two meta-algorithms: a \emph{teacher} that performs robust model fitting given learned features to generate \emph{geometric pseudo-labels}, and a \emph{student} that performs deep feature learning under noisy supervision of the pseudo-labels. As a third contribution, we apply \nameshort to two perception problems on large-scale real datasets, namely relative camera pose estimation on \megadepth and point cloud registration on \threedmatch. We demonstrate that \nameshort achieves state-of-the-art performance that is on-par or superior to the supervised oracles trained using ground-truth labels.\footnote{Code available at \url{https://github.com/theNded/SGP}.}
\end{abstract}

\section{Introduction}
\label{sec:introduction}
\emph{\topiclc} is the task of estimating geometric models (\eg,~camera poses, rigid transformations, and 3D structures) from visual measurements (\eg,~images or point clouds). It is a fundamental class of problems in computer vision that has extensive applications in object detection and pose estimation~\cite{Yang19rss-teaser,Zakharov2019dpod}, motion estimation and 3D reconstruction~\cite{Choi15cvpr-robustrecon,Dong19iros-gpuRpbustScene}, simultaneous localization and mapping (\slam)~\cite{Cadena16tro-slam}, structure from motion (\sfm)~\cite{Schonberger16cvpr-sfm}, and virtual and augmented reality~\cite{Klein07-ptam}, to name a few.

Modern \topicllc typically consists of a \emph{front-end} that detects, represents, and associates (sparse or dense) keypoints to establish \emph{putative correspondences}, and a \emph{back-end} that performs estimation of the geometric models while being robust to \emph{outliers} (\ie,~incorrect correspondences). Traditionally, hand-crafted keypoint detectors and feature descriptors, such as \sift~\cite{lowe2004ijcv-distinctive} and \fpfh~\cite{rusu2009icra}, have been used for feature matching in 2D images and 3D point clouds. Despite being general and efficient to compute, hand-crafted features typically lead to an overwhelming number of outliers so that robust estimation algorithms struggle to return accurate estimates of the geometric models. For example, it is not uncommon to have over $95\%$ of the correspondences estimated from \fpfh be outliers in point cloud registration~\cite{Para18pami-GORE,Yang20arXiv-teaser}. As a result, learning feature descriptors from data, particularly using deep neural networks, has become increasingly popular. Learned feature descriptors have been shown to consistently and significantly outperform their hand-crafted counterparts across applications such as relative camera pose estimation~\cite{wang20eccv-caps,schmidt16ral-slamvisualdescriptorlearning}, 3D point cloud registration~\cite{Choy19iccv-FCGF,gojcic2019cvpr}, and object detection and pose estimation~\cite{Peng19cvpr-pvnet,Zakharov2019dpod,Tekin18cvpr-yolo6d,Xiang17RSS-posecnn}. 

However, existing feature learning approaches have several major shortcomings. First, a large number of \emph{ground-truth} geometric model labels are required for training. For example, ground-truth relative camera poses are needed for training image keypoint descriptors~\cite{wang20eccv-caps,Melekhov17ICACIVS-relativepose,En18eccvW-rpnet}, pairwise rigid transformations are required for training point cloud descriptors~\cite{Choy19iccv-FCGF,gojcic2019cvpr,Xie20eccv-pointcontrast,yuan20eccv-deepgmr,Wang19iccv-DCP}, and object poses are used to train image keypoint predictors~\cite{Peng19cvpr-pvnet,Zakharov2019dpod}. Second, although obtaining ground-truth geometric labels is trivial in some controlled settings such as robotic manipulation~\cite{florence18corl-denseobjectnets}, in general the  labels come from full 3D reconstruction pipelines (\eg,~\colmap~\cite{Schonberger16cvpr-sfm}, \openthreed~\cite{Zhou18arxiv-open3D}) 
that require delicate parameter tuning, partial human supervision, and extra sensory information such as IMU and GPS. As a result, the success of feature learning is limited to a handful of datasets with ground-truth annotations~\cite{Zeng17cvpr-3dmatch,Dai17cvpr-scannet,Li18cvpr-megadepth,Xiang17RSS-posecnn,Brachmann14eccv-occulinemod}. 

In this paper, we ask the key question: \emph{Can we design a general framework for feature learning that requires no ground-truth geometric labels or sophisticated reconstruction pipelines?} Our answer is affirmative.

{\bf Contributions}. We formulate \topicllc as an optimization problem that jointly searches for the best feature descriptor (for correspondence matching) and the best geometric models given a large corpus of visual measurements. This formulation incorporates robust model fitting and deep feature learning as two \emph{subproblems}: (i) \emph{robust estimation} only searches for the geometric models, while consuming putative correspondences established from a given feature descriptor; (ii) \emph{feature learning} searches purely for the feature descriptor, while relying on full supervision from the ground-truth geometric models. This generalization naturally endows \topicllc with an iterative algorithm that solves the joint optimization based on alternating minimization, which we name as \emph{self-supervised \topicllc} (\nameshort). At each iteration, \nameshort alternates two meta-algorithms: a \emph{teacher}, that generates geometric pseudo-labels using correspondences established from the learned features, and a \emph{student}, that refines the learned features under the \emph{noisy} supervision from the updated geometric models. \nameshort is initialized by generating geometric pseudo-labels using a bootstrap descriptor, \eg,~a descriptor that is hand-crafted or is trained using synthetic data. We apply \nameshort to solve two perception problems -- relative camera pose estimation and 3D point cloud registration -- and demonstrate that (i) \nameshort achieves on-par or superior performance compared to the supervised oracles; (ii) \nameshort sets the new state of the art on the \megadepth~\cite{Li18cvpr-megadepth} and \threedmatch~\cite{Zeng17cvpr-3dmatch} benchmarks.


\section{Related Work}
\label{sec:relatedwork}




{\bf Deep feature learning}.
With the recent advance of deep learning, a plethora of deep features have been developed to replace classical hand-crafted feature descriptors such as \sift~\cite{lowe2004ijcv-distinctive} and \fpfh~\cite{rusu2009icra} for correspondence matching, and boost the performance of \topicllc tasks. 
For 2D features, Choy~\etal~\cite{Choy16neurips-UCN} develop Universal Correspondence Network (\ucn) for visual correspondence estimation with metric contrastive learning. Tian \etal~\cite{tian2017cvpr} introduce \ltwonet to extract patch descriptors for keypoints. While these methods require direct correspondence supervision, Wang \etal~\cite{wang20eccv-caps} only use 2D-2D camera poses to supervise the learning of feature descriptors.
%
The success of 2D feature learning extends to 3D.
Khoury \etal~\cite{Khoury17iccv-CGF} created Compact Geometric Features (\cgf) by optimizing deep networks that map high-dimensional histograms into low-dimensional Euclidean spaces. Gojcic \etal~\cite{gojcic2019cvpr} propose \threedsmooth for 3D keypoint descriptor generation with its network structure based on \ltwonet. Choy \etal~\cite{Choy19iccv-FCGF} developed fully convolutional geometric features (\fcgf) based on sparse convolutions. Bai \etal~\cite{bai2020cvpr} build \dthreefeat on kernel point convolution (\kpconv)~\cite{thomas2019iccv} and emphasize 3D keypoint detection. Since ground-truth 3D correspondences are non-trivial to obtain, nearest neighbor search using known 3D transformations is the standard supervision signal.

{\bf Robust estimation}.
Robust estimation ensures reliable geometric model estimation in the presence of outlier correspondences. \emph{Consensus maximization}~\cite{Chin17slcv-maximumConsensusAdvances} and \emph{M-estimation}~\cite{Bosse17fnt} are the two popular formulations. Algorithms for solving both formulations can be divided into \emph{fast heuristics}, \emph{global solvers}, and \emph{certifiable algorithms}. Fast heuristics, such as \ransac~\cite{Fischler81,Barath18cvpr-gcransac,Barath20cvpr-magsac++} and \gnc~\cite{Zhou16eccv-fastGlobalRegistration,Yang20ral-GNC,Antonante20arxiv-outlierrobust}, are efficient but offer few performance guarantees. Global solvers, typically based on branch-and-bound~\cite{Bazin12accv-globalRotSearch,Para18pami-GORE,Izatt17isrr-MIPregistration,Yang2014ECCV-optimalEssentialEstimationBnBConsensusMax,Bazin12pami-BnBGrouping,Li09iccv-consensusMax} or exhaustive search~\cite{Enqvist12eccv-robustFitting,Ask13-optimalTruncatedL2,Chin15-CMTreeAstar,Cai19ICCV-CMtreeSearch}, are globally optimal but often run in exponential time. Recently proposed certifiable algorithms~\cite{Yang20nips-certifiablePerception,Yang20arXiv-teaser,Yang20cvpr-shapeStar,yang19iccv-quasar} combine fast heuristics with scalable optimality certification. Outlier-pruning methods~\cite{Para18pami-GORE,Yang19rss-teaser,Shi20arxiv-robin} can significantly boost the robustness and efficiency of estimation algorithms. In this paper, we use robust estimation to \emph{teach} feature learning.


{\bf Self-supervision}.
Self-supervision has been widely adopted in visual learning~\cite{jing2020pami} to avoid massive human annotation. In such tasks, labels can be automatically generated by standard image operations~\cite{ledig2017photo, zhang2016colorful}, classical vision algorithms~\cite{li2016cvpr,jiang2018eccv}, or simulation~\cite{dosovitskiy2017arxiv,richter2017iccv}. In real-world setups, geometric vision has actively employed self-supervision in optical flow~\cite{liu2019cvpr}, depth prediction~\cite{wang2018cvpr,godard2019iccv}, visual odometry~\cite{zhou2017cvpr, yang2020cvpr}, and registration~\cite{yew2018eccv, Choy19iccv-FCGF, bai2020cvpr}. These tasks rely on the supervision from camera poses or relative rigid transformations for image warping and correspondence generation, and thus benefit from well-established  \slam~\cite{mur2017tro}, 3D reconstruction~\cite{Zhou18arxiv-open3D}, and \sfm~\cite{Schonberger16cvpr-sfm} pipelines. Although these systems are off-the-shelf, they usually require long execution times, delicate parameter tuning, and human supervision to safeguard their correctness. In this paper, we show how to perform self-supervised feature learning without 3D reconstruction pipelines and  ground-truth geometric labels.

{\bf Self-training}.
Self-training~\cite{yarowsky1995unsupervised,grandvalet2005semi}, as a special case of semi-supervised learning, has gained popularity in visual learning due to its potential to adapt to large-scale unlabeled data. Self-training first trains a model on a labeled dataset, then applies it on a larger unlabeled dataset to obtain \textit{pseudo-labels}~\cite{lee2013pseudo} for further training.
Although pseudo-labels can be noisy, recent studies have shown that SOTA performance can be achieved on image classification~\cite{xie2020cvpr,zoph2020arxiv}, and initial theoretical analyses have been proposed~\cite{wei2020theoretical}. Our work uses robust estimation to generate pseudo-labels without initial supervised training, the first work to showcase the effectiveness of pseudo-labels in training feature descriptors for \topicllc.


\section{The \nameshort Formulation}
\label{sec:mathematical-framework}
In this section, we first formulate \topicllc as a problem that \emph{jointly} optimizes a correspondence matching function (\ie,~learning a descriptor) and the geometric models given a corpus of visual data (Section~\ref{sec:jointfeatlearnmodelest}). Then we show that two of the most important research lines in computer vision, namely \emph{robust estimation} and \emph{feature learning}, correspond to fixing one part of the joint problem while optimizing the other part (Sections~\ref{sec:robustestimation} and~\ref{sec:metriclearning}).

\subsection{Joint Feature Learning and Model Estimation}
\label{sec:jointfeatlearnmodelest}
We focus on \topicllc with pairwise correspondences between visual measurements.

\begin{problem}[\topic]
\label{problem:perception} 
Consider a corpus of $\nrmeasurements$ \emph{pairwise} visual measurements $\cbrace{\measone_i, \meastwo_i}_{i=1}^{\nrmeasurements}$, such as images or point clouds, and assume $\measone_i$ and $\meastwo_i$ are related through a \emph{geometric model} with unknown parameters ${\model_i \in \modelset}$, where $\calX$ is the domain of the geometric models such as 3D poses. Suppose there is a preprocessing module $\phi$ that can extract a sparse or dense set of keypoint locations for each measurement,~\ie,
\bea
\kptone_i = \phi\parentheses{\measone_i} \in \Real{\dimone \times \nrkpt_{a_i}} , \ \ 
\kpttwo_i = \phi\parentheses{\meastwo_i} \in \Real{\dimtwo \times \nrkpt_{b_i}},
\eea
for all $i=1,\dots,\nrmeasurements$, where $\dimone,\dimtwo$ are the dimensions of the keypoint locations (\eg,~$2$ for images keypoints and $3$ for point cloud keypoints), and $\nrkpt_{a_i},\nrkpt_{b_i}$ are the number of keypoints in $\measone_i$ and $\meastwo_i$ (w.l.o.g., assume $ \nrkpt_{a_i} \leq \nrkpt_{b_i}$), then the problem of \topicllc seeks to jointly learn a correspondence function $\embedding$ and estimate the unknown geometric models $\model_i$ by solving the following optimization:
\bea
\min_{\embedding, \cbrace{\model_i}_{i=1}^\nrmeasurements \in \modelset^\nrmeasurements } & \displaystyle  \sum_{i=1}^\nrmeasurements \sum_{k=1}^{\nrkpt_{a_i}} \rho\parentheses{r\parentheses{\vxx_i, \kptone_{i,k}, \kptpredict_{i,k}}} \label{eq:geometricperception}\\
\subject\quad\quad &  \displaystyle \kptpredict_{i,k} = \embedding(\kptone_{i,k}, \measone_i, \kpttwo_{i}, \meastwo_i),
\label{eq:correspondence}
\eea
where $\kptone_{i,k} \in \Real{\dimone}$ denotes the location of the $k$-th keypoint in $\measone_i$, $\kptpredict_{i,k} \in \Real{\dimtwo}$ denotes the location of the \emph{corresponding} keypoint in $\meastwo_i$, $r\parentheses{\cdot}$ is the residual function that quantifies the mismatch between the two keypoints $\kptone_{i,k}$ and $\kptpredict_{i,k}$ under the geometric model $\model_i$, $\rho\parentheses{\cdot}$ is a robust cost function that penalizes the residuals, and $\embedding\parentheses{\cdot}$ is a function that takes each keypoint in $\measone_i$ as input and predicts the corresponding keypoint in $\meastwo_i$, by learning features from the visual data. 
\end{problem}
To the best of our knowledge, Problem~\ref{problem:perception} is the first formulation that considers \emph{joint} feature learning and model estimation in \topicllc. The correspondence function $\embedding$ typically contains a \emph{learnable} feature descriptor (\eg,~parametrized by a deep neural network) and a matching function (\eg,~soft or hard nearest neighbor search) that generates correspondences using the learned descriptor. We now give two examples of Problem~\ref{problem:perception}.

\setcounter{theorem}{0}
\begin{example}[Relative Pose Estimation]
\label{ex:relativepose}
Consider a corpus of image pairs $\{\measone_i,\meastwo_i \}_{i=1}^{\nrmeasurements}$ with known camera intrinsics, where $\measone_i, \meastwo_i$ are RGB images, let $\phi(\cdot)$ be a keypoint detector,~\eg,~\sift~\cite{lowe2004ijcv-distinctive}, \superpoint~\cite{Detone18cvprw-superpoint}, or a dense random pixel location sampler~\cite{wang20eccv-caps}, such that $\kptone_i = \phi(\measone_i) \in \Real{2 \times \nrkpt_{a_i}}$ and $\kpttwo_i = \phi(\meastwo_i) \in \Real{2\times \nrkpt_{b_i}}$ are two sets of 2D keypoint locations. Relative pose estimation seeks to jointly learn a correspondence prediction function $\embedding$ and estimate the relative poses $\model_i = (\MR_i,\vt_i) \in \SOthree \times \usphere{2}$ between images.\footnote{The translation $\vt \in \usphere{2} \doteq \{\vt \in \Real{3}| \norm{\vt}=1 \}$ is up to scale.} In particular, following~\cite{wang20eccv-caps}, let $\embedding$ be a composition of a deep feature descriptor $\describe(\cdot)$, a \softmax function~\cite{Goodfellow16book-deeplearning}, and a weighted average:
\bea
\kptpredict_{i,k} = \sum_{j=1}^{\nrkpt_{b_i}} \kpttwo_{i,j} \frac{\exp\parentheses{\describe(\kptone_{i,k},\measone_i)\tran \describe(\kpttwo_{i,j},\meastwo_i)}}{\sum_{j=1}^{\nrkpt_{b_i}} \exp\parentheses{\describe(\kptone_{i,k},\measone_i)\tran \describe(\kpttwo_{i,j},\meastwo_i)} }, \label{eq:correspondececaps}
\eea
where the descriptor $\describe$ takes the image and the keypoint location as input and outputs a high-dimensional feature vector for each keypoint,~\ie,~$\describe(\kptone_{i,k},\measone_i) \in \Real{\dimdescriptor}$, where $\dimdescriptor$ denotes the dimension of the descriptor, the \softmax function computes the probability of $\kpttwo_{i,j}$ being a match to $\kptone_{i,k}$ according to their inner product in the descriptor space, and the weighted average function returns the keypoint location as a weighted sum of all keypoint locations discounted by their matching probabilities.
\end{example} 

\begin{example}[Point Cloud Registration]
\label{ex:pointcloudregistration}
Consider a corpus of point cloud pairs $\{\measone_i,\meastwo_i\}_{i=1}^{\nrmeasurements}$, where $\measone_i, \meastwo_i$ are 3D point clouds, let $\phi(\cdot)$ be a 3D keypoint detector,~\eg,~\iss~\cite{Zhong09iccvw-ISS},~\usip~\cite{L19cvpr-usip}, or a dense uniform voxel downsampler~\cite{Choy19iccv-FCGF}, such that $\kptone_i = \phi(\measone_i) \in \Real{3\times \nrkpt_{a_i}}$, and $\phi(\meastwo_i) \in \Real{3\times \nrkpt_{b_i}}$ are two sets of 3D keypoints. Point cloud registration seeks to jointly learn a correspondence function $\embedding$ and estimate the rigid transformation $\model_i = (\MR_i,\vt_i) \in \SOthree \times \Real{3}$ between point clouds. In particular, following~\cite{Choy19iccv-FCGF,gojcic2019cvpr}, let $\embedding$ be a composition of a deep feature descriptor $\describe(\cdot)$ and nearest neighbor search:
\bea
\kptpredict_{i,k} = \argmin_{\kpttwo_{i,j} \in \kpttwo_i} \norm{\describe(\kpttwo_{i,j}, \meastwo_i) - \describe(\kptone_{i,k}, \measone_i)}, \label{eq:nns}
\eea
where the descriptor $\describe$ takes the point cloud and the keypoint location as input and outputs a high-dimensional feature vector for each keypoint,~\ie,~$\describe(\kptone_{i,k},\measone_i) \in \Real{\dimdescriptor}$, with $\dimdescriptor$ denoting the descriptor dimension, and condition~\eqref{eq:nns} asks that the corresponding keypoint $\kptpredict_{i,k}$ is the keypoint among $\kpttwo_i$ that achieves the shortest distance to $\kptone_{i,k}$ in  descriptor space.\footnote{Alternatively, one can establish correspondences through \emph{cross check}~\cite{Zhou16eccv-fastGlobalRegistration} or \emph{ratio test}~\cite{lowe2004ijcv-distinctive}. In addition to~\eqref{eq:nns}, cross check asks $\kptone_{i,k}$ is also the closest keypoint to $\kptpredict_{i,k}$ among $\kptone_i$, while ratio test asks the ratio $\| \describe(\kptone_{i,k}, \measone_i) - \describe(\kptpredict_{i,k},\meastwo_i) \| / \| \describe(\kptone_{i,k}, \measone_i) - \describe(\kpttwo_{i,j},\meastwo_i) \|$ is below a predefined threshold $\zeta < 1$ for all $\kpttwo_{i,j} \neq \kptpredict_{i,k}$.}
\end{example}


Examples~\ref{ex:relativepose}-\ref{ex:pointcloudregistration} represent two key problems in vision that concern pose estimation from 2D-2D and 3D-3D measurements, all of which involve the coupling of correspondence matching (\aka~data association) and geometric model estimation. Interestingly, although little is known about how to solve Problem~\ref{problem:perception} directly, significant efforts have been made to solve its two subproblems.


\subsection{Robust Estimation}
\label{sec:robustestimation}
\begin{problem}[Robust Estimation]
\label{problem:robustestimation}
In Problem~\ref{problem:perception}, assuming the correspondence matching function $\embedding$ is known, robust estimation seeks to estimate the unknown parameters of the geometric models given putative correspondences (corrupted by \emph{outliers}), by optimizing the following objective:
\bea
\min_{\cbrace{\model_i}_{i=1}^\nrmeasurements \in \modelset^\nrmeasurements } & \displaystyle  \sum_{i=1}^\nrmeasurements \sum_{k=1}^{\nrkpt_{a_i}} \rho\parentheses{r\parentheses{\vxx_i, \kptone_{i,k}, \kptpredict_{i,k}}}. \label{eq:robustestimation}
\eea
\end{problem}
Problem~\ref{problem:robustestimation} shows that robust estimation is a subproblem of Problem~\ref{problem:perception} with a known and fixed correspondence function. 
Despite the nonconvexity of problem~\eqref{eq:robustestimation} (\eg,~due to a nonconvex $\calX$ or a nonconvex $\rho$), research in robust estimation has focused on improving the robustness~\cite{Yang20arXiv-teaser,Para18pami-GORE}, efficiency~\cite{Barath20cvpr-magsac++} and theoretical guarantees~\cite{Yang20nips-certifiablePerception} of estimation algorithms to mitigate the adversarial effects of outliers on the estimated geometric models. 


\subsection{Supervised Feature Learning}
\label{sec:metriclearning}

\begin{problem}[Supervised Feature Learning]
\label{problem:featurelearning}
In Problem~\ref{problem:perception}, assuming the parameters of the geometric models are known and denoting them as $\modelgt_i,i=1,\dots,\nrmeasurements$, feature learning seeks to find the best correspondence matching function $\embedding_{\nnparam}$ by solving the following optimization problem:
\bea \label{eq:featurelearning}
\min_{\nnparam \in \Real{\dimnn}} & \displaystyle \sum_{i=1}^M \sum_{k=1}^{\nrkpt_{a_i}} \rho(r(\modelgt_i,\kptone_{i,k},\kptpredict_{i,k})) \label{eq:featurelearning} \\
\subject & \displaystyle \kptpredict_{i,k} = \embedding_{\nnparam}(\kptone_{i,k}, \measone_i, \kpttwo_{i}, \meastwo_i), \label{eq:featurelearningcon}
\eea
where the correspondence function is parametrized by the weights $\nnparam \in \Real{\dimnn}$ of a deep (descriptor) neural network and $\dimnn$ is the number of weight parameters in the network.
\end{problem}

At first glance, the optimization~\eqref{eq:featurelearning} is different from the loss functions designed in the supervised feature learning literature~\cite{Choy19iccv-FCGF,wang20eccv-caps,Zakharov2019dpod}. However, the next proposition states that, if we take $\rho(\cdot)$ to be the truncated least squares (\tls) cost function, then common loss functions can be designed using the \emph{Augmented Lagrangian Method} (\alm)~\cite{Bertsekas99book-nonlinearprogramming}.

\setcounter{theorem}{0}
\begin{proposition}[Feature Learning as \alm]
\label{prop:augmentlagrangian}
Let $\rho(r) = \min\cbrace{r^2, \barcsq}$ be the \tls cost function~\cite{Yang20nips-certifiablePerception}, where $\barc>0$ sets the maximum allowed inlier residual, supervised feature learning~\cite{wang20eccv-caps,Choy19iccv-FCGF} in Examples~\ref{ex:relativepose}-\ref{ex:pointcloudregistration} can solve the optimization~\eqref{eq:featurelearning}. In particular, the loss functions in~\cite{wang20eccv-caps,Choy19iccv-FCGF} can be interpreted as the Augmented Lagrangian of problem~\eqref{eq:featurelearning}.
\end{proposition}
\begin{proof}
See the \supp.
\end{proof}
Proposition~\ref{prop:augmentlagrangian} states that, just as robust estimation algorithms optimize geometric models given a fixed correspondence matching function, supervised feature learning methods optimize the feature descriptor given known geometric models. 
In the next section, we show that this framework naturally allows us to solve Problem~\ref{problem:perception} by alternating the execution of robust estimation and feature learning.


\section{The \nameshort Algorithm}
\label{sec:altmin}
We first give an overview of the \nameshort algorithm (Section~\ref{sec:overview}), then  discuss its applications (Section~\ref{sec:applications}).

\subsection{Overview}
\label{sec:overview}
\begin{figure}[h]
	\vspace{-5mm}
	\begin{center}
	\begin{tabular}{c}%
	\begin{minipage}{\columnwidth}%
	\centering%
	\includegraphics[width=\columnwidth]{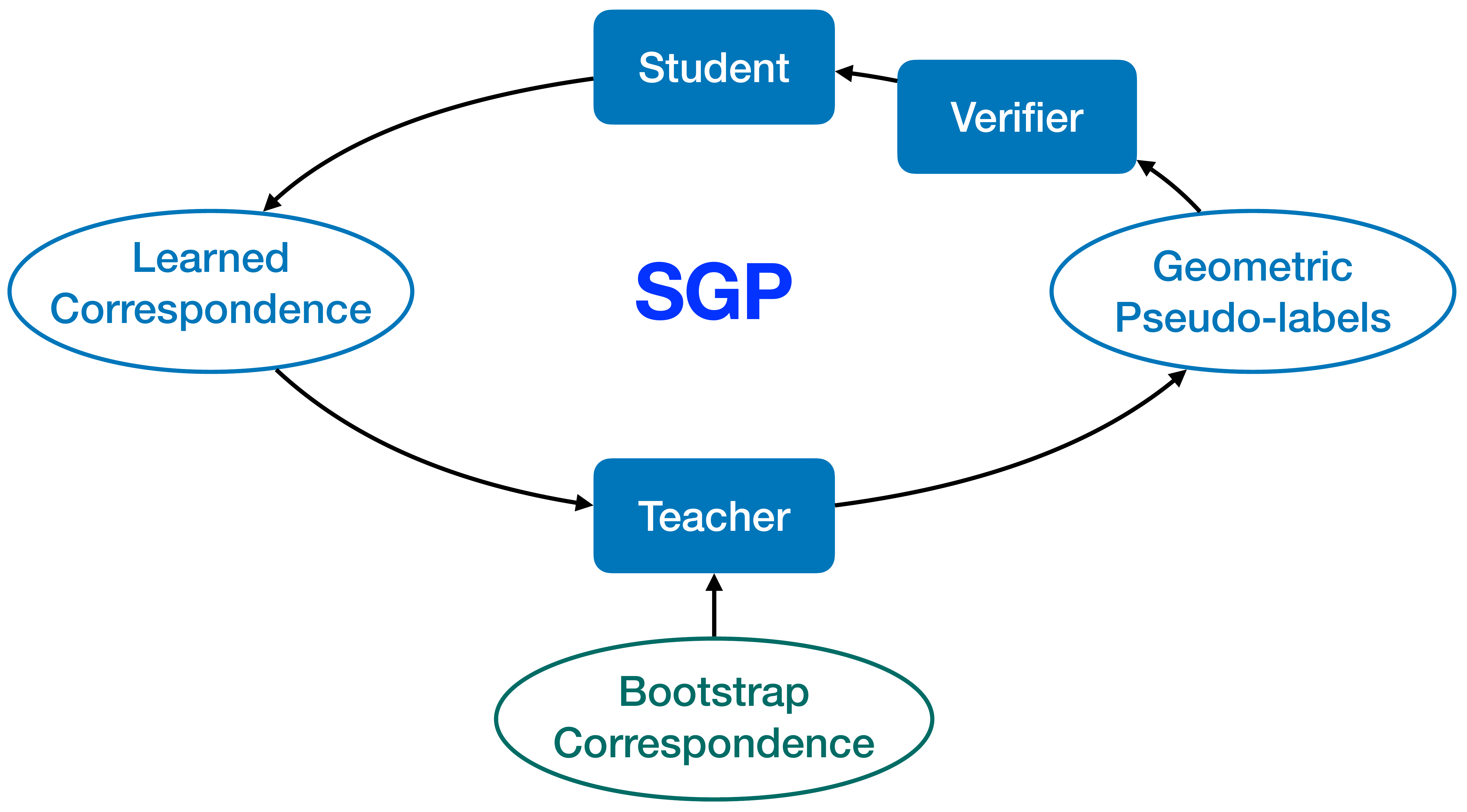} \\
	\end{minipage}
	\end{tabular}
	\end{center}
	\caption{Algorithmic overview of \nameshort.
	\label{fig:overview}}
	\vspace{-3mm}
\end{figure}
\setlength{\textfloatsep}{5pt}%
\begin{algorithm}[t]
\SetAlgoLined
\textbf{Input:} A corpus of visual measurements: $\{ \measone_i,\meastwo_i \}_{i=1}^{\nrmeasurements}$; a preprocessing module: $\phi$; an initial correspondence matching method: $\initembedding$; an architecture for a learned correspondence prediction function: $\embedding$, with initial weights: $\nnparam^{(0)}$ (default: \randomparam); Number of iterations: $T$; boolean: \filterlabel (default \true); boolean: \retrain (default \false); \\
\textbf{Output:} final weights of $\embedding$: $\estnnparam$; estimated geometric models: $\{\estmodel_i \}_{i=1}^{\nrmeasurements}$; \\

\gray{ \% Compute keypoint locations}  \\
$\kptone_i = \phi(\measone_i)$, $\kpttwo_i = \phi(\meastwo_i)$,$\ \ \forall i \in [\nrmeasurements]$; \label{line:computekeypoint} \\

\gray{\% Bootstrap (Initialize pseudo-labels)} \\
$\model_i^{(0)} = \teach(\measone_i,\meastwo_i,\kptone_i,\kpttwo_i,\initembedding),\ \  \forall i \in [\nrmeasurements]$; \label{line:bootstrap}

\gray{\% Alternating minimization} \\
\For{ $\tau = 1:T$ \label{line:maxIters}}{ 

\uIf{\filterlabel = \true}{
	\gray{\% Verify correctness of labels} \\
$\filterset = \filter(\{\model_i^{(\tau-1)},\measone_i,\meastwo_i,\kptone_i,\kpttwo_i\}_{i=1}^{\nrmeasurements})$; \label{line:filter}
}
\Else{
	$\filterset = [\nrmeasurements]$; \label{line:nofilter}
}
\gray{\% Feature learning (problem~\eqref{eq:featurelearning})} \\
\uIf{\retrain = \true}{
	$\nnparam = \randomparam$; \ \gray{\% \retrain} \label{line:randinit}
}
\Else{
	$\nnparam = \nnparam^{(\tau-1)}$; \ \gray{\% \finetune} \label{line:finetune}
}
$\nnparam^{(\tau)} = \learn(\{\model_i^{(\tau-1)},\measone_i,\meastwo_i,\kptone_i,\kpttwo_i\}_{i \in \filterset}, \nnparam)$; \label{line:learn}

\gray{\% Robust estimation (problem~\eqref{eq:robustestimation})}\\

$\model_i^{(\tau)} = \teach(\measone_i,\meastwo_i,\kptone_i,\kpttwo_i,\embedding_{\nnparam^{(\tau)}}),\ \ \forall i \in [\nrmeasurements]$; \label{line:teach}
}

 \textbf{return:} $\estnnparam = \nnparam^{(T)}$, $\estmodel_i = \model_i^{(T)},i=1,\dots,\nrmeasurements$.  \label{line:return}
 \caption{\nameshort \label{alg:sgp}}
\end{algorithm}

An overview of \nameshort is shown in Fig.~\ref{fig:overview}, and details of \nameshort are summarized in Algorithm~\ref{alg:sgp}. \nameshort does not have access to the ground-truth geometric models and internally creates \emph{geometric pseudo-labels}. \nameshort contains three key components: a \emph{teacher}, a \emph{student} and (optionally) a \emph{\namefilter}.
\setcounter{theorem}{0}
\begin{definition}[Teacher]
\label{def:teacher}
An algorithm that estimates geometric pseudo-labels given a correspondence matcher.
\end{definition}

\begin{definition}[Student]
\label{def:student}
An algorithm that estimates the parameters of a correspondence matching function under the supervision of geometric models.
\end{definition}

\begin{definition}[\Namefilter]
\label{def:filter}
An algorithm that verifies if a geometric model estimated by the teacher is correct.
\end{definition}

From the definitions above, one can see that a teacher is a solver for the robust estimation problem~\eqref{eq:robustestimation}, while a student is a solver for the supervised feature learning problem~\eqref{eq:featurelearning}. Because problems~\eqref{eq:robustestimation} and~\eqref{eq:featurelearning} are the two subproblems of the joint \topicllc problem~\eqref{eq:geometricperception}, the \nameshort algorithm~\ref{alg:sgp} alternates in executing the teacher and the student (\cf~line~\ref{line:learn}-\ref{line:teach}), referred to as the \emph{teacher-student loop}, to perform alternating minimization for the joint problem~\eqref{eq:geometricperception}. 

In particular, at the $\tau$-th iteration of the teacher-student loop, the student initializes the network parameters at $\nnparam$, and updates the parameters to $\nnparam^{(\tau)}$, by minimizing problem~\eqref{eq:featurelearning} (using stochastic gradient descent) under the noisy ``supervision'' of the geometric pseudo-labels estimated from iteration $\tau-1$ (line~\ref{line:learn}). The student either initializes $\nnparam$ at random (line~\ref{line:randinit}, referred to as \retrain), or initializes $\nnparam$ from the weights of the last iteration $\nnparam^{(\tau-1)}$ (line~\ref{line:finetune}, referred to as \finetune). Then, using the correspondence function with updated parameters, denoted by $\embedding_{\nnparam^{(\tau)}}$, the teacher solves robust estimation~\eqref{eq:robustestimation} to update the models (line~\ref{line:teach}). 

Throughout the teach-student loop, neither the correspondence matcher nor the teacher are perfect, leading to a significant fraction of the geometric pseudo-labels being incorrect, which can potentially bias the student. Therefore, \nameshort optionally uses a \namefilter~to generate a verified set of pseudo-labels, denoted by $\filterset$, that are more likely to be correct (line~\ref{line:filter}). If the flag \filterlabel is \false, then $\filterset = [\nrmeasurements]$ is the full set of pseudo-labels (line~\ref{line:nofilter}). The \namefilter~design is application dependent, as discussed in Section~\ref{sec:applications}.

An \emph{initialization} is required to start the iterative updates in alternating minimization. To do so, we initialize the geometric models by performing model estimation using a \emph{bootstrap matcher} $\initembedding$ (line~\ref{line:bootstrap}). Based on the specific application, the bootstrap matcher can be designed from a hand-crafted feature descriptor that requires no learning, or a descriptor that is trained with a small amount of data, or a descriptor that is trained on synthetic datasets. On the other hand, since we typically do not have prior information about the weights of $\embedding$, $\nnparam^{(0)}$ is initialized at random.

\setcounter{theorem}{0}
\begin{remark}[Implementation Considerations]
\label{remark:implementation}
(i) \emph{Convergence}: In the current \nameshort implementation, we execute the teacher-student loops for a fixed number of iterations $T$. However, one can stop \nameshort if the difference between $\model_i^{(\tau)}$ and $\model_i^{(\tau-1)}$, or between $\nnparam^{(\tau)}$ and $\nnparam^{(\tau-1)}$ is below some threshold. One can also choose the best $\embedding$ from \nameshort by using a validation dataset if available. (ii) \emph{Speedup}: When running the teacher to generate pseudo-labels (line~\ref{line:teach}) at each iteration, one can skip the updates for some labels that are already ``stable''. For example, if a label $\model_i$ remains unchanged for consecutively 3 iterations, or the robust solver achieves high confidence about $\model_i$ (\eg,~\ransac has inlier rate over $80\%$), then the teacher can skip the update for $\model_i$.
\end{remark}

\subsection{Applications}
\label{sec:applications}
We now discuss the application of \nameshort to Examples~\ref{ex:relativepose}-\ref{ex:pointcloudregistration}.

{\bf \nameshort	for Example~\ref{ex:relativepose}}. The teacher performs robust relative pose estimation~\cite{hartley2004book}. Therefore, a good candidate for a teacher is \ransac~\cite{Fischler81} (with Nister's 5-point method~\cite{Nister04pami-fivepoint}) and its variants, such as \gcransac~\cite{Barath18cvpr-gcransac} and \magsac~\cite{Barath20cvpr-magsac++}. The student performs descriptor learning using relative camera pose supervision. Recent work \caps~\cite{wang20eccv-caps} is able to learn a descriptor under the supervision of fundamental matrices, which can be computed from relative pose and camera intrinsics~\cite{hartley2004book}. Therefore, \caps is the student network. The \namefilter~can be designed based on the \emph{inlier rate} estimated by \ransac,~\ie,~the number of inlier matches divided by the total number of putative matches. Intuitively, the higher the inlier rate is, the more likely it is that \ransac has found a correct solution. To initialize \nameshort, we use the hand-crafted \sift descriptor (with ratio test)~\cite{lowe2004ijcv-distinctive}.

{\bf \nameshort for Example~\ref{ex:pointcloudregistration}}. The teacher performs robust registration. Many robust registration algorithms can serve as the teacher: \ransac (with Horn's 3-point method~\cite{horn87josa}) and its variants, \fgr~\cite{Zhou16eccv-fastGlobalRegistration}, and \teaser~\cite{Yang20arXiv-teaser}. As for the student, methods such as \fcgf~\cite{Choy19iccv-FCGF}, \threedsmooth~\cite{gojcic2019cvpr}, and \dthreefeat~\cite{bai2020cvpr} can learn point cloud descriptors under the supervision of rigid transformations. The \namefilter~can be designed based on the \emph{overlap ratio} computed from the estimated pose,~\ie,~the number of point pairs that are close to each other after transformation, divided by the total number of points in the point cloud. One can also use the certifier in \teaser~\cite{Yang20arXiv-teaser}. To initialize \nameshort, we can use the hand-crafted \fpfh descriptor (with cross check)~\cite{rusu2009icra}.


\begin{remark}[Novelty]
\label{remark:novelty}
Hand-crafted descriptors, robust estimation and feature learning are mature areas in computer vision. In this paper, instead of creating new techniques in each area, we show that combining existing techniques from each field in the \nameshort framework can tackle self-supervised \topicllc in full generality.
\end{remark}

\begin{remark}[Generality]
\label{remark:generality}
Although we only provide experimental results for relative pose estimation and point cloud registration, the joint optimization formulation in Problem~\ref{problem:perception} is general and the \nameshort algorithm~\ref{alg:sgp} can be applied in any perception problem where a robust solver and a supervised feature learning method is available. For example, we also present the formulation for \emph{object detection and pose estimation}~\cite{Tekin18cvpr-yolo6d,Peng19cvpr-pvnet,Zakharov2019dpod,chen19ICCVW-satellitePoseEstimation}, and discuss the application of \nameshort in the \supp.
\end{remark}

\section{Experiments}
\label{sec:experiments}

\newcommand{\mpwthree}{5.9cm}
\newcommand{\mpwthreetwo}{11.4cm}
\newcommand{\myhspace}{\hspace{-6mm}}
\newcommand{\myvspace}{\vspace{2mm}}
\newcommand{\subfigheight}{1.9cm}

\begin{figure*}[h]
	\begin{center}
	\begin{minipage}{\textwidth}
	\hspace{-4mm}
	\begin{tabular}{ccc}%
		\begin{minipage}{\mpwthree}%
			\centering%
			\includegraphics[width=\columnwidth]{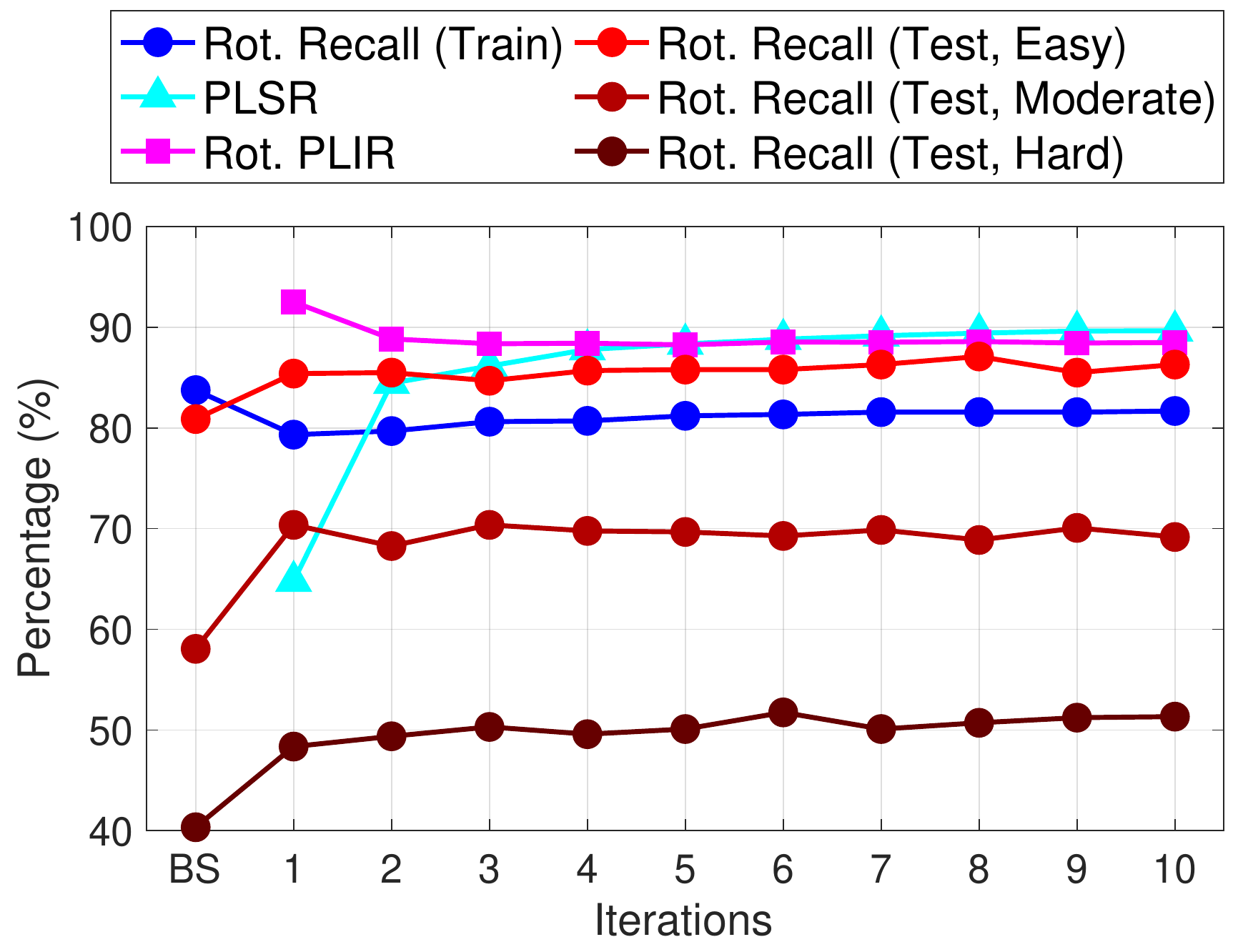} \\
			\vspace{-1mm}
			\caption{Dynamics of \nameshort on \megadepth~\cite{Li18cvpr-megadepth}. \plsr: \emph{Pseudo-Label Survival Rate}. \plir: \emph{Pseudo-Label Inlier Rate}. \bs: Boostrap.}
			\label{fig:relativepose-line-plot}
		  \end{minipage}
	  &
		\multicolumn{2}{c}{%
		\hspace{-4mm}
			\begin{minipage}{\mpwthreetwo}%

\renewcommand{\arraystretch}{1.3}

\adjustbox{max width=11.5cm}{%
\hspace{3mm}
\begin{tabular}{ccccccc}
 & \multicolumn{2}{c}{\emph{Easy} (\%)} & \multicolumn{2}{c}{\emph{Moderate} (\%)} & \multicolumn{2}{c}{\emph{Hard} (\%)} \\
Methods     & Rotation & Translation  & Rotation & Translation & Rotation & Translation  \\
\hline
\sift\!+\ransactenk~\cite{lowe2004ijcv-distinctive}\footnote{\smaller \sift and \ransac implemented in \opencv~\cite{bradski08book-opencv}. \sift uses $0.75$ ratio test. All \ransac use $99.9\%$ confidence.} & $80.9$ & $48.8$ & $58.1$ & $43.5$ & $40.4$ & $34.0$ \\ 
\sift\!+\scenario{Wang-}\caps~\cite{wang20eccv-caps}\footnote{\smaller \label{ft:caps}Recall statistics adapted from the original \caps paper~\cite{wang20eccv-caps}.} & $70.0$ & $30.5$ & $50.2$ & $24.8$ & $36.8$ & $16.1$ \\
\superpoint\!+\scenario{Wang-}\caps~\cite{wang20eccv-caps}\textsuperscript{\ref{ft:caps}} & $72.9$ & $30.5$ & $53.5$ & $27.9$ & $38.1$ & $19.2$ \\
\hline 
\sift\!+\capsgt+\ransactenk\footnote{\smaller Recall computed by using \ransactenk with the pretrained \capsgt (\ie,~the supervised oracle).} & $\bm{87.1}$ & $52.7$ & $\bm{72.5}$ & $\bm{53.8}$ & $\bm{52.7}$ & $45.6$ \\
\hline
\hline
\sift\!+\finalcaps+\ransactenk & $86.3$ & $53.1$ & $69.2$ & $50.3$ & $51.3$ & $\bm{47.1}$ \\
\sift\!+\starcaps+\ransactenk & $\bm{87.1}$ & $\bm{53.5}$ & $70.4$ & $53.3$ & $51.8$ & $\bm{47.1}$ \\
\end{tabular}
}
\\
\captionof{table}{Rotation and translation recalls on the \megadepth~\cite{Li18cvpr-megadepth} test dataset using different methods. \finalcaps: last \caps trained by \nameshort. \starcaps: best \caps trained by \nameshort.
\label{tab:megadepth}}
			\end{minipage}
		} \\

	  \begin{minipage}{\mpwthree}%
			\centering%
			\includegraphics[width=\columnwidth]{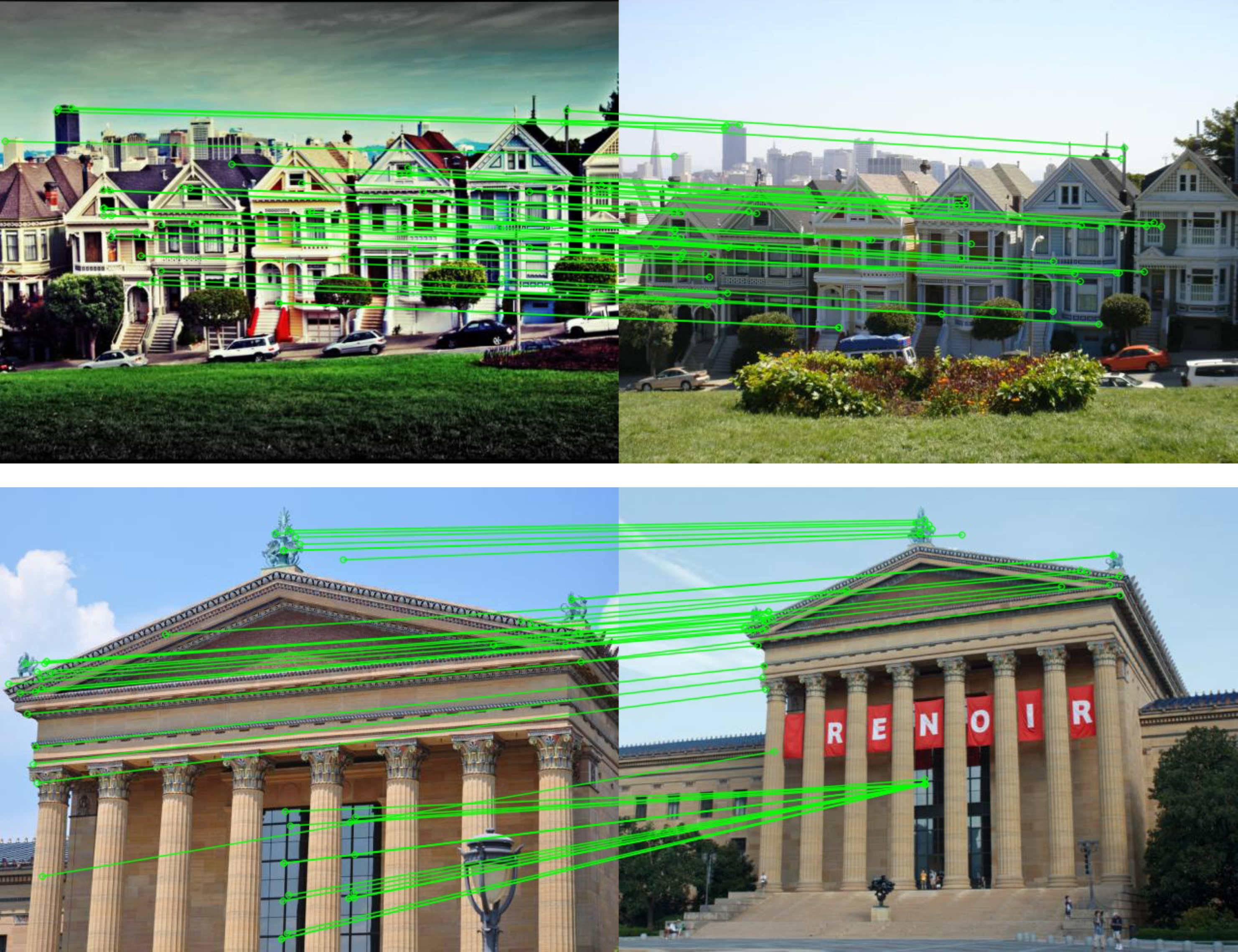}\\
			{\smaller (a) Success (top) and failure (bottom) by \sift.}
			\myvspace
		  \end{minipage}
	  &
	  \hspace{-4mm}
		\begin{minipage}{\mpwthree}%
			\centering%
			\includegraphics[width=\columnwidth]{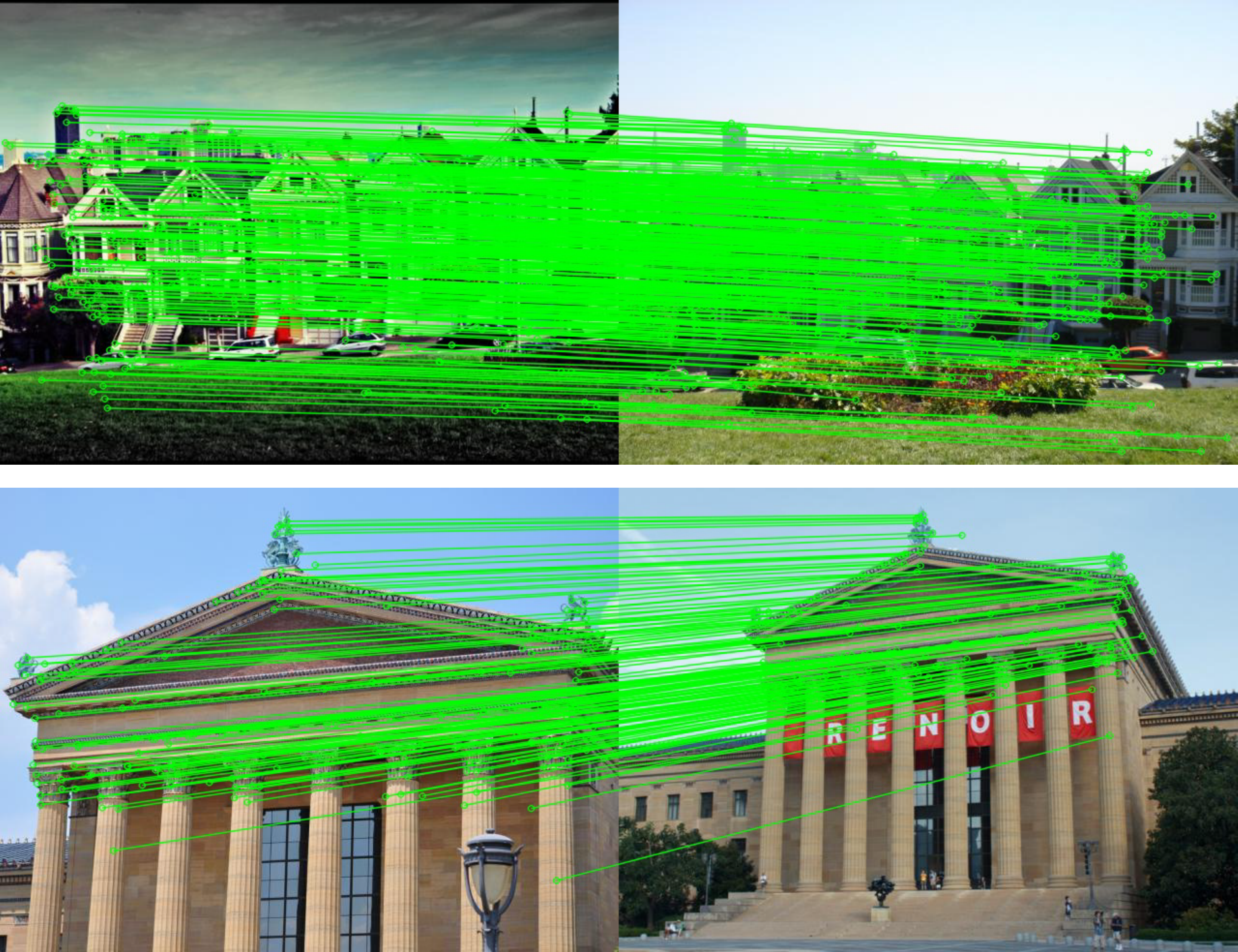}\\
			{\smaller (b) Successes by \finalcaps.}
			\myvspace
		  \end{minipage}
	   &
	   \hspace{-4mm}
	   \begin{minipage}{\mpwthree}%
			\centering%
			\includegraphics[width=\columnwidth]{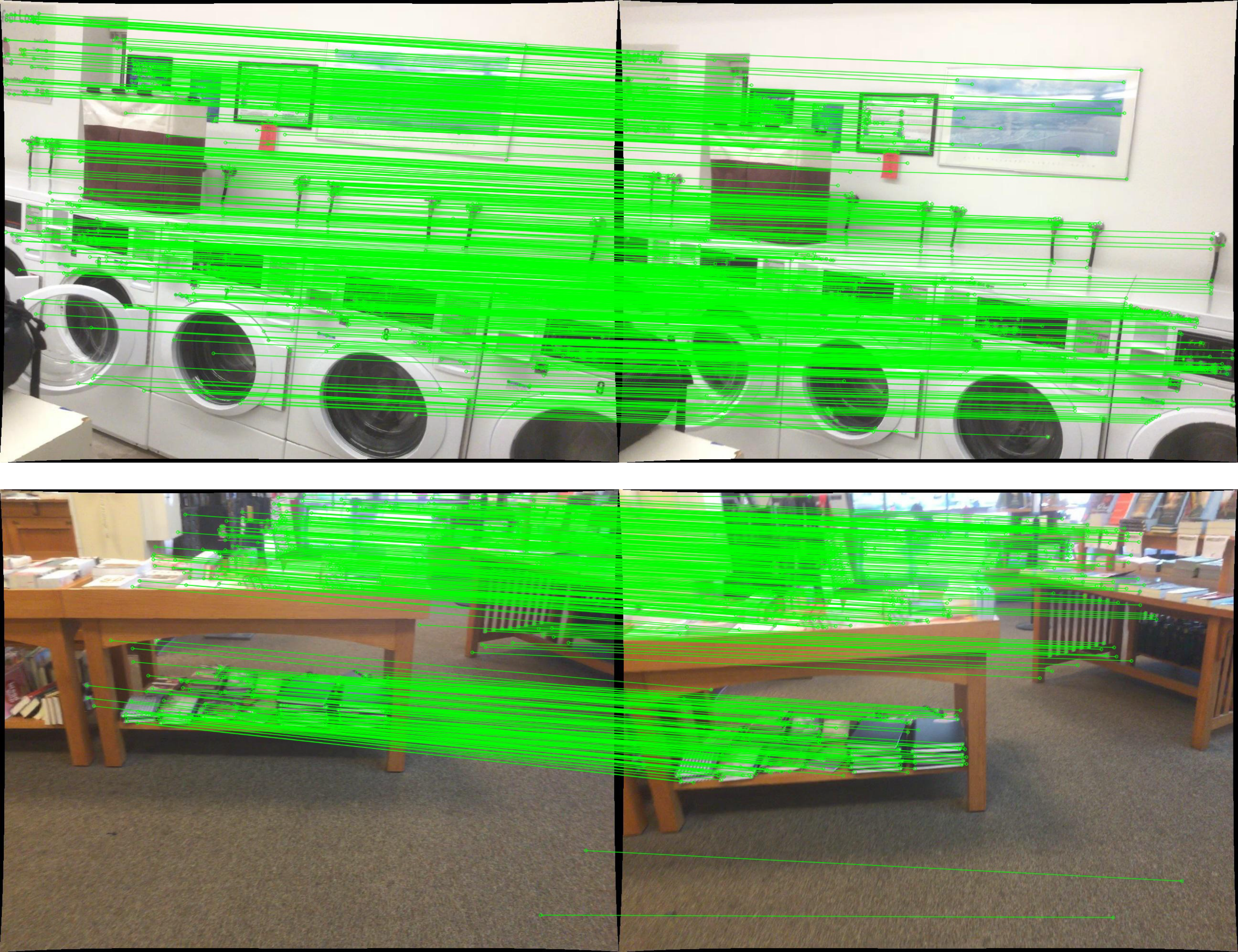}\\
			{\smaller (c) Cross-dataset generalization of \finalcaps.}
			\myvspace
		  \end{minipage} \\

		\multicolumn{3}{c}{
		\hspace{-3mm}
			\begin{minipage}{\textwidth}%
			  	\caption{Qualitative results showing the improved performance of (b) \finalcaps~over the bootstrap descriptor (a) \sift for relative pose estimation on \megadepth~\cite{Li18cvpr-megadepth}, and (c) cross-dataset generalization of \finalcaps~for relative pose estimation on the \scannet dataset~\cite{Dai17cvpr-scannet}. Green lines are inlier correspondences estimated by \ransactenk. \finalcaps~outputs reliable and dense matches. [Best viewed digitally]} \label{fig:relativepose-qualitative}
			\end{minipage}
		}
		\end{tabular}
	\end{minipage}
	\vspace{-8mm}
	\end{center}
\end{figure*}

We first provide results demonstrating successful applications of \nameshort to relative pose estimation (Section~\ref{sec:exp-relativepose}) and point cloud registration (Section~\ref{sec:exp-registration}), then report ablation studies on point cloud registration where we vary the algorithmic settings of \nameshort (Section~\ref{sec:exp-ablation}). \emph{Detailed experimental data are tabulated in the \supp}.

\subsection{Relative Pose Estimation}
\label{sec:exp-relativepose}
{\bf Setup}.
We first showcase \nameshort for Example~\ref{ex:relativepose} on the \megadepth~\cite{Li18cvpr-megadepth} benchmark containing a large collection of Internet images for the task of relative pose estimation. We adopted \ransactenk (\ie,~\ransac with maximum $10,000$ iterations) with $99.9\%$ confidence and $0.001$ inlier threshold as the teacher. We used the recently proposed \caps~\cite{wang20eccv-caps} feature learning framework as the student.\footnote{We assumed known camera intrinsics so the fundamental matrix can be computed from the essential matrix to supervise \caps.} To bootstrap \nameshort, we performed \ransactenk~with \sift~detector, \sift~descriptor, and $0.75$ ratio test to initialize the geometric pseudo-labels (\ie,~relative poses).

To speed up the training of \nameshort, we sampled $10\%$ of the original \megadepth~training set used in~\cite{wang20eccv-caps} uniformly at random, resulting in $78,836$ pairs of images \textit{without} relative pose labels. To train \caps, we modified the publicly available \caps implementation\footnote{\smaller \url{https://github.com/qianqianwang68/caps}}, adopted a smaller batch size 5, and kept the Adam optimizer with initial learning rate $10^{-4}$. We used \finetune (\cf~line~\ref{line:finetune}) for the teacher-student loop, and in every iteration, we trained \caps for $40,000$ steps. We trained \nameshort for a fixed number of $T=10$ iterations. 

In the teacher-student loop, we designed a \namefilter~that prunes pseudo-labels according to the results of \ransactenk~--~we only pass to the student pairs whose number of putative matches (either from \sift~with ratio test or \caps~with cross check) is above $100$ and whose \ransac estimated inlier rate is over $10\%$. Intuitively, pseudo-labels satisfying these two conditions are more likely to be correct. 

We name the \caps~descriptor learned from \nameshort~without ground-truth supervision as \scaps. We evaluated the performance of \scaps on (i) the \megadepth~test set, provided in~\cite{wang20eccv-caps}, including $3,000$ image pairs equally divided into \emph{easy}, \emph{moderate}, and \emph{hard} categories; (ii) the \scannet~\cite{Dai17cvpr-scannet} dataset to test cross-dataset generalization.

{\bf Results}. Fig.~\ref{fig:relativepose-line-plot} plots the dynamics of \nameshort~on \megadepth. \plsr~stands for \emph{Pseudo-Label Survival Rate} and is computed as $\abs{\calS}/\nrmeasurements \times 100\%$,~\ie,~the percentage of pseudo-labels that survived the \namefilter~(\cf~line~\ref{line:filter}). \plir stands for \emph{Pseudo-Label Inlier Rate} and denotes the percentage of correct labels in $\calS$, a number that is not used by \nameshort~but computed \emph{a posteriori} using the ground-truth labels to show that \nameshort is robust to partially incorrect labels. Besides \plsr~and \plir, Fig.~\ref{fig:relativepose-line-plot} plots the rotation recalls on both the training and the test sets (the translation recalls exhibit a similar trend and are shown in the \supp).\footnote{Recall is defined as the percentage of correctly estimated models divided by the total number of pairs. Following~\cite{wang20eccv-caps}, we say a rotation or a translation is estimated correctly if it has angular error less than $10^{\circ}$ \wrt to the groundtruth (note that translation is estimated up to scale).} The \bs~(bootstrap) iteration plots the training and test recalls using \sift. We make the following observations from Fig.~\ref{fig:relativepose-line-plot}: (i) \plsr gradually increases and approaches $90\%$ \wrt iterations, indicating that the \scaps~descriptor establishes dense correspondences with high inlier ratio, encouraged by the \namefilter; (ii) \plir remains close to $90\%$, and is always higher than the recall, indicating that the \namefilter~is effective in removing wrong labels; (iii) \scaps gradually improves itself on both the training and the test sets. (iv) While \sift~works better than \scaps~on the training set, \scaps~significantly outperforms \sift~on the test set.

Table~\ref{tab:megadepth} compares the performance of two versions of \scaps~to other SOTA methods. \finalcaps~is the \scaps~descriptor at the last iteration, while \starcaps~is the \scaps~descriptor that performs best on the \megadepth test set. We see that both versions of \scaps outperform the strong baseline using \sift with ratio test and \ransactenk, as well as the two SOTA results from the original \caps~\cite{wang20eccv-caps} using both \sift~detector and \superpoint~detector~\cite{Detone18cvprw-superpoint}.\footnote{We suspect the \ransac~in~\cite{wang20eccv-caps} is not carefully tuned.} Moreover, we report the performance of \ransactenk plus the \emph{supervised oracle}, \capsgt,~that is trained using full ground-truth supervision, on the test set. One can see that \scaps, trained using only $10\%$ of the unlabeled training set, performs on par compared with the supervised oracle. 

Fig.~\ref{fig:relativepose-qualitative} provides qualitative examples of correspondence matching results on both \megadepth and \scannet. More examples are provided in the \supp.

\subsection{Point Cloud Registration}
\label{sec:exp-registration}

\begin{figure*}[h]
	\begin{center}
	\begin{minipage}{\textwidth}
	\hspace{-4mm}
	\begin{tabular}{ccc}%
		\begin{minipage}{\mpwthree}%
			\centering%
			\includegraphics[width=\columnwidth]{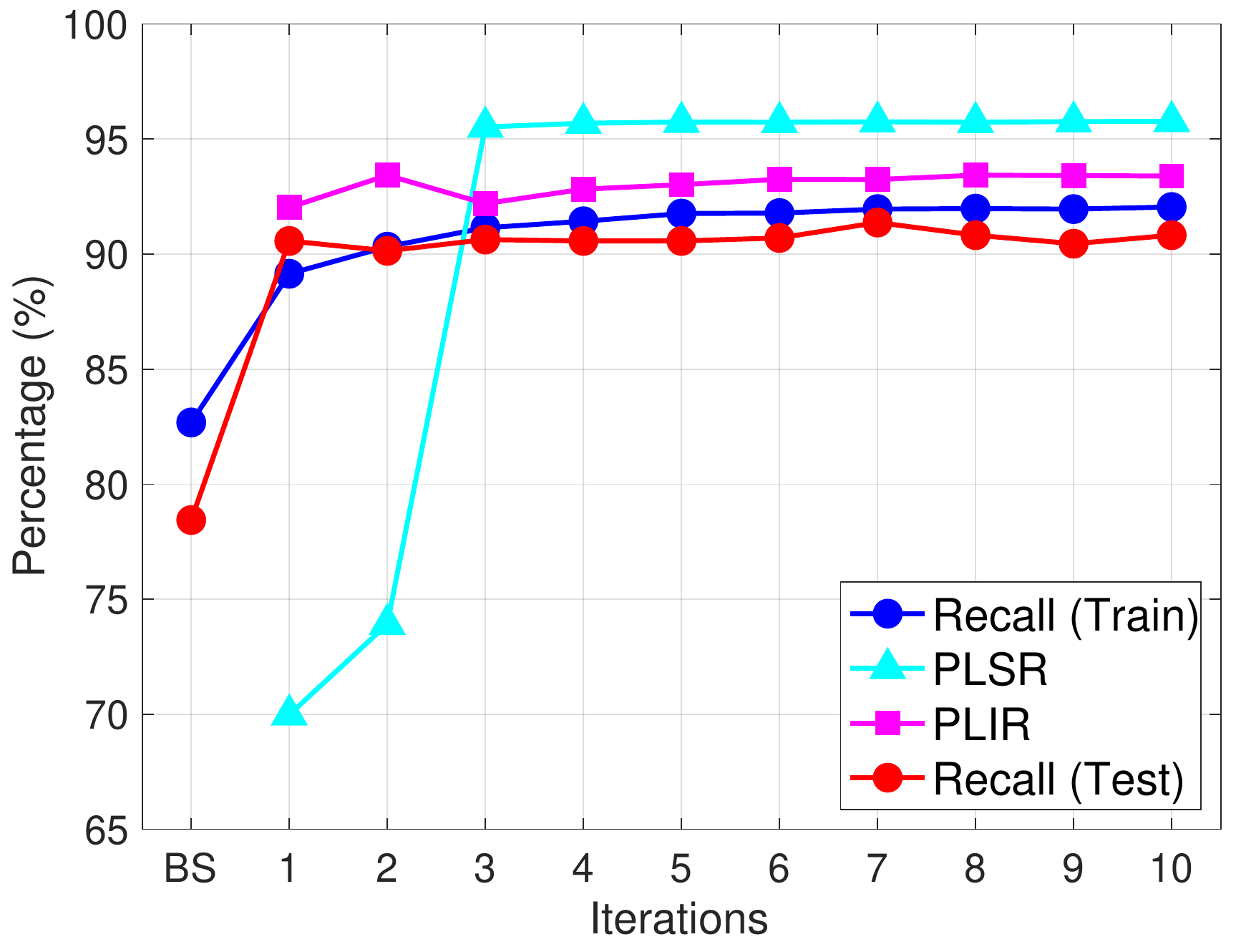} \\
			\vspace{-1mm}
			\caption{Dynamics of \nameshort on \threedmatch~\cite{Zeng17cvpr-3dmatch}. \plsr: \emph{Pseudo-Label Survival Rate}. \plir: \emph{Pseudo-Label Inlier Rate}. \bs: Boostrap.}
			\label{fig:registration-line-plot}
		  \end{minipage}
	  &
		\multicolumn{2}{c}{%
		\hspace{-10mm}
			\begin{minipage}{\mpwthreetwo}%

\renewcommand{\arraystretch}{1.4}

\adjustbox{max width=11.5cm}{%
\hspace{3mm}
\begin{tabular}{ccccccccc|c}
Methods     & \shortstack{Kitchen \\(\%)} & \shortstack{Home 1 \\ (\%)}  & \shortstack{ Home 2\\ (\%)} & \shortstack{Hotel 1\\ (\%)} & \shortstack{Hotel 2 \\  (\%)} & \shortstack{ Hotel 3\\ (\%)} & \shortstack{ Study \\ (\%)} & \shortstack{ MIT \\(\%)} & \shortstack{Overall\\ (\%)} \\
\hline
\!\!\fpfh\!+\ransactenk~\cite{rusu2009icra}\footnote{\smaller \fpfh implemented in \openthreed~\cite{Zhou18arxiv-open3D}. All \ransac use $99.9\%$ confidence.} \!\!& $80.6$         & $84.6$      & $69.2$     & $88.1$      & $76.9$       & $\bm{88.9}$       & $71.2$        & $70.1$   & $78.4$                \\
\!\!\fcgf~\cite{Choy19iccv-FCGF}\footnote{\smaller Recall statistics adapted from the original \fcgf paper~\cite{Choy19iccv-FCGF} evaluated with the criteria defined by \threedmatch.} \!\!& $93.0$         & $91.0$        & $71.0$      & $91.0$       & $87.0$       & $69.0$      & $75.0$         & $80.0$   & $82.0$               \\
\!\!\dgr~\cite{Choy20cvpr-deepGlobalRegistration}\!\! & $94.5$        & $89.7$       & $77.9$      & $92.9$       & $85.6$      & $79.6$      & $69.9$         & $72.7$   & $85.2$                  \\
\!\!\dgr\!+\ransaceightyk~\cite{Choy20cvpr-deepGlobalRegistration} \!\! & $\bm{98.8}$         & $96.2$        & $\bm{81.7}$     & $97.3$       & $91.2$       & $87.0$       & $81.9$         & $79.2$   &  $91.3$              \\
\hline
\!\!\fcgfgt+\ransactenk\footnote{\smaller Recall computed by using \ransactenk with the pretrained \fcgfgt (\ie,~the supervised oracle).}\!\! & $97.2$         & $\bm{97.4}$        & $77.9$      & $97.8$       & $\bm{91.3}$      & $83.3$      & $86.3$         & $76.6$   & $91.1$                   \\
\hline
\hline
\!\!\finalfcgf+\ransactenk\!\! & $98.4$         & $94.2$       & $75.0$     & $\bm{98.7}$       & $89.4$       & $79.6$       & $87.3$         & $76.6$   & $90.8$                 \\
\!\!\starfcgf+\ransactenk\!\! & $98.0$        & $94.2$        & $76.0$     & $\bm{98.7}$      & $90.4$       & $85.2$       & $\bm{88.0}$          & $\bm{80.5}$   &  $\bm{91.4}$                  \\
\end{tabular}
}
\\
\captionof{table}{Scene-wise and overall recalls on the \threedmatch~\cite{Zeng17cvpr-3dmatch} test dataset using different methods. \finalfcgf: last \fcgf trained by \nameshort. \starfcgf: best \fcgf trained by \nameshort.
\label{tab:threedmatch}}
			\end{minipage}
		} \\


	  \begin{minipage}{\mpwthree}%
			\centering%
			\includegraphics[width=\columnwidth]{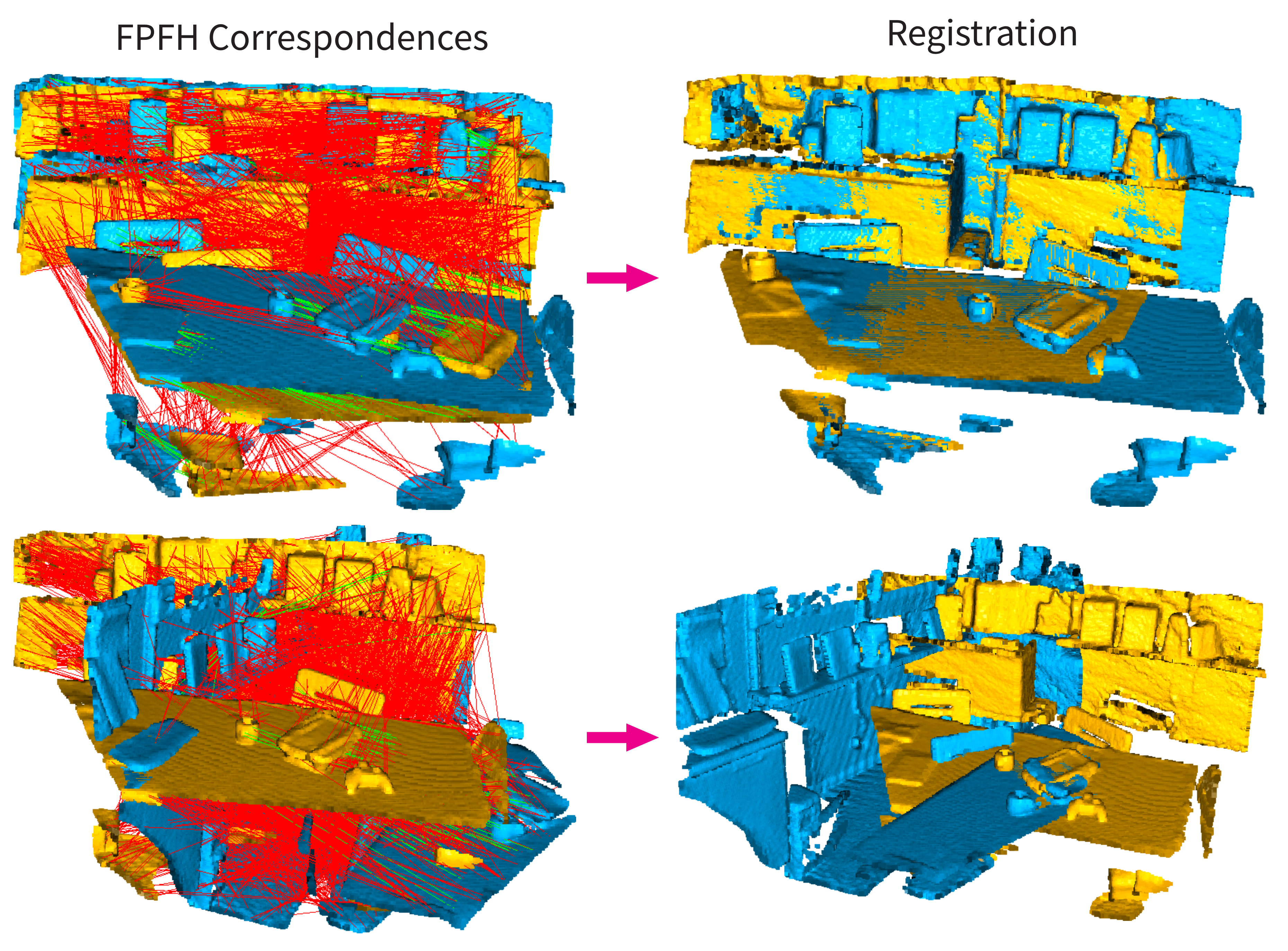}\\
			{\smaller (a) Success (top) and failure (bottom) by \fpfh.}
			\myvspace
		  \end{minipage}
	  &
		\begin{minipage}{\mpwthree}%
			\centering%
			\includegraphics[width=\columnwidth]{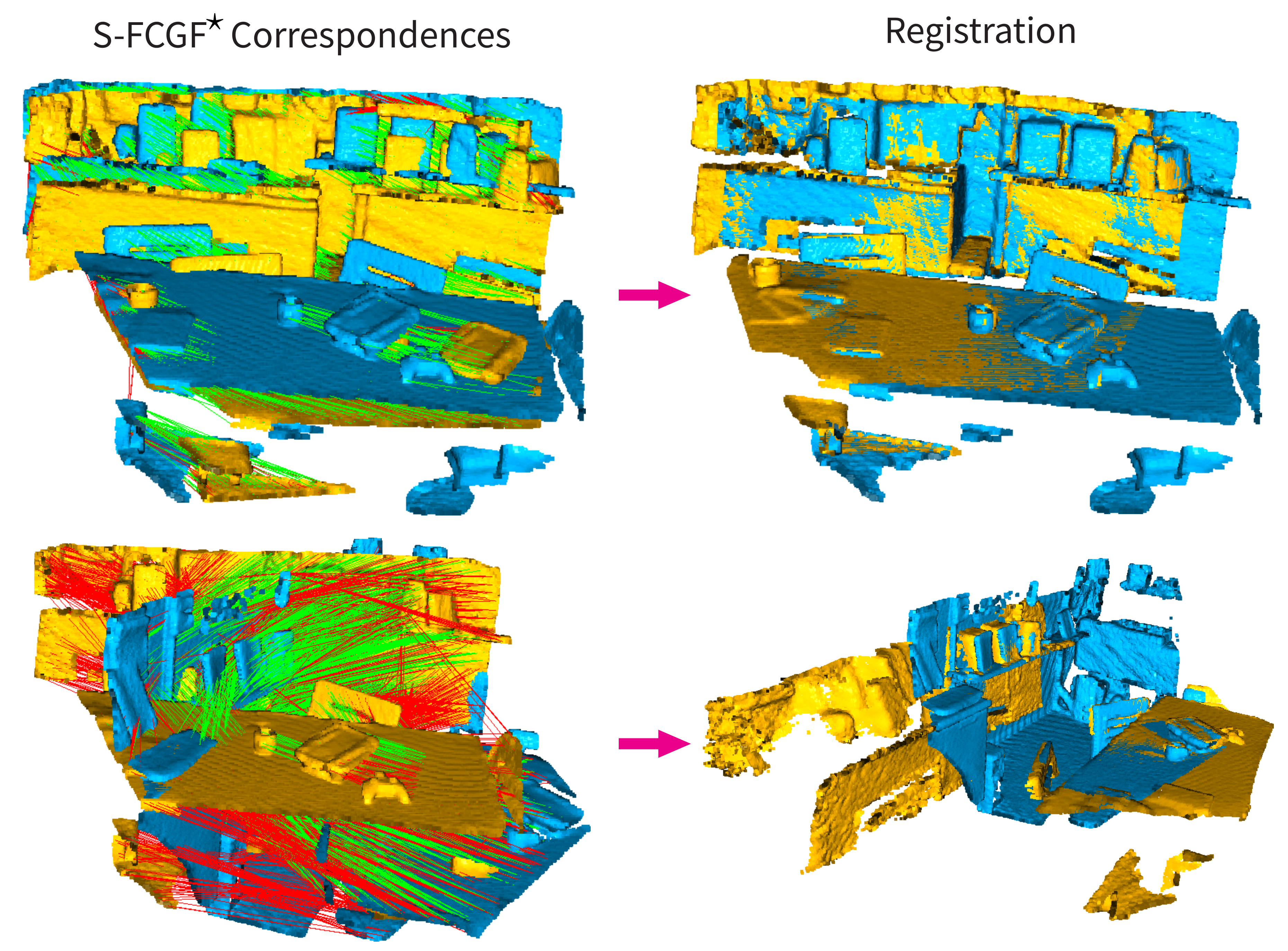}\\
			{\smaller (b) Successes by \starfcgf.}
			\myvspace
		  \end{minipage}
	   &
	   \hspace{-5mm}
	   \begin{minipage}{\mpwthree}%
			\centering%
			\includegraphics[width=\columnwidth]{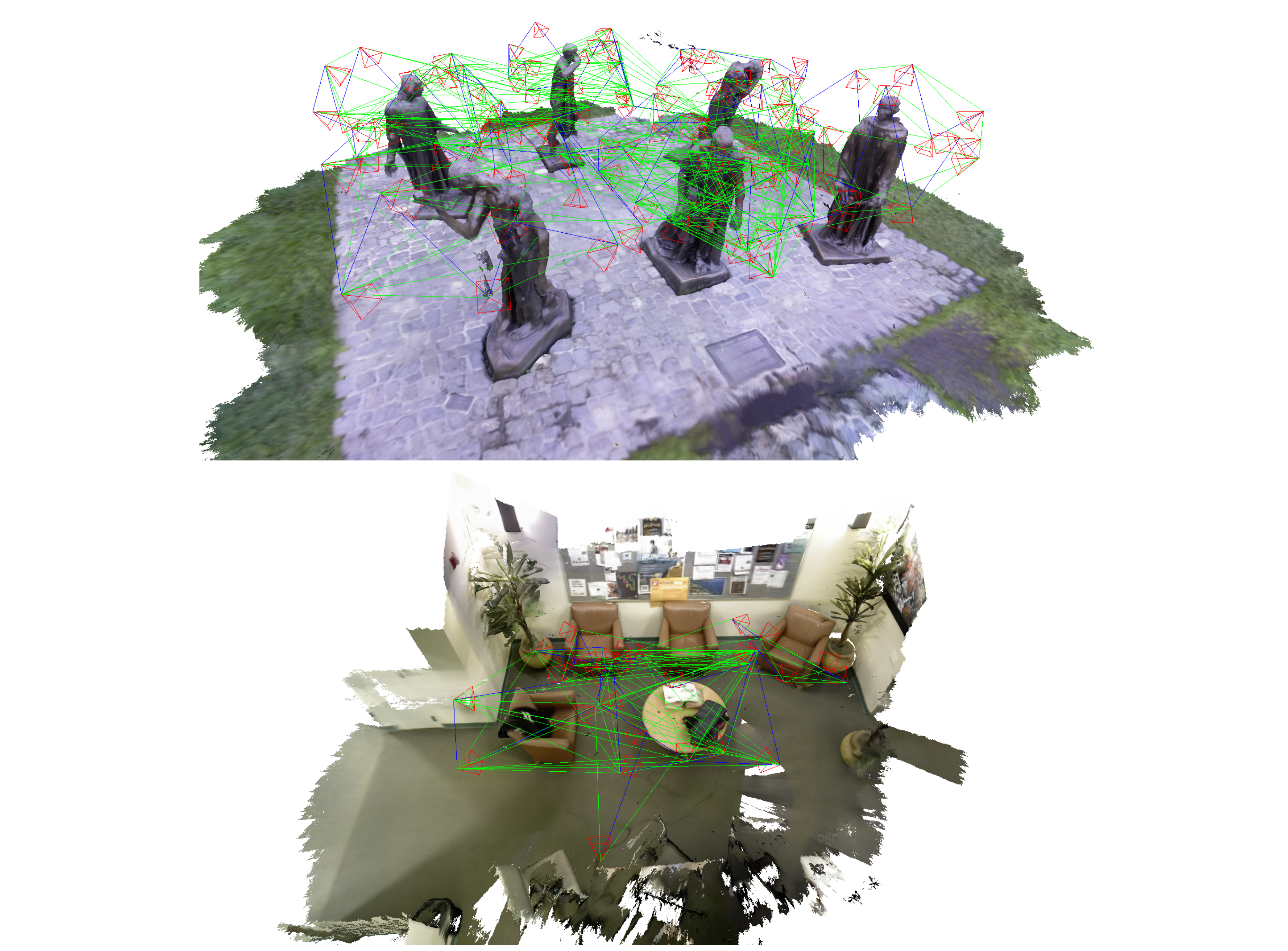}\\
			{\smaller (c) Multi-way registration using \starfcgf.}
			\myvspace
		  \end{minipage} \\

		\multicolumn{3}{c}{
		\hspace{-5mm}
			\begin{minipage}{\textwidth}%
			  	\caption{Qualitative results showing the improved performance of (b) \starfcgf~over the bootstrap descriptor (a) \fpfh for pairwise registration on \threedmatch~\cite{Zeng17cvpr-3dmatch}, and (c) cross-dataset generalization of \starfcgf~for multi-way registration on the \multiwayname dataset~\cite{Choi15cvpr-robustrecon}. In (a)-(b), the top pair has overlap ratio $89\%$, the bottom pair has overlap ratio $50\%$. Green lines: inlier correspondences. Red lines: outlier correspondences. In (c), top: \emph{Lounge}, bottom: \emph{Burghers}. Blue lines: odometry. Green lines: loop closures. [Best viewed digitally]} \label{fig:registration-qualitative}
			\end{minipage}
		}

		\end{tabular}
	\end{minipage}
	\vspace{-8mm}
	\end{center}
\end{figure*}

{\bf Setup}.
To demonstrate \nameshort for Example~\ref{ex:pointcloudregistration}, we conducted experiments on \threedmatch~\cite{Zeng17cvpr-3dmatch}, a benchmark containing point clouds of real-world indoor scenes. 
We used \ransactenk~(with 7cm inlier threshold) plus \icp~\cite{Besl92pami-icp} as the teacher, \fcgf~\cite{Choy19iccv-FCGF} as the student, and \fpfh~\cite{rusu2009icra} as the bootstrap descriptor to initialize transformation labels.

\nameshort was trained on the training set provided by \dgr~\cite{Choy20cvpr-deepGlobalRegistration} containing $9,856$ pairs of scans, \textit{without} ground-truth transformation labels. Input point clouds were all voxelized with 5cm resolution before feature extraction (both \fpfh and \fcgf) and registration.
To train \fcgf, we followed the configuration of the original \fcgf and used SGD with initial learning rate 0.1.\footnote{\smaller \url{https://github.com/chrischoy/FCGF}} In the teacher-student loop, we used \finetune, where we train \fcgf for $100$ epochs at iteration 1 and $50$ epochs for the rest of the iterations. We designed a \namefilter~based on estimated overlap ratio,~\ie,~only pairs with estimated overlap ratio over $\eta$ are passed to \fcgf. We set $\eta=30\%$ for the first two iterations and $\eta=10\%$ for the rest. \nameshort is trained for $T=10$ iterations. 

We name the \fcgf descriptor learned from \nameshort without ground-truth supervision as \sfcgf. We evaluated the performance of \sfcgf on (i) the \threedmatch test set including $1,623$ pairs; and (ii) the unseen \multiwayname dataset~\cite{Choi15cvpr-robustrecon} for multi-way registration~\cite{Zhou18arxiv-open3D}.

{\bf Results}. Fig.~\ref{fig:registration-line-plot} plots the dynamics of \nameshort on \threedmatch. We observe that: (i) \plsr increases and approaches $96\%$, indicating that more pairs enter the noisy student training; (ii) \plir remains close to $93\%$, and is always higher than the recall, showing the effect of the \namefilter; (iii) \sfcgf gradually improves itself on both training and test sets.

Table~\ref{tab:threedmatch} compares the performance of \finalfcgf and \starfcgf~to other SOTA methods.\footnote{Following~\cite{Choy20cvpr-deepGlobalRegistration}, we say a registration is successful if rotation error is below $15^{\circ}$ and translation error is below 30cm.} We see that \starfcgf~outperforms the baseline \fpfh, \fcgf~\cite{Choy19iccv-FCGF}, and the recently proposed \dgr (even with \ransaceightyk)~\cite{Choy20cvpr-deepGlobalRegistration}. We also provide results using \ransactenk plus the \emph{supervised oracle}, \fcgfgt,~that is trained using full ground-truth supervision. \starfcgf~outperforms the supervised oracle, while \finalfcgf~achieves similar performance.

Fig.~\ref{fig:registration-qualitative} shows qualitative results using \sfcgf for pairwise registration on \threedmatch and for multi-way registration on \multiwayname. More qualitative results are shown in the \supp.

\subsection{Ablation Study}
\label{sec:exp-ablation}
\newcommand{\mpwtwo}{4.5cm}
\renewcommand{\myvspace}{\vspace{2mm}}
\newcommand{\myvspaceone}{\vspace{-1mm}}

\begin{figure}[h]
	\begin{center}
	\begin{minipage}{\columnwidth}
	\hspace{-6mm}
	\begin{tabular}{cc}%
	  \begin{minipage}{\mpwtwo}%
			\centering%
			\includegraphics[width=\columnwidth]{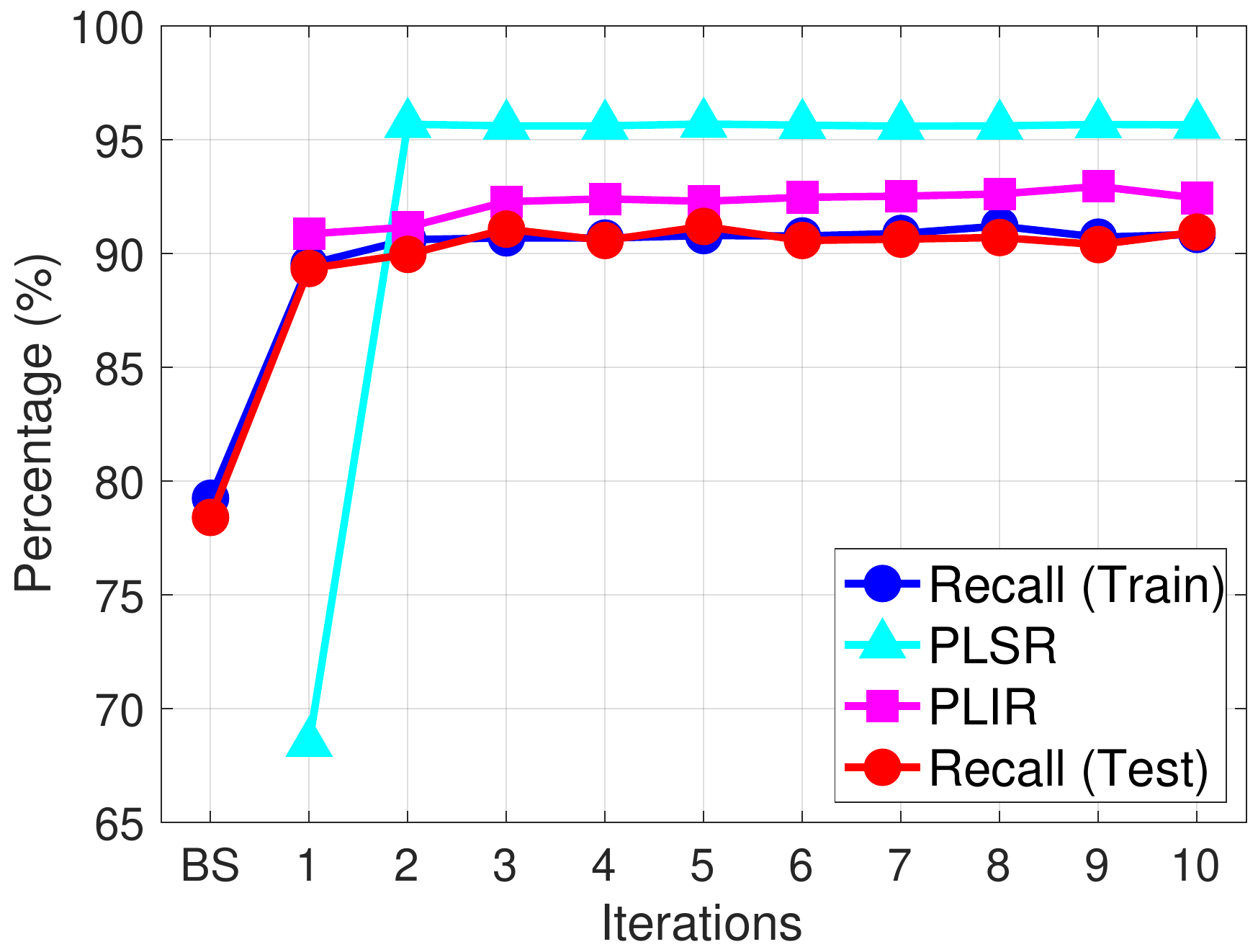}\\
			\myvspaceone
			{\smaller (a) \retrain = \true.}
			\myvspace
		  \end{minipage}
	  &
	  \hspace{-5mm}
		\begin{minipage}{\mpwtwo}%
			\centering%
			\includegraphics[width=\columnwidth]{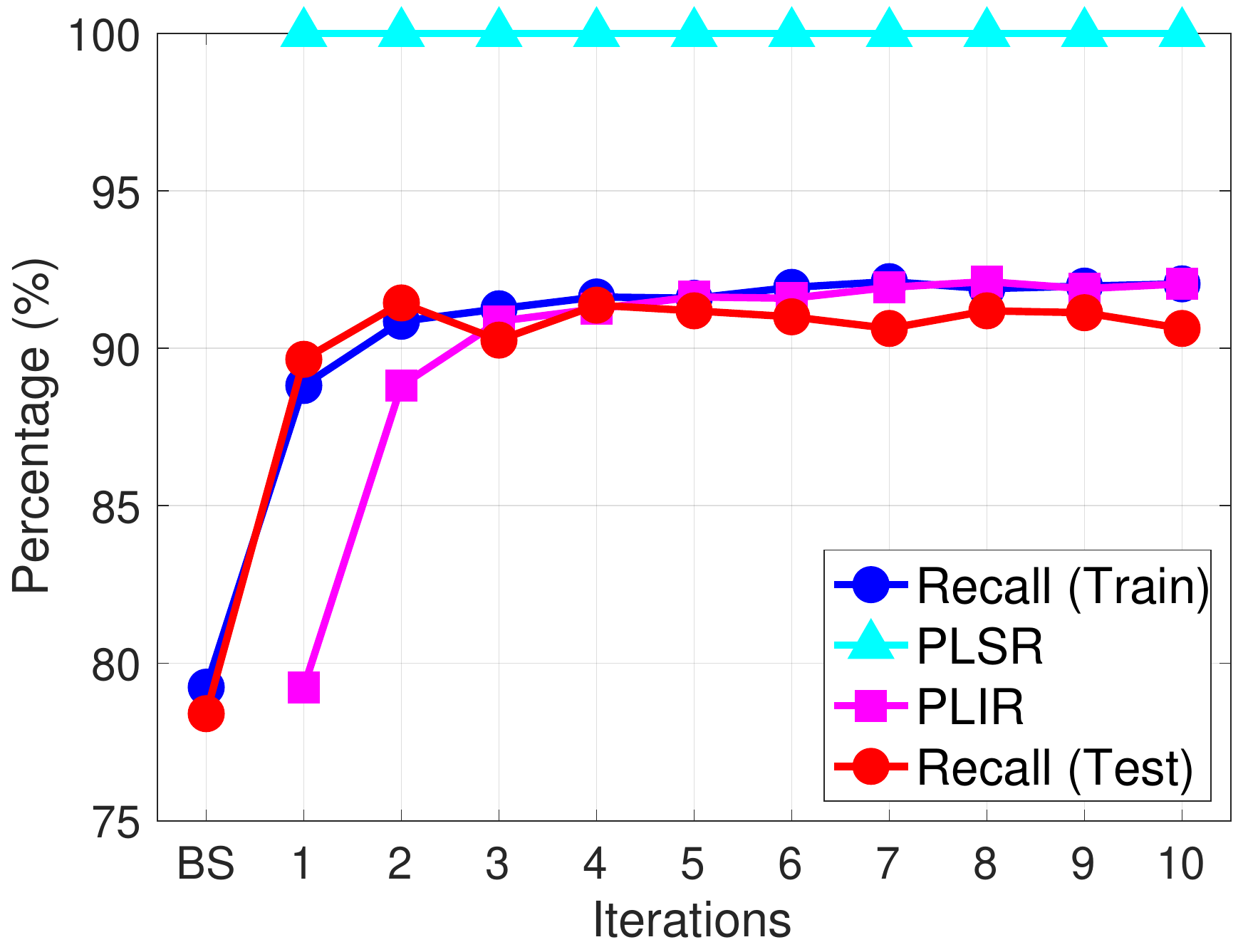}\\
			\myvspaceone
			{\smaller (b) \filterlabel = \false.}
			\myvspace
		  \end{minipage} \\

		\multicolumn{2}{c}{
			\begin{minipage}{\textwidth}%
			  	\caption{ Dynamics of \nameshort on \threedmatch~\cite{Zeng17cvpr-3dmatch} with (a) \retrain instead of \finetune (line~\ref{line:randinit}); (b) the \filter (line~\ref{line:nofilter}) turned off. \nameshort still achieves over $91\%$ overall recall on the test set. \label{fig:ablation-dynamics} }
			\end{minipage}
		}
		\end{tabular}
	\end{minipage}
	\vspace{-2mm}
	\end{center}
\end{figure}
We first study the effect of using \retrain vs \finetune in \nameshort for point cloud registration. We used the same setup as in Section~\ref{sec:exp-registration}, except that we changed from \finetune to \retrain, where in each iteration, we initialized the weights of \fcgf at random and trained it for $100$ epochs. We also set the \namefilter~overlap ratio $\eta = 10\%$ for all iterations. Fig.~\ref{fig:ablation-dynamics}(a) plots the corresponding dynamics, which overall looks similar to Fig.~\ref{fig:registration-line-plot}. The \finetune train recall is slightly higher and more stable than the \retrain train recall, due to the ``continuous'' weight update nature of \finetune. \nameshort with \retrain achieves similar performance on the test set: \starfcgf~has overall recall {$91.2\%$} and \finalfcgf~has overall recall {$90.9\%$}. 

We then study the effect of the \namefilter~by running \nameshort on \threedmatch without verification,~\ie,~we set $\eta = 0$. As shown in Fig.~\ref{fig:ablation-dynamics}(b), \plsr is always $100\%$. Despite higher noise in the pseudo-labels, the performance of \nameshort remains unaffected on the test set: \starfcgf~has overall recall $91.4\%$ and \finalfcgf~has overall recall $90.6\%$.

In the \supp, we provide two more ablation studies on \threedmatch: (i) we trained \nameshort on the small test set and tested \sfcgf on the large training set, to show better generalization of a large training set; (ii) we replaced \ransactenk with a non-robust registration solver as the teacher to show the importance of a robust solver.
\section{Conclusion}
\label{sec:conclusion}
We proposed \nameshort, the first general framework for feature learning in \topicllc without any supervision from ground-truth geometric labels. \nameshort iteratively performs robust estimation of the geometric models to generate pseudo-labels, and feature learning under the supervision of the noisy pseudo-labels.
We applied \nameshort to camera pose estimation and point cloud registration, demonstrating performance that is on par or even superior to supervised oracles in large-scale real datasets.

Future research includes (i) increasing the training recall towards 100\%; (ii) differentiating the robust estimation layer~\cite{Gould19arXiv-DDN}; (iii) designing an optimality-based~\cite{Yang20nips-certifiablePerception} and learnable verifier based on \emph{cycle consistency}~\cite{Huang19cvpr-transformationsync,Gojcic20cvpr-multiviewRegistration,Mangelson18icra-pcm}; (iv) speeding up the teacher-student loop; (iv) forming image and point cloud pairs using \emph{image retrieval}~\cite{Toft20PAMI-LongtermLocalization,Cummins08ijrr-fabmap}.



{\small
\bibliographystyle{ieee_fullname}
\bibliography{refs}

\begin{thebibliography}{10}\itemsep=-1pt

\bibitem{Antonante20arxiv-outlierrobust}
Pasquale Antonante, Vasileios Tzoumas, Heng Yang, and Luca Carlone.
\newblock Outlier-robust estimation: Hardness, minimally-tuned algorithms, and
  applications.
\newblock {\em arXiv preprint arXiv:2007.15109}, 2020.

\bibitem{Ask13-optimalTruncatedL2}
Erik Ask, Olof Enqvist, and Fredrik Kahl.
\newblock Optimal geometric fitting under the truncated l2-norm.
\newblock In {\em IEEE Conf. on Computer Vision and Pattern Recognition
  (CVPR)}, pages 1722--1729, 2013.

\bibitem{bai2020cvpr}
Xuyang Bai, Zixin Luo, Lei Zhou, Hongbo Fu, Long Quan, and Chiew-Lan Tai.
\newblock D3feat: Joint learning of dense detection and description of 3d local
  features.
\newblock In {\em IEEE Conf. on Computer Vision and Pattern Recognition
  (CVPR)}, 2020.

\bibitem{Barath18cvpr-gcransac}
Daniel Barath and Ji{\v{r}}{\'\i} Matas.
\newblock {Graph-cut RANSAC}.
\newblock In {\em IEEE Conf. on Computer Vision and Pattern Recognition
  (CVPR)}, pages 6733--6741, 2018.

\bibitem{Barath20cvpr-magsac++}
Daniel Barath, Jana Noskova, Maksym Ivashechkin, and Jiri Matas.
\newblock {MAGSAC++}, a fast, reliable and accurate robust estimator.
\newblock In {\em IEEE Conf. on Computer Vision and Pattern Recognition
  (CVPR)}, pages 1304--1312, 2020.

\bibitem{Bazin12pami-BnBGrouping}
Jean-Charles Bazin, Hongdong Li, In~So Kweon, C{\'e}dric Demonceaux, Pascal
  Vasseur, and Katsushi Ikeuchi.
\newblock A branch-and-bound approach to correspondence and grouping problems.
\newblock {\em {IEEE} Trans. Pattern Anal. Machine Intell.}, 35(7):1565--1576,
  2012.

\bibitem{Bazin12accv-globalRotSearch}
J.~C. Bazin, Y. Seo, and M. Pollefeys.
\newblock Globally optimal consensus set maximization through rotation search.
\newblock In {\em Asian Conference on Computer Vision}, pages 539--551.
  Springer, 2012.

\bibitem{Bertsekas99book-nonlinearprogramming}
Dimitri Bertsekas.
\newblock {\em Nonlinear programming}.
\newblock Athena Scientific, 1999.

\bibitem{Besl92pami-icp}
P.~J. Besl and N.~D. McKay.
\newblock A method for registration of {3-D} shapes.
\newblock {\em {IEEE} Trans. Pattern Anal. Machine Intell.}, 14(2), 1992.

\bibitem{Bosse17fnt}
M. Bosse, G. Agamennoni, and I. Gilitschenski.
\newblock Robust estimation and applications in robotics.
\newblock {\em Foundations and Trends in Robotics}, 4(4):225--269, 2016.

\bibitem{Brachmann14eccv-occulinemod}
Eric Brachmann, Alexander Krull, Frank Michel, Stefan Gumhold, Jamie Shotton,
  and Carsten Rother.
\newblock Learning 6d object pose estimation using 3d object coordinates.
\newblock In {\em European Conf. on Computer Vision (ECCV)}, pages 536--551.
  Springer, 2014.

\bibitem{bradski08book-opencv}
Gary Bradski and Adrian Kaehler.
\newblock {\em Learning OpenCV: Computer vision with the OpenCV library}.
\newblock " O'Reilly Media, Inc.", 2008.

\bibitem{Cadena16tro-slam}
Cesar Cadena, Luca Carlone, Henry Carrillo, Yasir Latif, Davide Scaramuzza,
  Jos{\'e} Neira, Ian Reid, and John~J Leonard.
\newblock Past, present, and future of simultaneous localization and mapping:
  Toward the robust-perception age.
\newblock {\em {IEEE} Trans. Robotics}, 32(6):1309--1332, 2016.

\bibitem{Cai19ICCV-CMtreeSearch}
Zhipeng Cai, Tat-Jun Chin, and Vladlen Koltun.
\newblock Consensus maximization tree search revisited.
\newblock In {\em Intl. Conf. on Computer Vision (ICCV)}, pages 1637--1645,
  2019.

\bibitem{chen19ICCVW-satellitePoseEstimation}
Bo Chen, Jiewei Cao, Alvaro Parra, and Tat-Jun Chin.
\newblock Satellite pose estimation with deep landmark regression and nonlinear
  pose refinement.
\newblock In {\em Proceedings of the IEEE International Conference on Computer
  Vision Workshops}, 2019.

\bibitem{Chin15-CMTreeAstar}
Tat-Jun Chin, Pulak Purkait, Anders Eriksson, and David Suter.
\newblock Efficient globally optimal consensus maximisation with tree search.
\newblock In {\em Proceedings of the IEEE Conference on Computer Vision and
  Pattern Recognition}, pages 2413--2421, 2015.

\bibitem{Chin17slcv-maximumConsensusAdvances}
T.~J. Chin and D. Suter.
\newblock The maximum consensus problem: recent algorithmic advances.
\newblock {\em Synthesis Lectures on Computer Vision}, 7(2):1--194, 2017.

\bibitem{Choi15cvpr-robustrecon}
Sungjoon Choi, Qian-Yi Zhou, and Vladlen Koltun.
\newblock Robust reconstruction of indoor scenes.
\newblock In {\em IEEE Conf. on Computer Vision and Pattern Recognition
  (CVPR)}, pages 5556--5565, 2015.

\bibitem{choi2016arxiv}
Sungjoon Choi, Qian-Yi Zhou, Stephen Miller, and Vladlen Koltun.
\newblock A large dataset of object scans.
\newblock {\em arXiv:1602.02481}, 2016.

\bibitem{Choy20cvpr-deepGlobalRegistration}
Christopher Choy, Wei Dong, and Vladlen Koltun.
\newblock Deep global registration.
\newblock In {\em IEEE Conf. on Computer Vision and Pattern Recognition
  (CVPR)}, 2020.

\bibitem{Choy16neurips-UCN}
Christopher Choy, JunYoung Gwak, Silvio Savarese, and Manmohan Chandraker.
\newblock Universal correspondence network.
\newblock In {\em Advances in Neural Information Processing Systems}, pages
  2414--2422, 2016.

\bibitem{Choy19iccv-FCGF}
Christopher Choy, Jaesik Park, and Vladlen Koltun.
\newblock Fully convolutional geometric features.
\newblock In {\em Intl. Conf. on Computer Vision (ICCV)}, pages 8958--8966,
  2019.

\bibitem{Cummins08ijrr-fabmap}
Mark Cummins and Paul Newman.
\newblock Fab-map: Probabilistic localization and mapping in the space of
  appearance.
\newblock {\em Intl. J. of Robotics Research}, 27(6):647--665, 2008.

\bibitem{Dai17cvpr-scannet}
Angela Dai, Angel~X Chang, Manolis Savva, Maciej Halber, Thomas Funkhouser, and
  Matthias Nie{\ss}ner.
\newblock Scannet: Richly-annotated 3d reconstructions of indoor scenes.
\newblock In {\em IEEE Conf. on Computer Vision and Pattern Recognition
  (CVPR)}, pages 5828--5839, 2017.

\bibitem{Detone18cvprw-superpoint}
Daniel DeTone, Tomasz Malisiewicz, and Andrew Rabinovich.
\newblock Superpoint: Self-supervised interest point detection and description.
\newblock In {\em Proceedings of the IEEE Conference on Computer Vision and
  Pattern Recognition Workshops}, pages 224--236, 2018.

\bibitem{Dong19iros-gpuRpbustScene}
Wei Dong, Jaesik Park, Yi Yang, and Michael Kaess.
\newblock Gpu accelerated robust scene reconstruction.
\newblock In {\em IEEE/RSJ Intl. Conf. on Intelligent Robots and Systems
  (IROS)}, pages 7863--7870. IEEE, 2019.

\bibitem{dosovitskiy2017arxiv}
Alexey Dosovitskiy, German Ros, Felipe Codevilla, Antonio Lopez, and Vladlen
  Koltun.
\newblock Carla: An open urban driving simulator.
\newblock {\em arXiv preprint arXiv:1711.03938}, 2017.

\bibitem{En18eccvW-rpnet}
Sovann En, Alexis Lechervy, and Fr{\'e}d{\'e}ric Jurie.
\newblock Rpnet: An end-to-end network for relative camera pose estimation.
\newblock In {\em Proceedings of the European Conference on Computer Vision
  (ECCV) Workshops}, pages 0--0, 2018.

\bibitem{Enqvist12eccv-robustFitting}
O. Enqvist, E. Ask, F. Kahl, and K. {\AA}str{\"o}m.
\newblock Robust fitting for multiple view geometry.
\newblock In {\em European Conf. on Computer Vision (ECCV)}, pages 738--751.
  Springer, 2012.

\bibitem{Fischler81}
M. Fischler and R. Bolles.
\newblock Random sample consensus: a paradigm for model fitting with
  application to image analysis and automated cartography.
\newblock {\em Commun. ACM}, 24:381--395, 1981.

\bibitem{florence18corl-denseobjectnets}
Peter~R Florence, Lucas Manuelli, and Russ Tedrake.
\newblock Dense object nets: Learning dense visual object descriptors by and
  for robotic manipulation.
\newblock In {\em Conference on Robot Learning (CoRL)}, 2018.

\bibitem{gao03PAMI-P3P}
Xiao-Shan Gao, Xiao-Rong Hou, Jianliang Tang, and Hang-Fei Cheng.
\newblock Complete solution classification for the perspective-three-point
  problem.
\newblock {\em {IEEE} Trans. Pattern Anal. Machine Intell.}, 25(8):930--943,
  2003.

\bibitem{godard2019iccv}
Cl{\'e}ment Godard, Oisin Mac~Aodha, Michael Firman, and Gabriel~J Brostow.
\newblock Digging into self-supervised monocular depth estimation.
\newblock In {\em Intl. Conf. on Computer Vision (ICCV)}, pages 3828--3838,
  2019.

\bibitem{gojcic2019cvpr}
Zan Gojcic, Caifa Zhou, Jan~Dirk Wegner, and Wieser Andreas.
\newblock The perfect match: 3d point cloud matching with smoothed densities.
\newblock In {\em IEEE Conf. on Computer Vision and Pattern Recognition
  (CVPR)}, 2019.

\bibitem{Gojcic20cvpr-multiviewRegistration}
Zan Gojcic, Caifa Zhou, Jan~D Wegner, Leonidas~J Guibas, and Tolga Birdal.
\newblock Learning multiview 3d point cloud registration.
\newblock In {\em IEEE Conf. on Computer Vision and Pattern Recognition
  (CVPR)}, pages 1759--1769, 2020.

\bibitem{Goodfellow16book-deeplearning}
Ian Goodfellow, Yoshua Bengio, Aaron Courville, and Yoshua Bengio.
\newblock {\em Deep learning}, volume~1.
\newblock MIT press Cambridge, 2016.

\bibitem{Gould19arXiv-DDN}
Stephen Gould, Richard Hartley, and Dylan Campbell.
\newblock Deep declarative networks: A new hope.
\newblock {\em arXiv preprint arXiv:1909.04866}, 2019.

\bibitem{grandvalet2005semi}
Yves Grandvalet and Yoshua Bengio.
\newblock Semi-supervised learning by entropy minimization.
\newblock In {\em Advances in Neural Information Processing Systems (NIPS)},
  pages 529--536, 2005.

\bibitem{hartley2004book}
Richard Hartley and Andrew Zisserman.
\newblock {\em Multiple View Geometry in Computer Vision}.
\newblock Cambridge University Press, ISBN: 0521540518, second edition, 2004.

\bibitem{horn87josa}
Berthold K.~P. Horn.
\newblock Closed-form solution of absolute orientation using unit quaternions.
\newblock {\em J. Opt. Soc. Amer.}, 4(4):629--642, Apr 1987.

\bibitem{Huang19cvpr-transformationsync}
Xiangru Huang, Zhenxiao Liang, Xiaowei Zhou, Yao Xie, Leonidas~J Guibas, and
  Qixing Huang.
\newblock Learning transformation synchronization.
\newblock In {\em IEEE Conf. on Computer Vision and Pattern Recognition
  (CVPR)}, pages 8082--8091, 2019.

\bibitem{Izatt17isrr-MIPregistration}
G. Izatt, H. Dai, and R. Tedrake.
\newblock Globally optimal object pose estimation in point clouds with
  mixed-integer programming.
\newblock In {\em Proc. of the Intl. Symp. of Robotics Research (ISRR)}, 2017.

\bibitem{jiang2018eccv}
Huaizu Jiang, Gustav Larsson, Michael Maire Greg~Shakhnarovich, and Erik
  Learned-Miller.
\newblock Self-supervised relative depth learning for urban scene
  understanding.
\newblock In {\em European Conf. on Computer Vision (ECCV)}, pages 19--35,
  2018.

\bibitem{jing2020pami}
Longlong Jing and Yingli Tian.
\newblock Self-supervised visual feature learning with deep neural networks: A
  survey.
\newblock {\em {IEEE} Trans. Pattern Anal. Machine Intell.}, 2020.

\bibitem{Khoury17iccv-CGF}
Marc Khoury, Qian{-}Yi Zhou, and Vladlen Koltun.
\newblock Learning compact geometric features.
\newblock In {\em Intl. Conf. on Computer Vision (ICCV)}, pages 153--161, 2017.

\bibitem{Klein07-ptam}
Georg Klein and David Murray.
\newblock Parallel tracking and mapping for small ar workspaces.
\newblock In {\em 2007 6th IEEE and ACM international symposium on mixed and
  augmented reality}, pages 225--234. IEEE, 2007.

\bibitem{kneip2014ECCV-UPnP}
Laurent Kneip, Hongdong Li, and Yongduek Seo.
\newblock {UPnP}: An optimal {o}(n) solution to the absolute pose problem with
  universal applicability.
\newblock In {\em European Conf. on Computer Vision (ECCV)}, pages 127--142.
  Springer, 2014.

\bibitem{ledig2017photo}
Christian Ledig, Lucas Theis, Ferenc Husz{\'a}r, Jose Caballero, Andrew
  Cunningham, Alejandro Acosta, Andrew Aitken, Alykhan Tejani, Johannes Totz,
  Zehan Wang, et~al.
\newblock Photo-realistic single image super-resolution using a generative
  adversarial network.
\newblock In {\em IEEE Conf. on Computer Vision and Pattern Recognition
  (CVPR)}, pages 4681--4690, 2017.

\bibitem{lee2013pseudo}
Dong-Hyun Lee.
\newblock Pseudo-label: The simple and efficient semi-supervised learning
  method for deep neural networks.
\newblock In {\em Workshop on challenges in representation learning, ICML},
  volume~3, 2013.

\bibitem{Li09iccv-consensusMax}
Hongdong Li.
\newblock Consensus set maximization with guaranteed global optimality for
  robust geometry estimation.
\newblock In {\em Intl. Conf. on Computer Vision (ICCV)}, pages 1074--1080.
  IEEE, 2009.

\bibitem{L19cvpr-usip}
Jiaxin Li and Gim~Hee Lee.
\newblock {USIP: Unsupervised stable interest point detection from 3d point
  clouds}.
\newblock In {\em IEEE Conf. on Computer Vision and Pattern Recognition
  (CVPR)}, pages 361--370, 2019.

\bibitem{li2016cvpr}
Yin Li, Manohar Paluri, James~M Rehg, and Piotr Doll{\'a}r.
\newblock Unsupervised learning of edges.
\newblock In {\em IEEE Conf. on Computer Vision and Pattern Recognition
  (CVPR)}, pages 1619--1627, 2016.

\bibitem{Li18cvpr-megadepth}
Zhengqi Li and Noah Snavely.
\newblock Megadepth: Learning single-view depth prediction from internet
  photos.
\newblock In {\em IEEE Conf. on Computer Vision and Pattern Recognition
  (CVPR)}, pages 2041--2050, 2018.

\bibitem{liu2019cvpr}
Pengpeng Liu, Michael Lyu, Irwin King, and Jia Xu.
\newblock Selflow: Self-supervised learning of optical flow.
\newblock In {\em IEEE Conf. on Computer Vision and Pattern Recognition
  (CVPR)}, pages 4571--4580, 2019.

\bibitem{lowe2004ijcv-distinctive}
David~G. Lowe.
\newblock Distinctive image features from scale-invariant keypoints.
\newblock {\em Intl. J. of Computer Vision}, 60(2):91--110, 2004.

\bibitem{Mangelson18icra-pcm}
Joshua~G Mangelson, Derrick Dominic, Ryan~M Eustice, and Ram Vasudevan.
\newblock Pairwise consistent measurement set maximization for robust
  multi-robot map merging.
\newblock In {\em IEEE Intl. Conf. on Robotics and Automation (ICRA)}, pages
  2916--2923. IEEE, 2018.

\bibitem{Melekhov17ICACIVS-relativepose}
Iaroslav Melekhov, Juha Ylioinas, Juho Kannala, and Esa Rahtu.
\newblock Relative camera pose estimation using convolutional neural networks.
\newblock In {\em International Conference on Advanced Concepts for Intelligent
  Vision Systems}, pages 675--687. Springer, 2017.

\bibitem{mur2017tro}
Ra\'ul Mur-Artal and Juan~D. Tard\'os.
\newblock {ORB-SLAM2}: an open-source {SLAM} system for monocular, stereo and
  {RGB-D} cameras.
\newblock {\em {IEEE} Trans. Robotics}, 33(5):1255--1262, 2017.

\bibitem{Nister04pami-fivepoint}
David Nist{\'e}r.
\newblock An efficient solution to the five-point relative pose problem.
\newblock {\em {IEEE} Trans. Pattern Anal. Machine Intell.}, 26(6):756--770,
  2004.

\bibitem{park2017iccv}
Jaesik Park, Qian-Yi Zhou, and Vladlen Koltun.
\newblock Colored point cloud registration revisited.
\newblock In {\em ICCV}, 2017.

\bibitem{Para18pami-GORE}
{\'A}. {Parra Bustos} and T.~J. Chin.
\newblock Guaranteed outlier removal for point cloud registration with
  correspondences.
\newblock {\em {IEEE} Trans. Pattern Anal. Machine Intell.}, 40(12):2868--2882,
  2018.

\bibitem{Peng19cvpr-pvnet}
Sida Peng, Yuan Liu, Qixing Huang, Xiaowei Zhou, and Hujun Bao.
\newblock {PVNet: Pixel-wise voting network for 6dof pose estimation}.
\newblock In {\em IEEE Conf. on Computer Vision and Pattern Recognition
  (CVPR)}, pages 4561--4570, 2019.

\bibitem{richter2017iccv}
Stephan~R Richter, Zeeshan Hayder, and Vladlen Koltun.
\newblock Playing for benchmarks.
\newblock In {\em Intl. Conf. on Computer Vision (ICCV)}, pages 2213--2222,
  2017.

\bibitem{rusu2009icra}
Radu~Bogdan Rusu, Nico Blodow, and Michael Beetz.
\newblock Fast point feature histograms (fpfh) for 3d registration.
\newblock In {\em IEEE Intl. Conf. on Robotics and Automation (ICRA)}, pages
  3212--3217. IEEE, 2009.

\bibitem{schmidt16ral-slamvisualdescriptorlearning}
Tanner Schmidt, Richard Newcombe, and Dieter Fox.
\newblock Self-supervised visual descriptor learning for dense correspondence.
\newblock {\em {IEEE} Robotics and Automation Letters}, 2(2):420--427, 2016.

\bibitem{Schonberger16cvpr-sfm}
Johannes~L Schonberger and Jan-Michael Frahm.
\newblock Structure-from-motion revisited.
\newblock In {\em IEEE Conf. on Computer Vision and Pattern Recognition
  (CVPR)}, pages 4104--4113, 2016.

\bibitem{Shi20arxiv-robin}
Jingnan Shi, Heng Yang, and Luca Carlone.
\newblock {ROBIN: a Graph-Theoretic Approach to Reject Outliers in Robust
  Estimation using Invariants}.
\newblock In {\em IEEE Intl. Conf. on Robotics and Automation (ICRA)}, 2021.

\bibitem{sturm12iros}
J. Sturm, N. Engelhard, F. Endres, W. Burgard, and D. Cremers.
\newblock A benchmark for the evaluation of rgb-d slam systems.
\newblock In {\em Proc. of the International Conference on Intelligent Robot
  Systems (IROS)}, Oct. 2012.

\bibitem{Tekin18cvpr-yolo6d}
Bugra Tekin, Sudipta~N Sinha, and Pascal Fua.
\newblock Real-time seamless single shot 6d object pose prediction.
\newblock In {\em IEEE Conf. on Computer Vision and Pattern Recognition
  (CVPR)}, pages 292--301, 2018.

\bibitem{thomas2019iccv}
Hugues Thomas, Charles~R. Qi, Jean-Emmanuel Deschaud, Beatriz Marcotegui,
  Fran{\c{c}}ois Goulette, and Leonidas~J. Guibas.
\newblock Kpconv: Flexible and deformable convolution for point clouds.
\newblock {\em Intl. Conf. on Computer Vision (ICCV)}, 2019.

\bibitem{tian2017cvpr}
Yurun Tian, Bin Fan, and Fuchao Wu.
\newblock L2-net: Deep learning of discriminative patch descriptor in euclidean
  space.
\newblock In {\em IEEE Conf. on Computer Vision and Pattern Recognition
  (CVPR)}, pages 661--669, 2017.

\bibitem{Toft20PAMI-LongtermLocalization}
Carl Toft, Will Maddern, Akihiko Torii, Lars Hammarstrand, Erik Stenborg,
  Daniel Safari, Masatoshi Okutomi, Marc Pollefeys, Josef Sivic, Tomas Pajdla,
  et~al.
\newblock Long-term visual localization revisited.
\newblock {\em {IEEE} Trans. Pattern Anal. Machine Intell.}, 2020.

\bibitem{wang2018cvpr}
Chaoyang Wang, Jos{\'e} Miguel~Buenaposada, Rui Zhu, and Simon Lucey.
\newblock Learning depth from monocular videos using direct methods.
\newblock In {\em IEEE Conf. on Computer Vision and Pattern Recognition
  (CVPR)}, pages 2022--2030, 2018.

\bibitem{wang20eccv-caps}
Qianqian Wang, Xiaowei Zhou, Bharath Hariharan, and Noah Snavely.
\newblock Learning feature descriptors using camera pose supervision.
\newblock In {\em European Conf. on Computer Vision (ECCV)}, 2020.

\bibitem{Wang19iccv-DCP}
Yue Wang and Justin~M Solomon.
\newblock Deep closest point: Learning representations for point cloud
  registration.
\newblock In {\em Intl. Conf. on Computer Vision (ICCV)}, pages 3523--3532,
  2019.

\bibitem{wei2020theoretical}
Colin Wei, Kendrick Shen, Yining Chen, and Tengyu Ma.
\newblock Theoretical analysis of self-training with deep networks on unlabeled
  data.
\newblock {\em arXiv preprint arXiv:2010.03622}, 2020.

\bibitem{Xiang17RSS-posecnn}
Yu Xiang, Tanner Schmidt, Venkatraman Narayanan, and Dieter Fox.
\newblock {PoseCNN}: A convolutional neural network for {6D} object pose
  estimation in cluttered scenes.
\newblock In {\em Robotics: Science and Systems (RSS)}, 2018.

\bibitem{xie2020cvpr}
Qizhe Xie, Minh-Thang Luong, Eduard Hovy, and Quoc~V Le.
\newblock Self-training with noisy student improves imagenet classification.
\newblock In {\em IEEE Conf. on Computer Vision and Pattern Recognition
  (CVPR)}, pages 10687--10698, 2020.

\bibitem{Xie20eccv-pointcontrast}
Saining Xie, Jiatao Gu, Demi Guo, Charles~R Qi, Leonidas~J Guibas, and Or
  Litany.
\newblock Pointcontrast: Unsupervised pre-training for 3d point cloud
  understanding.
\newblock In {\em European Conf. on Computer Vision (ECCV)}, 2020.

\bibitem{Yang20ral-GNC}
Heng Yang, Pasquale Antonante, Vasileios Tzoumas, and Luca Carlone.
\newblock Graduated non-convexity for robust spatial perception: From
  non-minimal solvers to global outlier rejection.
\newblock {\em {IEEE} Robotics and Automation Letters}, 2020.

\bibitem{yang19iccv-quasar}
Heng Yang and Luca Carlone.
\newblock {A quaternion-based certifiably optimal solution to the Wahba problem
  with outliers}.
\newblock In {\em Intl. Conf. on Computer Vision (ICCV)}, pages 1665--1674,
  2019.

\bibitem{Yang19rss-teaser}
Heng Yang and Luca Carlone.
\newblock A polynomial-time solution for robust registration with extreme
  outlier rates.
\newblock In {\em Robotics: Science and Systems (RSS)}, 2019.

\bibitem{Yang20cvpr-shapeStar}
Heng Yang and Luca Carlone.
\newblock In perfect shape: Certifiably optimal {3D} shape reconstruction from
  {2D} landmarks.
\newblock In {\em IEEE Conf. on Computer Vision and Pattern Recognition
  (CVPR)}, 2020.

\bibitem{Yang20nips-certifiablePerception}
Heng Yang and Luca Carlone.
\newblock One ring to rule them all: Certifiably robust geometric perception
  with outliers.
\newblock In {\em Advances in Neural Information Processing Systems (NIPS)},
  2020.

\bibitem{Yang20arXiv-teaser}
Heng Yang, Jingnan Shi, and Luca Carlone.
\newblock {TEASER: Fast and Certifiable Point Cloud Registration}.
\newblock {\em {IEEE} Trans. Robotics}, 2020.

\bibitem{Yang2014ECCV-optimalEssentialEstimationBnBConsensusMax}
Jiaolong Yang, Hongdong Li, and Yunde Jia.
\newblock Optimal essential matrix estimation via inlier-set maximization.
\newblock In {\em European Conf. on Computer Vision (ECCV)}, pages 111--126.
  Springer, 2014.

\bibitem{yang2020cvpr}
Nan Yang, Lukas~von Stumberg, Rui Wang, and Daniel Cremers.
\newblock D3vo: Deep depth, deep pose and deep uncertainty for monocular visual
  odometry.
\newblock In {\em IEEE Conf. on Computer Vision and Pattern Recognition
  (CVPR)}, pages 1281--1292, 2020.

\bibitem{yarowsky1995unsupervised}
David Yarowsky.
\newblock Unsupervised word sense disambiguation rivaling supervised methods.
\newblock In {\em 33rd annual meeting of the association for computational
  linguistics}, pages 189--196, 1995.

\bibitem{yew2018eccv}
Zi~Jian Yew and Gim~Hee Lee.
\newblock 3dfeat-net: Weakly supervised local 3d features for point cloud
  registration.
\newblock In {\em European Conf. on Computer Vision (ECCV)}, 2018.

\bibitem{yuan20eccv-deepgmr}
Wentao Yuan, Ben Eckart, Kihwan Kim, Varun Jampani, Dieter Fox, and Jan Kautz.
\newblock {DeepGMR: Learning Latent Gaussian Mixture Models for Registration}.
\newblock 2020.

\bibitem{Zakharov2019dpod}
Sergey Zakharov, Ivan Shugurov, and Slobodan Ilic.
\newblock {DPOD: 6d pose object detector and refiner}.
\newblock In {\em Intl. Conf. on Computer Vision (ICCV)}, pages 1941--1950,
  2019.

\bibitem{Zeng17cvpr-3dmatch}
Andy Zeng, Shuran Song, Matthias Nie{\ss}ner, Matthew Fisher, Jianxiong Xiao,
  and T Funkhouser.
\newblock 3dmatch: Learning the matching of local 3d geometry in range scans.
\newblock In {\em IEEE Conf. on Computer Vision and Pattern Recognition
  (CVPR)}, volume~1, page~4, 2017.

\bibitem{zhang2016colorful}
Richard Zhang, Phillip Isola, and Alexei~A Efros.
\newblock Colorful image colorization.
\newblock In {\em European Conf. on Computer Vision (ECCV)}, pages 649--666.
  Springer, 2016.

\bibitem{Zhong09iccvw-ISS}
Yu Zhong.
\newblock {Intrinsic shape signatures: A shape descriptor for 3d object
  recognition}.
\newblock In {\em 2009 IEEE 12th International Conference on Computer Vision
  Workshops, ICCV Workshops}, pages 689--696. IEEE, 2009.

\bibitem{Zhou16eccv-fastGlobalRegistration}
Qian-Yi. Zhou, Jaesik Park, and Vladlen Koltun.
\newblock Fast global registration.
\newblock In {\em European Conf. on Computer Vision (ECCV)}, pages 766--782.
  Springer, 2016.

\bibitem{Zhou18arxiv-open3D}
Qian-Yi Zhou, Jaesik Park, and Vladlen Koltun.
\newblock {Open3D}: {A} modern library for {3D} data processing.
\newblock {\em arXiv:1801.09847}, 2018.

\bibitem{zhou2017cvpr}
Tinghui Zhou, Matthew Brown, Noah Snavely, and David~G Lowe.
\newblock Unsupervised learning of depth and ego-motion from video.
\newblock In {\em IEEE Conf. on Computer Vision and Pattern Recognition
  (CVPR)}, pages 1851--1858, 2017.

\bibitem{zoph2020arxiv}
Barret Zoph, Golnaz Ghiasi, Tsung-Yi Lin, Yin Cui, Hanxiao Liu, Ekin~D Cubuk,
  and Quoc~V Le.
\newblock Rethinking pre-training and self-training.
\newblock {\em arXiv preprint arXiv:2006.06882}, 2020.

\end{thebibliography}
}

\clearpage
\onecolumn
\begin{center}
\large{\bf \emph{Supplementary Material}}
\end{center}

\renewcommand{\thesection}{A\arabic{section}}
\renewcommand{\theequation}{A\arabic{equation}}
\renewcommand{\thetheorem}{A\arabic{theorem}}
\renewcommand{\thefigure}{A\arabic{figure}}
\renewcommand{\thetable}{A\arabic{table}}
\setcounter{equation}{0}
\setcounter{section}{0}
\setcounter{theorem}{0}
\setcounter{figure}{0}
\setcounter{table}{0}

\section{Proof of Proposition~\ref{prop:augmentlagrangian}}
\label{sec:supp-proof-alm}
\begin{proof}
We prove Proposition~\ref{prop:augmentlagrangian} for Examples~\ref{ex:relativepose}-\ref{ex:pointcloudregistration} separately.

{\bf Example~\ref{ex:relativepose}: Relative Pose Estimation.} In relative pose estimation, the known geometric model for the $i$-th measurement pair is $\MRgt_i \in \SOthree$ and $\vtgt_i \in \usphere{2}$, where $\MRgt_i$ is the relative rotation, and $\vtgt_i$ is the up-to-scale relative translation between two images $\measone_i$ and $\meastwo_i$. Using $(\MRgt,\vtgt)$, we can form the \emph{essential matrix} $\MEgt_i \doteq \hatmap{\vtgt_i} \MRgt_i$, from which we further compute the \emph{fundamental matrix} $\MFgt_i \doteq (\Ktwo_i)^{-T} \MEgt_i (\Kone_i)\inv$, where $\Kone_i$, $\Ktwo_i$ are the camera intrinsics for the two images $\measone_i$ and $\meastwo_i$~\cite{hartley2004book}. Now we let the residual function $r(\cdot)$ be the algebraic error~\cite{hartley2004book}:
\bea
r(\MFgt_i, \kptonehomo_{i,k}, \kptpredicthomo_{i,k}) = (\kptpredicthomo_{i,k})\tran \MFgt_i \kptonehomo_{i,k}, \label{eq:2viewconstraint}
\eea
which should vanish if there is no measurement noise, and $\kptonehomo, \kptpredicthomo \in \Real{3}$ denotes the homogeneous coordinates of the keypoint locations. In eq.~\eqref{eq:2viewconstraint}, $\MFgt_i \kptonehomo_{i,k}$ is called the \emph{epipolar line} (in fact, $\MFgt_i \kptonehomo_{i,k}$ represents the normal vector of the plane formed by the epipolar line and the camera optical center). 

Because we have adopted a \tls cost function,~\ie~$\rho(r) = \min \cbrace{r^2, \barcsq}$ (and assume $\barcsq$ is small), obviously, the global minimizer of problem~\eqref{eq:featurelearning} is the following:
\bea
\kptpredict_{i,k} = \embedding(\kptone_{i,k}, \measone_i, \kpttwo_i, \meastwo_i) \in \begin{cases}
\text{the epipolar line } \MFgt_i \kptonehomo_{i,k} & \text{if} \text{ the epipolar line intersects } \meastwo_i \\
\meastwo_i & \text{otherwise}
\end{cases}, \label{eq:maptoepipolarline}
\eea
which says that the predicted keypoint $\kptpredict_{i,k}$ should lie precisely on the epipolar line if the epipolar line has a nonempty intersection with the image $\meastwo_i$ (so that the residual~\eqref{eq:2viewconstraint} is zero and $\rho(r)=0$), or it can be an arbitrary point on the image otherwise (so that the residual~\eqref{eq:2viewconstraint} is nonzero and $\rho(r) = \barcsq$ is very small). In~\cite{wang20eccv-caps}, the authors designed another constraint that enforces cycle consistency,~\ie,~the back-predicted keypoint of the predicted keypoint should be the original keypoint:
\bea
\embedding(\kptpredict_{i,k},\meastwo_i,\kptone_i,\measone_i) = \kptone_{i,k}. \label{eq:mapcycleconsistent}
\eea
Combining eq.~\ref{eq:maptoepipolarline} and~\eqref{eq:mapcycleconsistent}, we can reformulate the original feature learning problem~\eqref{eq:featurelearning} as:
\bea
\text{find} & \embedding_{\nnparam} \label{eq:reformulaterelativepose}\\
\subject & \embedding \text{ satisfies}~\eqref{eq:maptoepipolarline} \text{ and}~\eqref{eq:mapcycleconsistent},
\eea
which enforces the correspondence function $\embedding$ (parametrized by $\nnparam \in \Real{\dimnn}$) to map keypoints in $\measone_i$ to their corresponding epipolar lines (if the epipolar line exists) in $\meastwo_i$, and to map the predicted keypoints in $\meastwo_i$ back to their original keypoints, which is connected to the cross check criteria mentioned in the main text. 

The reformulated problem~\eqref{eq:reformulaterelativepose} is a constrained optimization problem that is not suitable for training neural networks. Therefore, the last step we do is to move the constraints to the cost function and penalize the \emph{violation} of the constraints, which is commonly referred to as the \emph{Augmented Lagrangian Method} (\alm), or the \emph{penalty method}:
\bea
\hspace{-4mm} \min_{\nnparam \in \Real{\dimnn}} \sum_{i=1}^{\nrmeasurements} \sum_{k=1}^{\nrkpt_{a_i}} \lambda_{\text{epipolar}} \cdot \dist\parentheses{\underbrace{\embedding(\kptone_{i,k}, \measone_i, \kpttwo_i, \meastwo_i)}_{\kptpredict_{i,k}}, \MFgt_i \kptonehomo_{i,k}}^2 + \lambda_{\text{cycle}} \cdot \dist\parentheses{\embedding\left( \underbrace{\embedding(\kptone_{i,k}, \measone_i, \kpttwo_i, \meastwo_i)}_{\kptpredict_{i,k}},\meastwo_i,\kptone_i,\measone_i\right), \kptone_{i,k}}^2, \label{eq:capsalm}
\eea
where $\lambda_{\text{epipolar}}, \lambda_{\text{cycle}} > 0$ are constants chosen by the user. Finally, let the correspondence function be the form in~\eqref{eq:correspondececaps}, we recover the loss function in the \caps paper~\cite{wang20eccv-caps}.\footnote{The $\dist\parentheses{\cdot}$ function in~\eqref{eq:capsalm} is equivalent to the $\ell_2$ norm $\norm{\cdot}$. \cite{wang20eccv-caps} used the $\dist\parentheses{\cdot}$ instead of $\dist\parentheses{\cdot}^2$. This can be easily seen as the Augmented Lagrangian if using the constraint $\sqrt{\dist\parentheses{\cdot}} = 0$, instead of $\dist\parentheses{\cdot}=0$.} Therefore, the \caps neural network can be seen as a method to solve the feature learning problem~\eqref{eq:featurelearning} by solving its Augmented Lagrangian~\eqref{eq:capsalm}.


{\bf Example~\ref{ex:pointcloudregistration}: Point Cloud Registration.} In point cloud registration, the known geometric model for the $i$-th measurement pairs is the rigid transformation $\MRgt_i \in \SOthree$ and $\vtgt_i \in \Real{3}$ between the two point clouds $\measone_i$ and $\meastwo_i$. Let the residual function $r(\cdot)$ be the Euclidean distance:
\bea
r(\MRgt_i,\vtgt_i,\kptone_{i,k},\kptpredict_{i,k}) = \norm{\kptpredict_{i,k} - \MRgt_i \kptone_{i,k} - \vtgt_i},
\eea
which should be zero without measurement noise. Under the \tls cost function $\rho(r) = \min \cbrace{r^2, \barcsq}$, the global minimizer of problem~\eqref{eq:featurelearning} is 
\bea
\kptpredict_{i,k} = \embedding(\kptone_{i,k},\measone_i,\kpttwo_i,\meastwo_i) = \begin{cases}
\displaystyle \argmin_{\kpttwo_{i,j} \in \kpttwo_i } r(\MRgt_i,\vtgt_i,\kptone_{i,k}, \kpttwo_{i,j}) & \text{if } \displaystyle \min_{\kpttwo_{i,j} \in \kpttwo_i } r(\MRgt_i,\vtgt_i,\kptone_{i,k}, \kpttwo_{i,j})  < \barc \\
\emptyset & \text{otherwise}
\end{cases}, \label{eq:nnsmap}
\eea
which states that the correspondence function $\embedding$ should output the nearest neighbor of $(\MRgt_i\kptone_{i,k} + \vtgt_i)$ in $\kpttwo_i$ if the Euclidean distance between the nearest neighbor and $(\MRgt_i\kptone_{i,k} + \vtgt_i)$ is close enough to be considered as an inlier, and outputs nothing otherwise (\ie,~$\kptone_{i,k}$ does not have a corresponding point in $\kpttwo_i$). Therefore, we can reformulate problem~\eqref{eq:featurelearning} as:
\bea
\text{find} & \embedding \label{eq:findcorr}\\
\subject & \embedding \text{ satisfies~\eqref{eq:nnsmap}}.
\eea
We then use the fact that $\embedding$ is a composition of a feature descriptor and nearest neighbor search in the feature space (\cf~eq.~\eqref{eq:nns} in Example~\ref{ex:pointcloudregistration}), and hence, problem~\eqref{eq:findcorr} is further equivalent to finding a descriptor $\describe$ such that:
\bea
\text{find} & \describe \\
\subject & \dist\parentheses{ \describe(\kptone_{i,k},\measone_i), \describe(\kptpredict_{i,k},\meastwo_i) } \leq \dist\parentheses{ \describe(\kptone_{i,k},\measone_i), \describe(\kpttwo_{i,j},\meastwo_i) }, \forall  \kpttwo_{i,j} \neq \kptpredict_{i,k}, \label{eq:nnsfeature}
\eea
which precisely states that the distance in the feature space between $\kptone_{i,k}$ and the corresponding keypoint $\kptpredict_{i,k}$ is smaller than the distance between $\kptone_{i,k}$ and any other point in $\kpttwo_i$. In fact, we can ask for stronger conditions on the feature descriptor $\describe$:
\bea
\text{find} & \describe \label{eq:findstrongfeature} \\
\subject & \dist\parentheses{ \describe(\kptone_{i,k},\measone_i),\describe(\kptpredict_{i,k},\meastwo_i) } \leq \pmargin, \label{eq:pmargin} \\
& \dist\parentheses{ \describe(\kptone_{i,k},\measone_i),\describe(\kpttwo_{i,j},\meastwo_i) } \geq \nmargin, \forall  \kpttwo_{i,j} \neq \kptpredict_{i,k},  \label{eq:nmargin} \\
& \dist\parentheses{ \describe(\kptone_{i,k},\measone_i),\describe(\kptpredict_{i,k},\meastwo_i) } \leq \pnmargin + \dist\parentheses{ \describe(\kptone_{i,k},\measone_i),\describe(\kpttwo_{i,j},\meastwo_i) }, \label{eq:pnmargin}
\eea
that says: (i) the feature distance between the matched keypoint pair $\kptone_{i,k}$ and $\kptpredict_{i,k}$ has to be smaller than a predefined margin $\pmargin > 0$ (eq.~\eqref{eq:pmargin}); (ii) the feature distance between $\kptone_{i,k}$ and all the other non-matched keypoints has to be larger than a predefined margin $\nmargin > \pmargin$ (eq.~\eqref{eq:nmargin}); (iii) the feature distance between non-matched keypoint pairs has to be at least $\pnmargin$ larger than the feature distance between matched keypoint pairs (eq.~\eqref{eq:pnmargin}). Obviously, conditions~\eqref{eq:pmargin}-\eqref{eq:pnmargin} are sufficient (but not necessary) for ensuring condition~\eqref{eq:nnsfeature}.

Again, problem~\eqref{eq:findstrongfeature} is a constrained optimization that is not suitable for neural network training. Therefore, we develop its Augmented Lagrangian (for the constraints related to the keypoint $\kptone_{i,k}$) to be:
\begin{multline}
\calL(\kptone_{i,k},\pslack,\nslack,\pnslack) = 
\pcoeff \parentheses{ \pmargin - \dist\parentheses{\describe(\kptone_{i,k}, \measone_i), \describe(\kptpredict_{i,k},\meastwo_i)}   - \pslack }^2 + \\
\sum_{ \kpttwo_{i,j} \neq \kptpredict_{i,k} } \ncoeff \parentheses{ \dist\parentheses{ \describe(\kptone_{i,k},\measone_i),\describe(\kpttwo_{i,j},\meastwo_i) } - \nmargin - \nslack }^2 + \\
\sum_{ \kpttwo_{i,j} \neq \kptpredict_{i,k} } \pncoeff \parentheses{ \dist\parentheses{ \describe(\kptone_{i,k},\measone_i),\describe(\kpttwo_{i,j},\meastwo_i) } - \dist\parentheses{\describe(\kptone_{i,k}, \measone_i), \describe(\kptpredict_{i,k},\meastwo_i)} - \pnmargin - \pnslack }^2, \label{eq:fcgfloss}
\end{multline}
where $\pslack,\nslack,\pnslack\geq 0$ are nonnegative slack variables. In eq.~\eqref{eq:fcgfloss}, the first two terms denote the \emph{contrastive loss}, while the last term denotes the \emph{triplet loss}. The \alm~\cite{Bertsekas99book-nonlinearprogramming} solves the following optimization:
\bea \label{eq:almfeatlearn}
\min_{ \describe ,\pslack\geq0,\nslack\geq0,\pnslack \geq 0} \sum_{i=1}^{\nrmeasurements} \sum_{k=1}^{\nrkpt_{a_i}} \calL(\kptone_{i,k},\pslack,\nslack,\pnslack).
\eea
Finally, by enforcing $\pslack=\nslack=\pnslack=0$, problem~\eqref{eq:almfeatlearn} recovers the metric learning problem in the \fcgf paper~\cite{Choy19iccv-FCGF}. Therefore, the \fcgf neural network can be seen as a method to solve the feature learning problem~\eqref{eq:featurelearning} by solving the Augmented Lagrangian~\eqref{eq:almfeatlearn}.
\end{proof}
\section{Application of \nameshort on Object Detection and Pose Estimation}
\label{sec:supp-application-objectdetect}

\begin{example}[Object Detection and Pose Estimation]
\label{ex:detectpose}
Given a collection of 3D models $\{\measone_i\}_{i=1}^{\nrmodels}$, where each model $\measone_i \in \Real{3\times \nrkpt_{a_i}}$ consists of a set of 3D keypoints, let $\measone \in \Real{3\times \nrkpt_a}, \nrkpt_a = \sum_{i=1}^{\nrmodels} \nrkpt_{a_i}$, be the concatenation of all 3D keypoints. In addition, given a corpus of 2D images $\{\meastwo_i\}_{i=1}^{\nrmeasurements}$, where each $\meastwo_i$ is an RGB image that contains the (partial, occluded) projections of the 3D models plus some background. Object detection and pose estimation seeks to jointly learn a keypoint prediction function $\embedding$ and estimate the poses of the 3D models $\model_i = \{ (\MR_{i,j}, \vt_{i,j}) \}_{j \in \calS \subset [\nrmodels]} \in (\SOthree \times \Real{3})^{|\calS|}$, where $\calS \subset [\nrmodels]$ is the subset of 3D models observed by the $i$-th 2D image ($|\calS|$ denotes the cardinality of the set). In particular, following~\cite{Zakharov2019dpod}, let $\embedding$ be a combination of UVW mapping and semantic ID masking,~\ie,~for each pixel in $\meastwo_i$, $\embedding$ predicts which 3D model it belongs to (from $1$ to $\nrmodels$, and $0$ for background), and what is the corresponding 3D coordinates in the specific model, thus deciding which point in $\measone$ is the corresponding 3D point.\footnote{There are many different ways to establish 2D-3D correspondences, see \pvnet~\cite{Peng19cvpr-pvnet}, \yolosixd~\cite{Tekin18cvpr-yolo6d} and references therein.}
\end{example}

{\bf \nameshort for Example~\ref{ex:detectpose}}. The teacher performs robust absolute pose estimation, \aka~\emph{perspective-$n$-point} (\pnp)~\cite{hartley2004book,kneip2014ECCV-UPnP}. A good candidate for the teacher is \ransac and its variants (\eg,~using \pthreep~\cite{gao03PAMI-P3P}). The student trains a 2D keypoint predictor under the supervision of camera poses. Recent works such as \yolosixd~\cite{Tekin18cvpr-yolo6d}, \pvnet~\cite{Peng19cvpr-pvnet}, and \dpod~\cite{Zakharov2019dpod} can all serve as the student network, despite using different methodologies. As for the \namefilter, similar to Example~\ref{ex:relativepose}, it can be designed based on the estimated inlier rate by \ransac. Alternatively, one can project the 3D models onto the 2D image using the estimated absolute poses and compute the overlap ratio (in terms of pixels) between the 2D projection and the estimated semantic ID mask. To initialize \nameshort, we can train a bootstrap predictor using synthetic datasets,~\ie,~by rendering synthetic projections of the 3D models under different simulated poses, which is common in~\cite{Zakharov2019dpod,Peng19cvpr-pvnet,Tekin18cvpr-yolo6d,chen19ICCVW-satellitePoseEstimation}.
\section{Detailed Experimental Data}
\label{sec:supp-experiments}
\subsection{Relative Pose Estimation}
In Section~\ref{sec:exp-relativepose}, Fig.~\ref{fig:relativepose-line-plot} plots the rotation statistics for running \nameshort on the \megadepth~\cite{Li18cvpr-megadepth} dataset for relative pose estimation. Here in Fig.~\ref{fig:supp-fig-line-plots}(a), we plot the translation statistics. In addition, the full statistics of \nameshort are tabulated in Table~\ref{tab:supp-megadepth-finetune}. Fig.~\ref{fig:supp-fig-megadepth-qualitative} visualizes 9 qualitative examples of relative pose estimation using \finalcaps on the \megadepth test set.

\renewcommand{\mpwthree}{5.9cm}
\renewcommand{\myvspace}{\vspace{2mm}}
\renewcommand{\myvspaceone}{\vspace{-1mm}}

\begin{figure}[h]
	\begin{center}
	\begin{minipage}{\columnwidth}
	\begin{tabular}{ccc}%
	\hspace{-4mm}
	  \begin{minipage}{\mpwthree}%
			\centering%
			\includegraphics[width=\columnwidth]{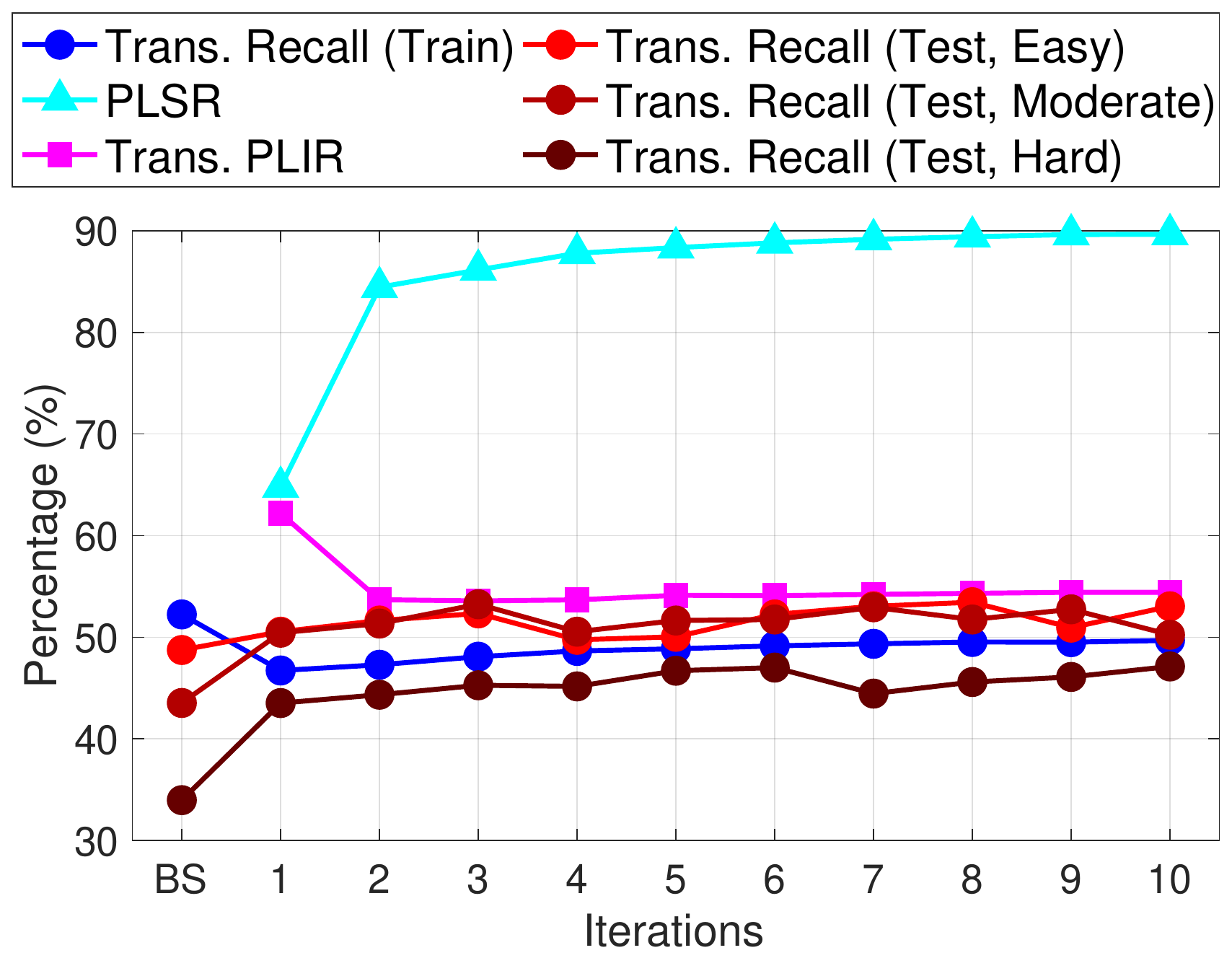}\\
			\myvspaceone
			{\smaller (a) \nameshort translation statistics on \megadepth.}
			\myvspace
		  \end{minipage}
	  &
	  \hspace{-4mm}
		\begin{minipage}{\mpwthree}%
			\centering%
			\includegraphics[width=\columnwidth]{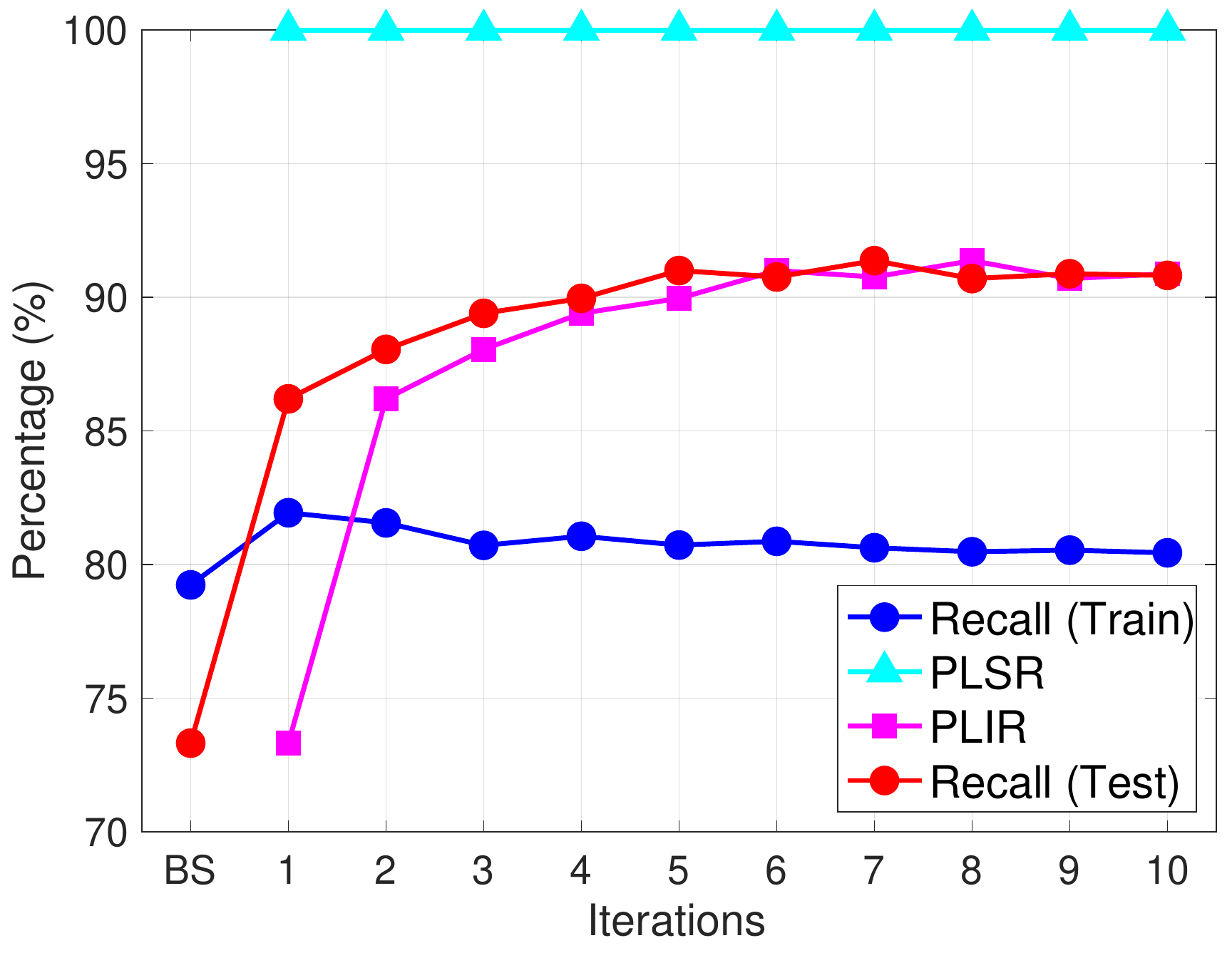}\\
			\myvspaceone
			{\smaller (b) \nameshort on \threedmatch with exchanged train and test.}
			\myvspace
		  \end{minipage}
	   &
	   \hspace{-4mm}
	   \begin{minipage}{\mpwthree}%
			\centering%
			\includegraphics[width=\columnwidth]{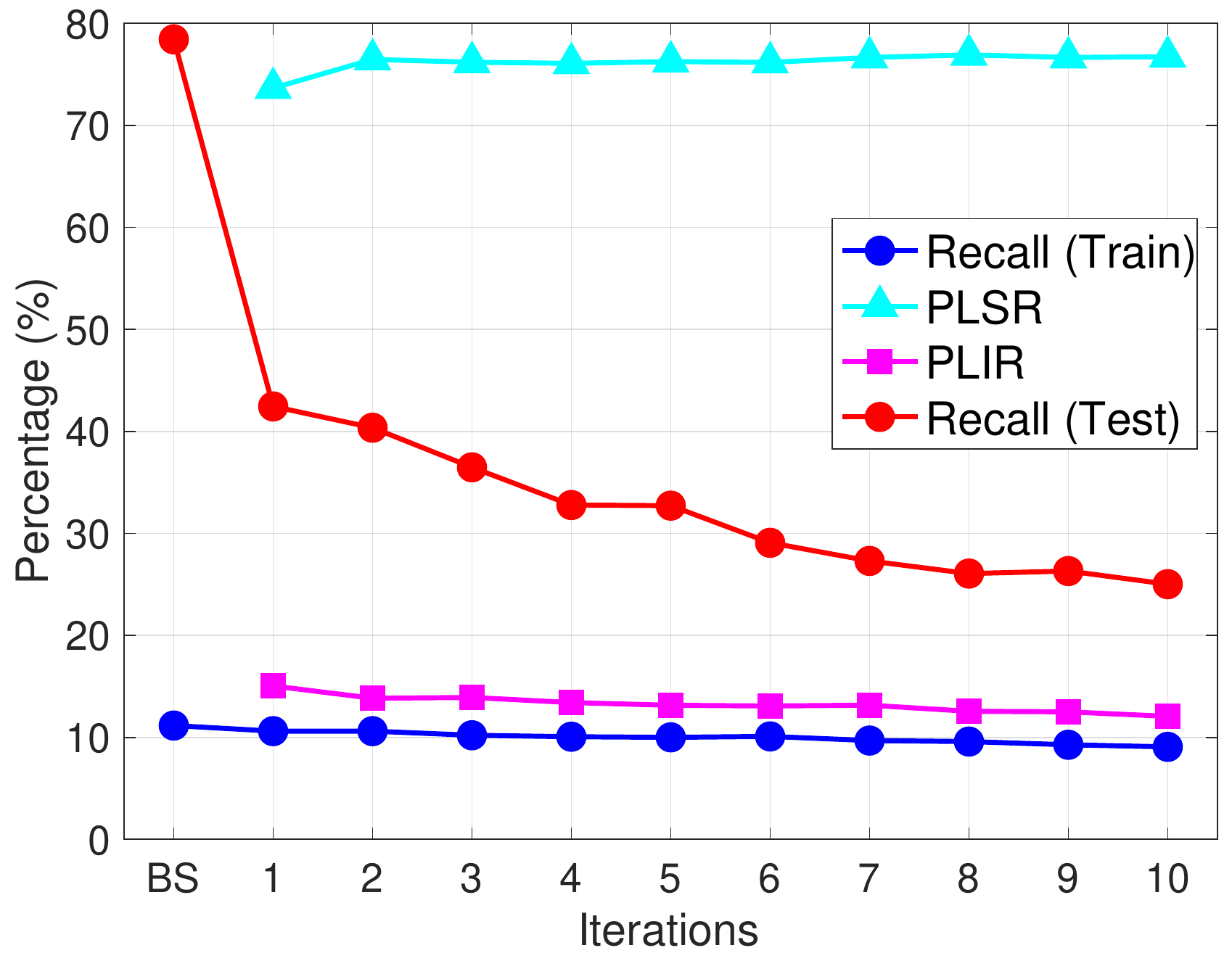}\\
			\myvspaceone
			{\smaller (c) \nameshort on \threedmatch with \horn as teacher.}
			\myvspace
		  \end{minipage} \\

		\multicolumn{3}{c}{
			\begin{minipage}{\textwidth}%
			  	\caption{Supplementary statistics. (a) The translation statistics for using \nameshort on \megadepth~\cite{Li18cvpr-megadepth} (rotation statistics shown in Fig.~\ref{fig:relativepose-line-plot} in the main text). (b) Dynamics of \nameshort on \threedmatch~\cite{Zeng17cvpr-3dmatch} with training and test sets exchanged,~\ie,~we train \nameshort on the smaller test set ($1,623$ pairs), but test \sfcgf on the larger training set ($9,856$ pairs). (c) Dynamics of \nameshort on \threedmatch by replacing the original \ransactenk teacher with a non-robust Horn's method~\cite{horn87josa} as the teacher. }
			  	\label{fig:supp-fig-line-plots}
			\end{minipage}
		}
		\end{tabular}
	\end{minipage}
	\vspace{-4mm}
	\end{center}
\end{figure}
\renewcommand{\arraystretch}{1.2}
\begin{table}[h]
\centering
\begin{tabular}{cc|c|cccccccccc}
& & \sift & \multicolumn{10}{c}{\nameshort trained \caps (\scaps)} \\
 \multicolumn{2}{c|}{Statistics (\%)} & Bootstrap & $1$ & $2$ & $3$ & $4$ & $5$ & $6$ & $7$ & $8$ & $9$ & $10$ \\
 \hline 
 \hline 
 \multirow{5}{*}{\rotatebox{90}{ \hspace{-3mm} \emph{Train} }}  & \plsr & $**$ & $64.79$ & $84.44$ & $86.14$ & $87.78$ & $88.34$ & $88.80$ & $89.15$ & $89.41$ & $89.62$ & $89.64$ \\
 & Rot. \plir & $**$ & $92.50$ & $88.83$ & $88.35$ & $88.41$ & $88.25$ & $88.52$ & $88.50$ & $88.57$ & $88.43$ & $88.48$ \\ 
 & Rot. Recall & $87.75$ & $79.33$ & $79.69$ & $80.62$ & $80.70$ & $81.20$ & $81.34$ & $81.57$ & $81.57$ & $81.56$ & $81.68$ \\ 
 \cline{2-13}
 & Trans. \plir & $**$ & $62.20$ & $53.70$ & $53.58$ & $53.68$ & $54.13$ & $54.09$ & $54.22$ & $54.33$ & $54.43$ & $54.43$ \\ 
 & Trans. Recall & $52.25$ & $46.74$ & $47.30$ & $48.09$ & $48.66$ & $48.87$ & $49.16$ & $49.36$ & $49.54$ & $49.52$ & $49.70$ \\ 
 \hline 
 \hline 
 \multirow{6}{*}{\rotatebox{90}{ \hspace{-3mm} \emph{Test Recall} }} 
 & Rot., \emph{Easy}      & $80.88$ & $85.39$ & $85.49$ & $84.68$ & $85.69$ & $85.79$ & $85.79$ & $86.29$ & $\bm{87.09}$ & $85.49$ & $86.29$\\ 
 & Rot., \emph{Moderate}  & $58.06$ & $\bm{70.37}$ & $68.27$ & $\bm{70.37}$ & $69.77$ & $69.67$ & $69.27$ & $69.87$ & $68.87$ & $70.07$ & $69.17$ \\ 
 & Rot., \emph{Hard}      & $40.35$ & $48.36$ & $49.38$ & $50.31$ & $49.59$ & $50.10$ & $\bm{51.75}$ & $50.10$ & $50.72$ & $51.23$ & $51.33$\\ 
  \cline{2-13}
 & Trans., \emph{Easy}    & $48.75$ & $50.55$ & $51.65$ & $52.35$ & $49.75$ & $50.05$ & $52.25$ & $53.05$ & $\bm{53.45}$ & $50.95$ & $53.05$ \\ 
 & Trans., \emph{Moderate}& $43.54$ & $50.45$ & $51.35$ & $\bm{53.25}$ & $50.55$ & $51.65$ & $51.75$ & $52.95$ & $51.75$ & $52.75$ & $50.25$\\ 
 & Trans., \emph{Hard}    & $33.98$ & $43.53$ & $44.35$ & $45.28$ & $45.17$ & $46.71$ & $47.02$ & $44.46$ & $45.60$ & $46.10$ & $\bm{47.13}$\\ 
\end{tabular}
\vspace{2mm}
\caption{Train and test statistics of running \nameshort on \megadepth~\cite{Li18cvpr-megadepth}. \nameshort setting: \retrain = \false, \filterlabel = \true, \namefilter~criteria: number of matches larger than 100 and \ransac estimated inlier rate larger than $10\%$. Rotation statistics plotted in Fig.~\ref{fig:relativepose-line-plot} in the main text. Translation statistics plotted in Fig.~\ref{fig:supp-fig-line-plots}(a).}
\label{tab:supp-megadepth-finetune}
\end{table}
\renewcommand{\mpwthree}{5.9cm}
\renewcommand{\myvspace}{\vspace{2mm}}
\renewcommand{\myvspaceone}{\vspace{-1mm}}

\begin{figure}[h]
	\begin{center}
	\begin{minipage}{\columnwidth}
	\begin{tabular}{ccc}%
	\hspace{-4mm}
	  \begin{minipage}{\mpwthree}%
			\centering%
			\includegraphics[width=\columnwidth]{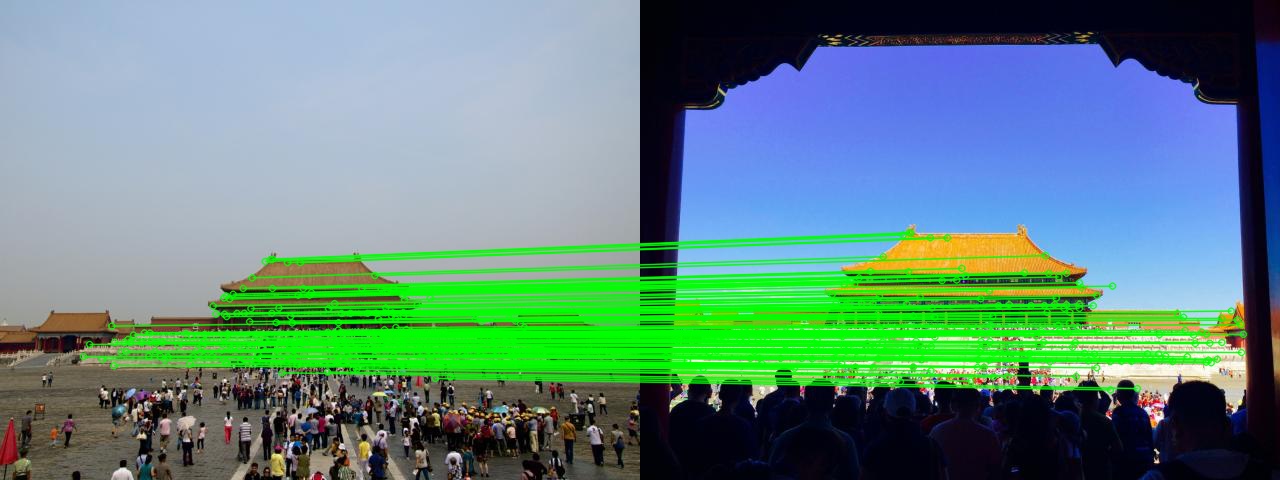}\\
		  \end{minipage}
	  &
	  \hspace{-4mm}
		\begin{minipage}{\mpwthree}%
			\centering%
			\includegraphics[width=\columnwidth]{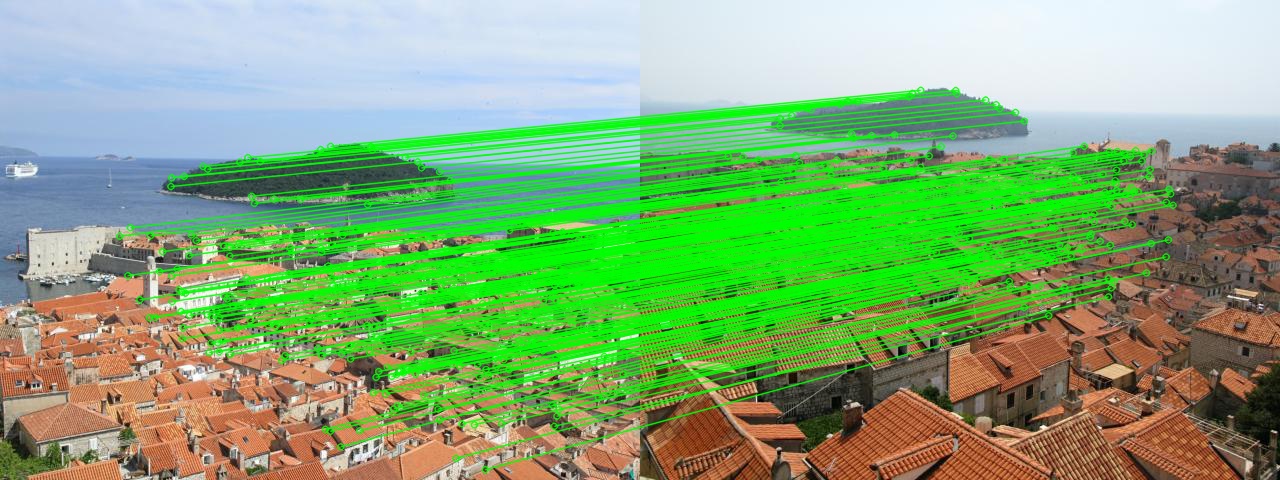}\\
		  \end{minipage}
	   &
	  \hspace{-4mm}
		\begin{minipage}{\mpwthree}%
			\centering%
			\includegraphics[width=\columnwidth]{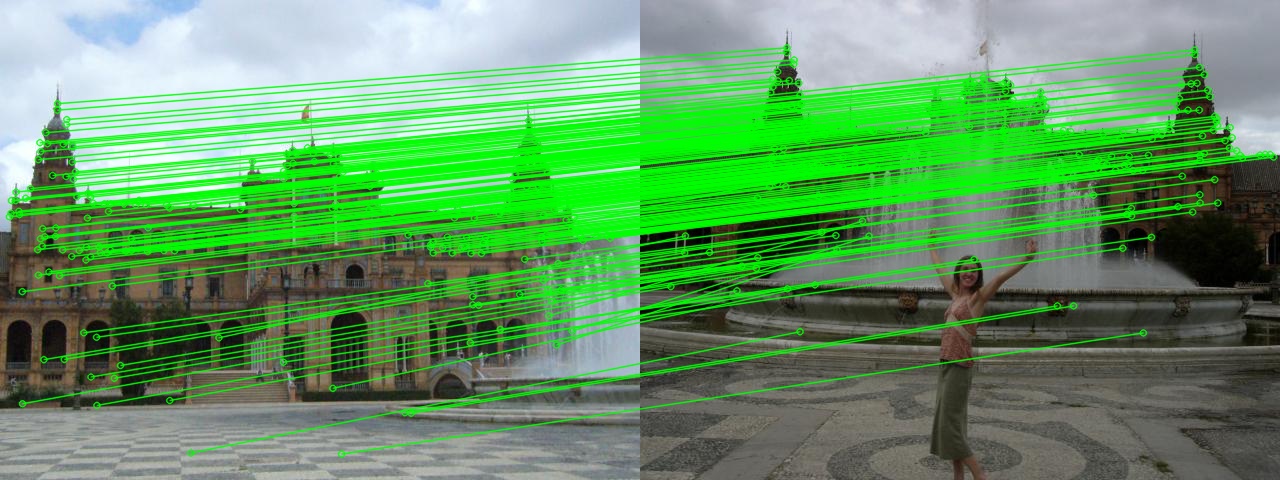}\\
		  \end{minipage} \\
	  \multicolumn{3}{c}{
			\begin{minipage}{\textwidth}%
			\centering
			  	{\smaller (a) \emph{Easy}}
			\end{minipage}
	  }\\
	\hspace{-4mm}
	  \begin{minipage}{\mpwthree}%
			\centering%
			\includegraphics[width=\columnwidth]{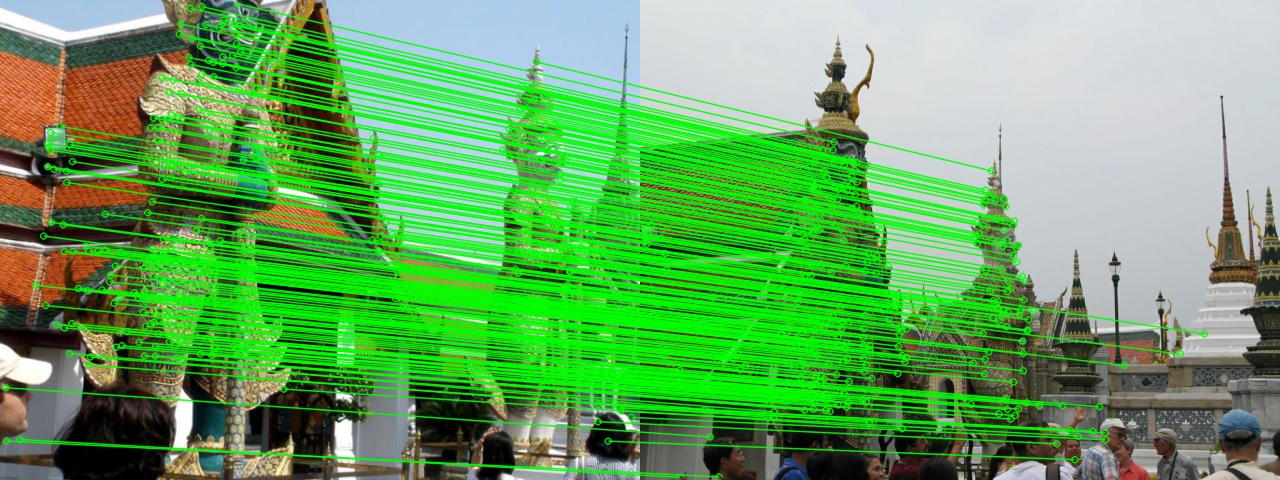}\\
		  \end{minipage}
	  &
	  \hspace{-4mm}
		\begin{minipage}{\mpwthree}%
			\centering%
			\includegraphics[width=\columnwidth]{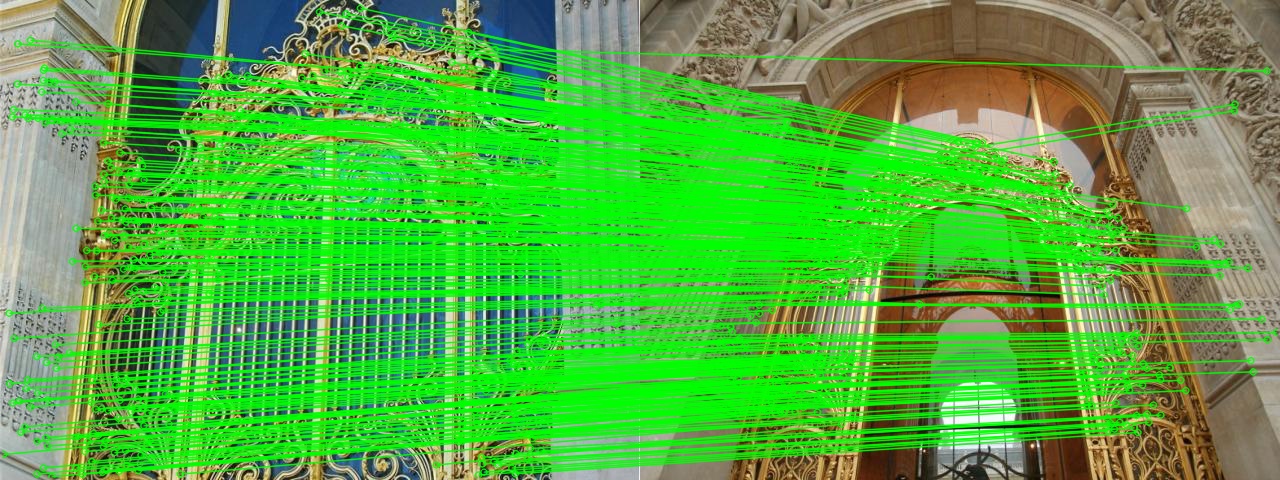}\\
		  \end{minipage}
	   &
	  \hspace{-4mm}
		\begin{minipage}{\mpwthree}%
			\centering%
			\includegraphics[width=\columnwidth]{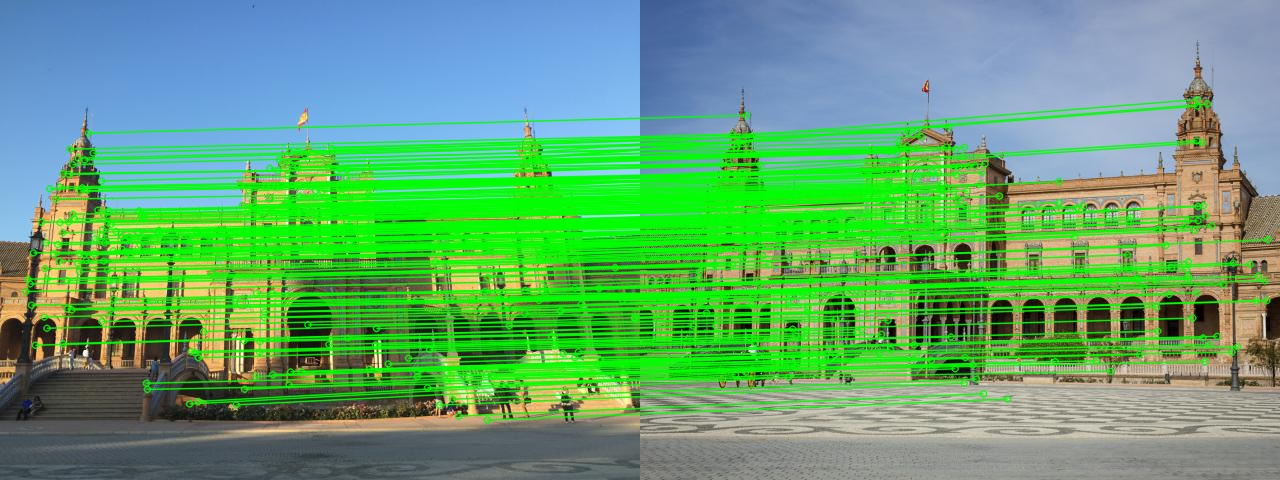}\\
		  \end{minipage} \\
	  \multicolumn{3}{c}{
			\begin{minipage}{\textwidth}%
			\centering
			  	{\smaller (b) \emph{Moderate}}
			\end{minipage}
	  }\\

	  \hspace{-4mm}
	  \begin{minipage}{\mpwthree}%
			\centering%
			\includegraphics[width=\columnwidth]{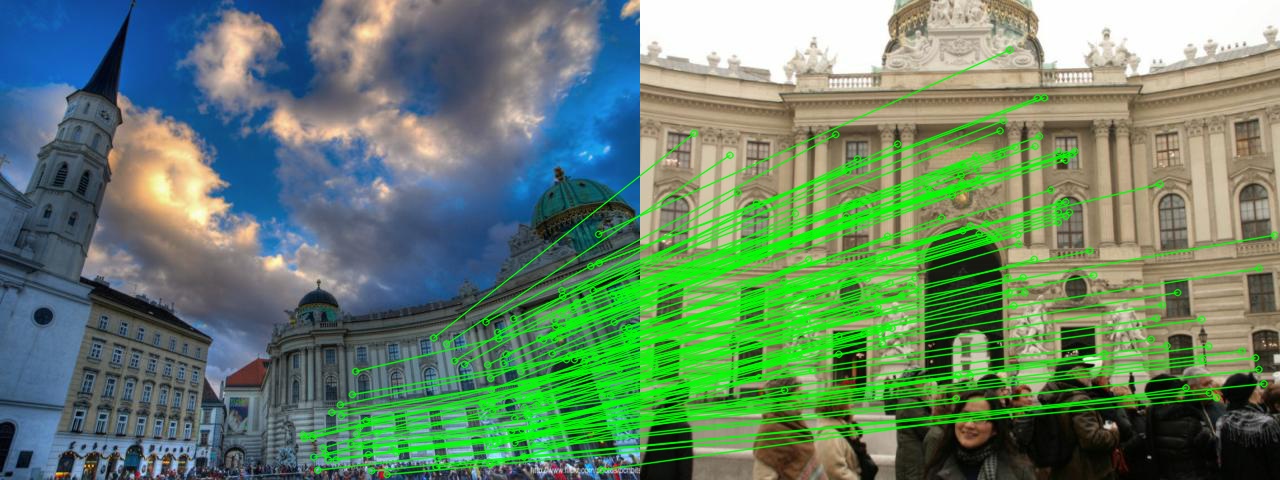}\\
		  \end{minipage}
	  &
	  \hspace{-4mm}
		\begin{minipage}{\mpwthree}%
			\centering%
			\includegraphics[width=\columnwidth]{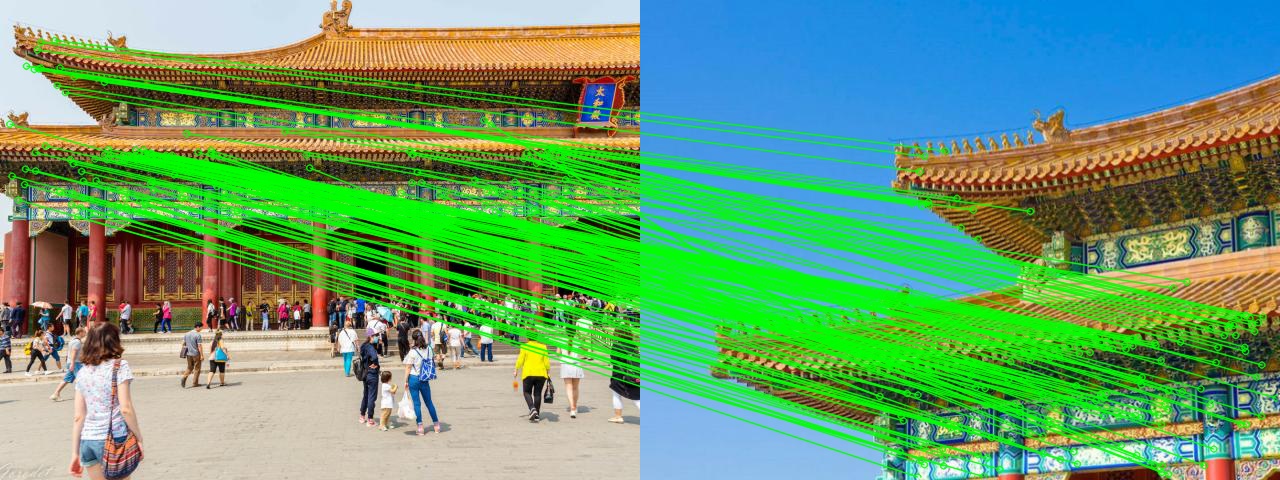}\\
		  \end{minipage}
	   &
	  \hspace{-4mm}
		\begin{minipage}{\mpwthree}%
			\centering%
			\includegraphics[width=\columnwidth]{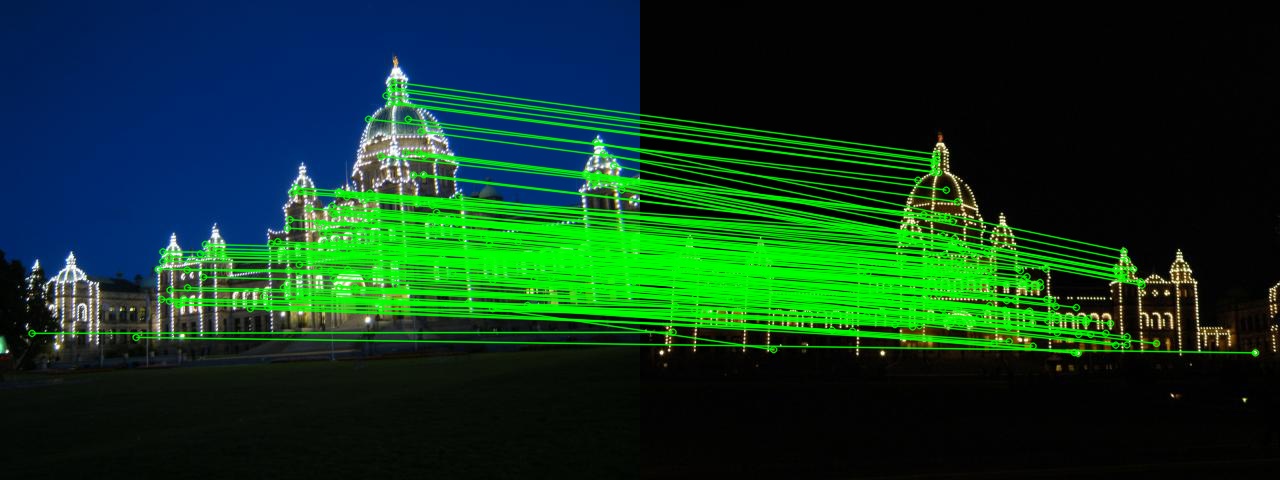}\\
		  \end{minipage} \\
	  \multicolumn{3}{c}{
			\begin{minipage}{\textwidth}%
			\centering
			  	{\smaller (c) \emph{Hard}}
			\end{minipage}
	  }\\

		\multicolumn{3}{c}{
			\begin{minipage}{\textwidth}%
			  	\caption{Supplementary qualitative results for relative pose estimation on the \megadepth dataset~\cite{Li18cvpr-megadepth} using \finalcaps.}
			  	\label{fig:supp-fig-megadepth-qualitative}
			\end{minipage}
		}

		\end{tabular}
	\end{minipage}
	\vspace{-4mm}
	\end{center}
\end{figure}

\subsection{Point Cloud Registration}
In Section~\ref{sec:exp-registration}, Fig.~\ref{fig:registration-line-plot} plots the dynamics of runing \nameshort on the \threedmatch~\cite{Zeng17cvpr-3dmatch} dataset. Here we provide the full 
statistics in Table~\ref{tab:supp-3dmatch-finetune}.

\renewcommand{\arraystretch}{1.2}
\begin{table}[h]
\centering
\begin{tabular}{cc|c|cccccccccc}
& & \fpfh & \multicolumn{10}{c}{\nameshort trained \fcgf (\sfcgf)} \\
 \multicolumn{2}{c|}{Statistics (\%)} & Bootstrap & $1$ & $2$ & $3$ & $4$ & $5$ & $6$ & $7$ & $8$ & $9$ & $10$ \\
 \hline 
 \hline 
 \multirow{3}{*}{\rotatebox{90}{ \hspace{-3mm} \emph{Train} }}  & \plsr & $**$ & $69.98$ & $73.91$ & $95.53$ & $95.69$ & $95.74$ & $95.73$ & $95.75$ & $95.73$ & $95.76$ & $95.77$ \\
 & \plir & $**$ & $92.03$ & $93.42$ & $92.19$ & $92.82$ & $93.02$ & $93.25$ & $93.24$ & $93.43$ & $93.41$ & $93.39$ \\ 
 & Recall & $82.68$ & $89.14$ & $90.92$ & $91.14$ & $91.43$ & $91.76$ & $91.78$ & $91.95$ & $91.97$ & $91.95$ & $92.05$ \\ 
 \hline 
 \hline 
 \multirow{9}{*}{\rotatebox{90}{ \hspace{-3mm} \emph{Test Recall} }} 
 & Kitchen & $80.63$ & $98.42$ & $98.02$ & $98.22$ & $98.02$ & $98.22$ & $98.42$ & $98.02$ & $97.83$ & $98.62$ & $98.42$\\ 
 & Home 1 & $84.62$  & $92.31$ & $93.59$ & $91.03$ & $93.59$ & $92.95$ & $94.23$ & $94.23$ & $94.23$ & $94.23$ & $94.23$ \\ 
 & Home 2 & $69.23$ & $77.88$ & $74.04$ & $75.48$ & $75.00$ & $75.96$ & $73.08$ & $75.96$ & $76.92$ & $73.08$ & $75.00$\\ 
 & Hotel 1 & $88.05$ & $96.90$ & $97.35$ & $98.23$ & $97.79$ & $98.23$ & $99.12$ & $98.67$ & $98.67$ & $98.23$ & $98.67$ \\ 
 & Hotel 2 & $76.92$ & $87.50$ & $85.58$ & $86.54$ & $90.38$ & $89.42$ & $90.38$ & $90.38$ & $89.42$ & $89.42$ & $89.42$\\ 
 & Hotel 3 & $88.89$ & $85.19$  & $83.33$ & $83.33$ & $79.63$ & $81.48$ & $79.63$ & $85.19$ & $79.63$ & $77.78$ & $79.63$\\ 
 & Study & $71.23$ & $85.27$ & $86.30$ & $87.67$ & $86.99$ & $85.96$ & $86.99$ & $88.01$ & $86.99$ & $86.30$ & $87.33$ \\ 
 & MIT & $70.13$ & $79.22$ & $79.22$ & $80.52$ & $77.92$ & $77.92$ & $77.92$ & $80.52$ & $76.62$ & $79.22$ & $76.62$ \\ 
 \cline{2-13}
 & \emph{Overall} & $78.44$ & $90.57$ & $90.14$ & $90.63$ & $90.57$ & $90.57$ & $90.70$ & $\bm{91.37}$ & $90.82$ & $90.45$ & $90.82$
\end{tabular}
\vspace{2mm}
\caption{Train and test statistics of running \nameshort on \threedmatch~\cite{Zeng17cvpr-3dmatch}. \nameshort setting: \retrain = \false, \filterlabel = \true, \namefilter~overlap ratio threshold $\eta$: $\eta=30\%$ for iterations $\tau=1,2$, $\eta = 10\%$ for iterations $\tau=3,\dots,10$. Statistics plotted in Fig.~\ref{fig:registration-line-plot} in the main text.}
\label{tab:supp-3dmatch-finetune}
\end{table}

For qualitative results, in Fig.~\ref{fig:supp-fig-multiway} we showcase multiway registration results on various RGB-D datasets~\cite{sturm12iros,Choi15cvpr-robustrecon,choi2016arxiv,park2017iccv} in addition to Fig.~\ref{fig:registration-qualitative}. With \sfcgf, rich loop closures can be detected (in green lines), ensuring high-fidelity camera poses for dense reconstruction. It is worth noting that global registration with trained \sfcgf\!+\ransactenk, unlike \dgr, can easily run in parallel on a single graphics card due to its inexpensive memory cost. This results in at least $4\times$ speedup comparing to \dgr in practice when multi-thread loop closure detection is enabled~\cite{Zhou18arxiv-open3D}.
\renewcommand{\mpwthree}{5.9cm}
\renewcommand{\mpwtwo}{8.8cm}
\renewcommand{\myvspace}{\vspace{2mm}}
\renewcommand{\myvspaceone}{\vspace{-1mm}}

\begin{figure}[h]
	\begin{center}
	\begin{minipage}{\columnwidth}
	\begin{tabular}{ccc}%
	\hspace{-4mm}
	  \begin{minipage}{\mpwtwo}%
			\centering%
			\includegraphics[width=\columnwidth]{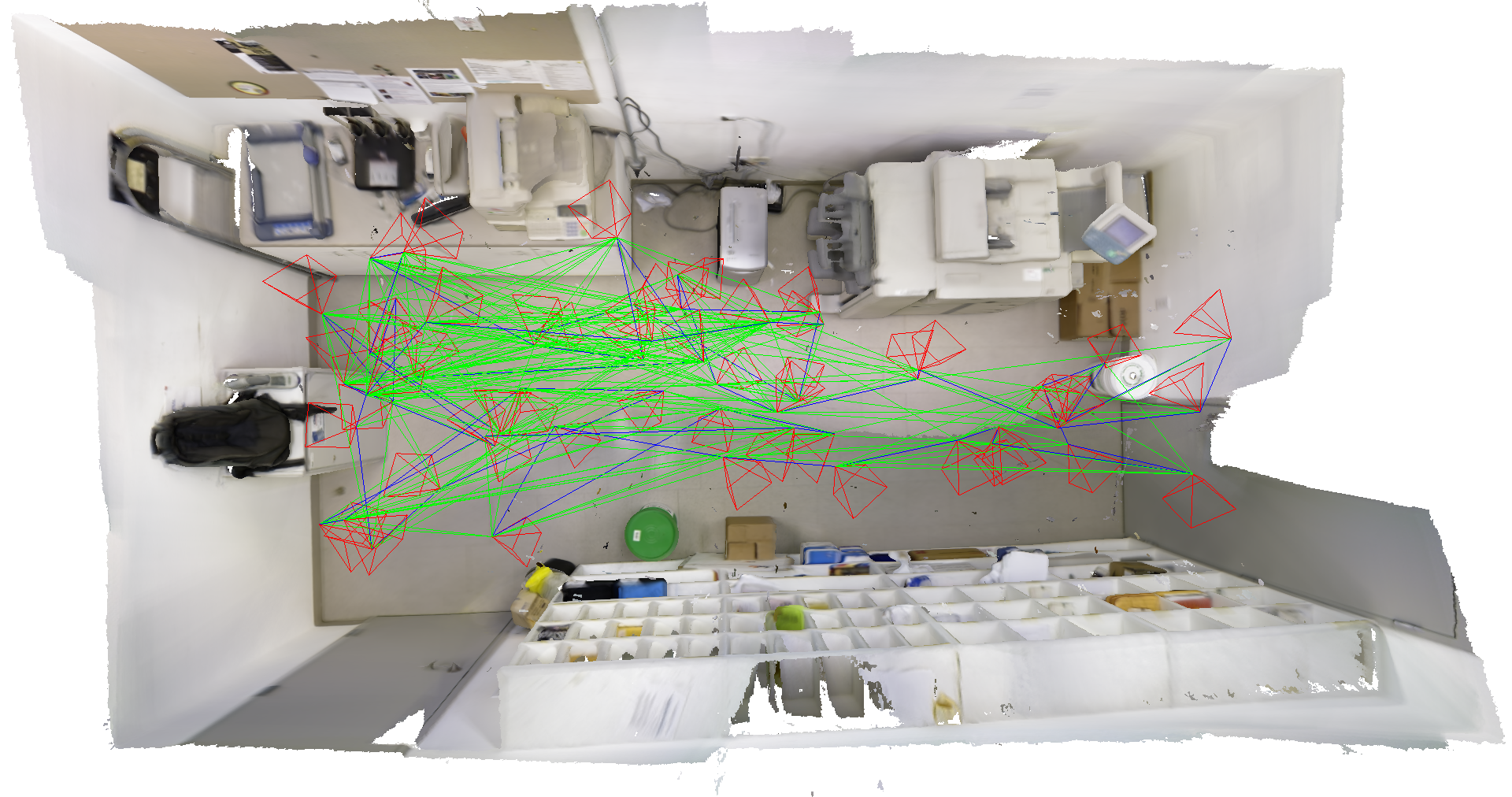}\\
			\myvspaceone
			{\smaller (a) \textit{copyroom} from \multiwayname~\cite{Choi15cvpr-robustrecon}.}
			\myvspace
		  \end{minipage}
	  &
	  \hspace{-4mm}
		\begin{minipage}{\mpwtwo}%
			\centering%
			\includegraphics[width=\columnwidth]{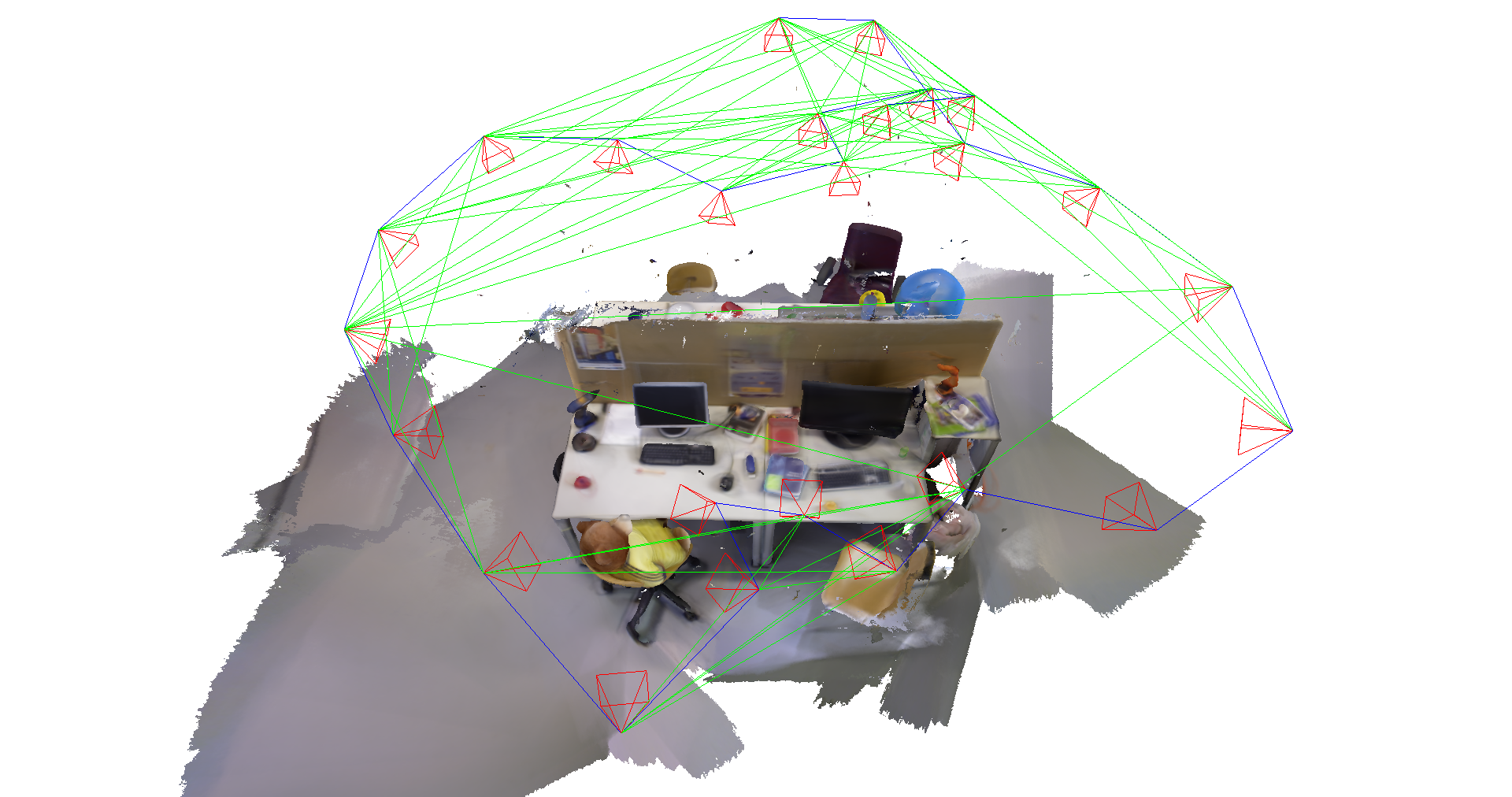}\\
			\myvspaceone
			{\smaller (b) \textit{long\_office} from \tum~\cite{sturm12iros}.}
			\myvspace
		  \end{minipage}\\	   

	   \begin{minipage}{\mpwtwo}%
			\centering%
			\includegraphics[width=\columnwidth]{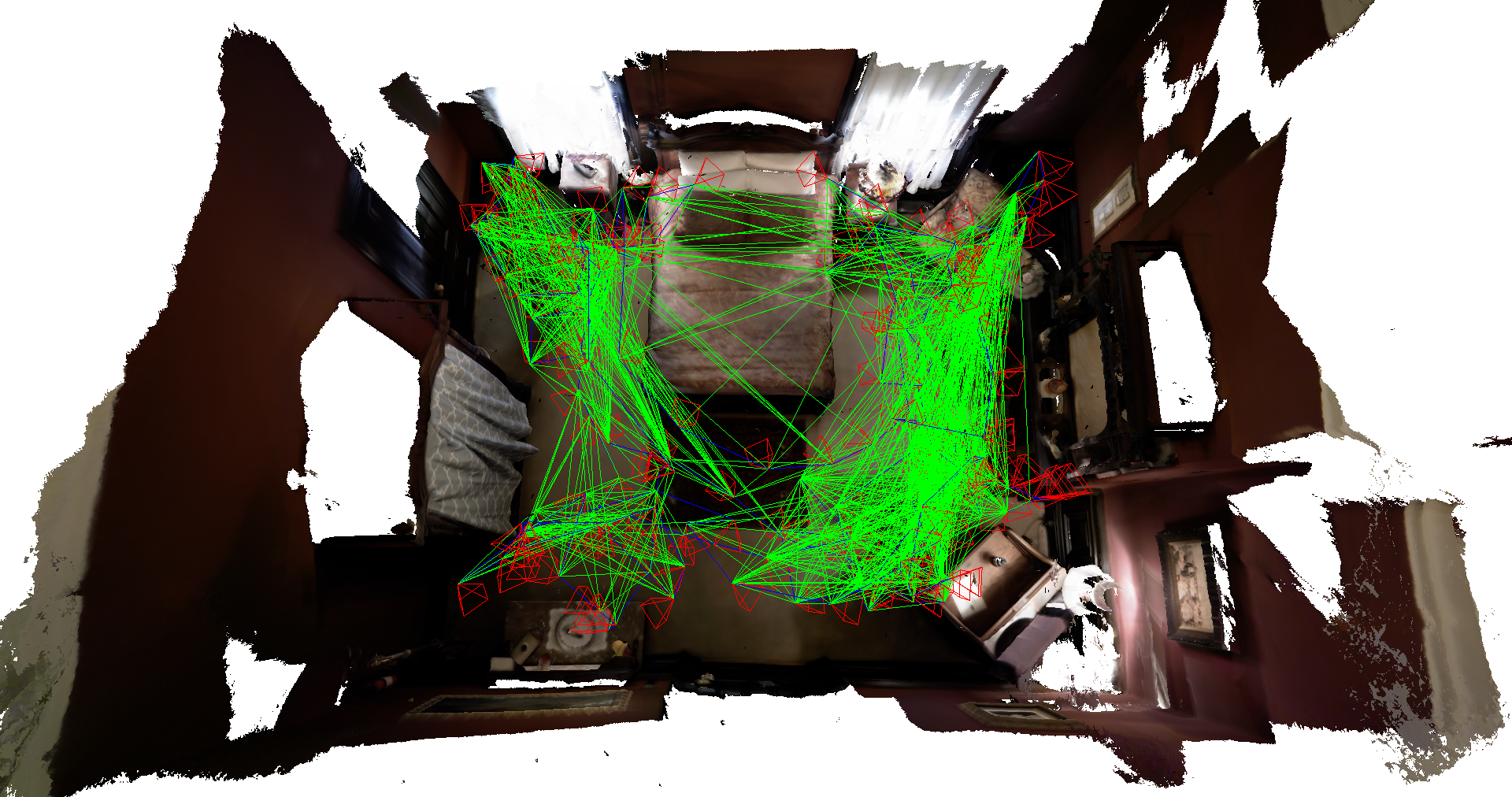}\\
			\myvspaceone
			{\smaller (c) \textit{bedroom} from \indoorlidarrgbd~\cite{park2017iccv}.}
			\myvspace
        \end{minipage} 
        &
        \hspace{-4mm}          
        \begin{minipage}{\mpwtwo}%
			\centering%
			\includegraphics[width=\columnwidth]{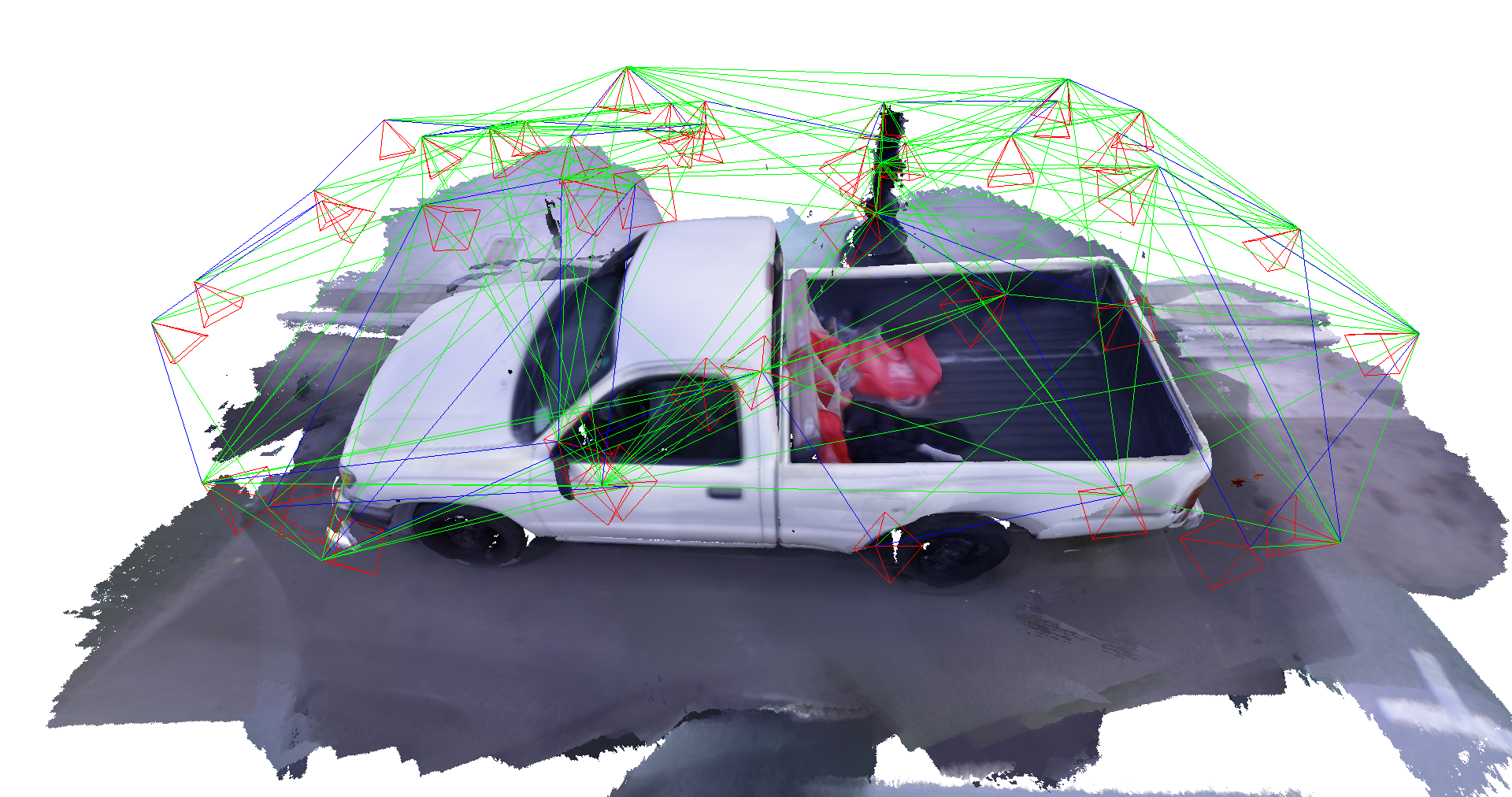}\\
			\myvspaceone
			{\smaller (d) \textit{truck} from \redwoodobjs~\cite{choi2016arxiv}.}
			\myvspace
        \end{minipage}
        \\

		\multicolumn{3}{c}{
			\begin{minipage}{\textwidth}%
			  	\caption{Supplementary qualitative results for 3D registration. Multi-way reconstruction using \sfcgf\!\!+\ransactenk as the global registration method succeeds on various unseen RGB-D datasets. Blue lines: odometry. Green lines: loop closures.}
			  	\label{fig:supp-fig-multiway}
			\end{minipage}
		}
		\end{tabular}
	\end{minipage}
	\vspace{-4mm}
	\end{center}
\end{figure}

\subsection{Ablation Study}
In Section~\ref{sec:exp-ablation}, Fig.~\ref{fig:ablation-dynamics} plots the dynamics of running \nameshort on \threedmatch with two different algorithmic settings: (a) set \retrain=\true and use \retrain instead of \finetune; (b) set \filterlabel = \false and turn off the \namefilter. Here we provide the full statistics for (a) and (b) in Table~\ref{tab:supp-3dmatch-retrain} and Table~\ref{tab:supp-3dmatch-finetune-nofilter}, respectively.

\renewcommand{\arraystretch}{1.2}
\begin{table}[h]
\centering
\begin{tabular}{cc|c|cccccccccc}
& & \fpfh & \multicolumn{10}{c}{\nameshort trained \fcgf (\sfcgf)} \\
 \multicolumn{2}{c|}{Statistics (\%)} & Bootstrap & $1$ & $2$ & $3$ & $4$ & $5$ & $6$ & $7$ & $8$ & $9$ & $10$ \\
 \hline 
 \hline 
 \multirow{3}{*}{\rotatebox{90}{ \hspace{-3mm} \emph{Train} }}  & \plsr & $**$ & $68.48$ & $95.68$ & $95.61$ & $95.61$ & $95.69$ & $95.64$ & $95.60$ & $95.61$ & $95.67$ & $95.65$\\
 & \plir & $**$ & $90.86$ & $91.16$ & $92.27$ &	$92.40$ & $92.29$ &$92.47$ & $92.52$ & $92.61$ & $92.95$ & $92.44$\\
 & Recall & $79.24$	& $89.53$ & $90.60$ & $90.69$ & $90.68$ & $90.79$ & $90.77$ & $90.88$ & $91.20$ & $90.72$ & $90.84$ \\
 \hline 
 \hline 
 \multirow{9}{*}{\rotatebox{90}{ \hspace{-3mm} \emph{Test Recall} }} & Kitchen & $**$ & $97.23$	& $97.63$ & $98.22$ & $97.83$ & $98.42$ & $97.83$	& $97.83$ & $97.23$ & $98.42$ & $98.22$\\
 & Home 1 & $**$ & $91.67$  & $93.59$ & $94.23$ & $95.51$ & $94.87$ & $93.59$ & $95.51$ & $95.51$ & $91.03$ & $93.59$\\
 & Home 2 & $**$ & $73.56$  & $71.63$ & $76.92$ & $73.56$ & $75.00$ & $74.04$ & $72.60$	& $76.44$ & $75.00$ & $75.00$\\
 & Hotel 1 & $**$ & $96.90$ &	$96.90$ & $96.90$ &	$96.46$&	$96.90$ &	$96.46$ &	$96.90$&	$98.23$ & $97.35$ & $96.90$\\ 
 & Hotel 2 & $**$ & $85.58$&	$89.42$&	$92.31$&	$88.46$&	$87.50$&	$90.38$&	$88.46$&	$88.46$ & $86.54$ & $91.35$\\ 
 & Hotel 3 & $**$ & $85.19$ &	$88.89$&	$83.33$&	$81.48$&	$83.33$&	$83.33$&	$83.33$&	$85.19$& $85.19$ & $83.33$\\ 
 & Study & $**$ &  $82.88$&		$84.59$&	$86.64$&	$88.36$&	$88.70$&	$87.67$&	$87.67$&	$86.30$ & $87.33$ & $86.64$\\ 
 & MIT & $**$ & $85.71$&	$83.12$&	$79.22$&	$79.22$&	$83.12$&	$80.52$&	$83.12$&	$77.92$ & $77.92$ & $84.42$\\ 
 \cline{2-13}
 & \emph{Overall} & $**$ & $89.34$&	$89.96$&	$91.07$&	$90.57$&	$\bm{91.19}$&	$90.57$&	$90.63$&	$90.70$ & $90.39$ & $90.94$
\end{tabular}
\vspace{2mm}
\caption{Train and test statistics of running \nameshort on \threedmatch~\cite{Zeng17cvpr-3dmatch}. \nameshort setting: \retrain = \true, \filterlabel = \true, \namefilter~overlap ratio threshold $\eta$: $\eta=10\%$ for all iterations $\tau=1,\dots,10$. Statistics plotted in Fig.~\ref{fig:ablation-dynamics}(a) in the main text.}
\label{tab:supp-3dmatch-retrain}
\end{table}
\renewcommand{\arraystretch}{1.2}
\begin{table}[h]
\centering
\begin{tabular}{cc|c|cccccccccc}
& & \fpfh & \multicolumn{10}{c}{\nameshort trained \fcgf (\sfcgf)} \\
 \multicolumn{2}{c|}{Statistics (\%)} & Bootstrap & $1$ & $2$ & $3$ & $4$ & $5$ & $6$ & $7$ & $8$ & $9$ & $10$ \\
 \hline 
 \hline 
 \multirow{3}{*}{\rotatebox{90}{ \hspace{-3mm} \emph{Train} }}  & \plsr & $**$ & $100.0$ & $100.0$ & $100.0$ & $100.0$ & $100.0$ & $100.0$ & $100.0$ & $100.0$ & $100.0$ & $100.0$ \\
 & \plir & $**$ & $79.24$ & $88.82$ & $90.86$ & $91.25$ & $91.63$ & $91.59$ & $91.93$ & $92.12$ & $91.89$ & $91.97$\\
 & Recall & $79.24$ & $88.82$ & $90.86$ & $91.25$ & $91.63$ & $91.59$ & $91.93$ & $92.12$ & $91.89$ & $91.97$ & $92.05$\\
 \hline 
 \hline 
 \multirow{9}{*}{\rotatebox{90}{ \hspace{-3mm} \emph{Test Recall} }} & Kitchen & $**$ & $97.43$ & $98.22$ & $98.62$ & $97.83$ & $98.62$ & $98.81$ & $98.22$ & $98.62$ & $98.22$ & $98.22$ \\
 & Home 1 & $**$ & $92.31$ & $94.23$ & $91.67$ & $94.23$ & $94.23$ & $92.95$ & $93.59$ & $93.59$ & $94.87$ & $92.95$\\
 & Home 2 & $**$ & $74.04$ & $75.00$ & $72.12$ & $77.40$ & $74.04$ & $74.04$ & $73.56$ & $74.04$ & $73.56$ & $73.08$\\
 & Hotel 1 & $**$ & $95.58$ & $98.23$ & $97.35$ & $97.79$ & $99.12$ & $98.23$ & $98.67$ & $96.90$ & $97.79$ & $97.35$\\
 & Hotel 2 & $**$ & $90.38$ & $93.27$ & $88.46$ & $90.38$ & $88.46$ & $87.50$ & $86.54$ & $88.46$ & $88.46$ & $89.42$\\
 & Hotel 3 & $**$ & $88.89$ & $85.19$ & $83.33$ & $87.04$ & $85.19$ & $85.19$ & $81.48$ & $85.19$ & $83.33$ & $81.48$\\
 & Study & $**$ & $84.59$ & $87.33$ & $87.67$ & $86.64$ & $88.01$ & $88.01$ & $87.67$ & $88.70$ & $88.70$ & $87.67$\\
 & MIT & $**$ & $76.62$	& $83.12$ & $77.92$ & $84.42$ & $79.22$ & $80.52$ & $80.52$	& $84.42$ & $83.12$ & $83.12$\\
 \cline{2-13}
 & \emph{Overall} & $**$ & $89.65$ & $\bm{91.44}$ & $90.26$ & $91.37$ & $91.19$ & $91.00$ & $90.63$ & $91.19$ & $91.13$ & $90.63$
\end{tabular}
\vspace{2mm}
\caption{Train and test statistics of running \nameshort on \threedmatch~\cite{Zeng17cvpr-3dmatch}. \nameshort setting: \retrain = \false, \filterlabel = \false. Statistics plotted in Fig.~\ref{fig:ablation-dynamics}(b) in the main text.}
\label{tab:supp-3dmatch-finetune-nofilter}
\end{table}

Additionally, we show results for two extra ablation experiments on the \threedmatch dataset for point cloud registration.

{\bf Exchange the training and test sets}. Because \nameshort requires no ground-truth pose labels, there is no fundamental difference between the training and test set, except that the training set ($9,856$ pairs) is much larger than the test set ($1,623$ pairs). Therefore, we ask the question: \emph{Can \nameshort learn an equally good feature representation from the much smaller test set?} Our answer is: \emph{it depends on the purpose}. We performed an experiment where we trained \nameshort on the test set, and tested the learned \sfcgf representation on the much larger training set. For \nameshort we used \retrain = \false and \filterlabel = \false. Fig.~\ref{fig:supp-fig-line-plots}(b) plots the dynamics and Table~\ref{tab:supp-3dmatch-train-test-flip} provides the full statistics. Two observations can be made: (i) Exchanging the training and test set has almost no effect on the recall of \sfcgf on the test set (\cf~Table~\ref{tab:supp-3dmatch-train-test-flip} vs Table~\ref{tab:supp-3dmatch-finetune}-\ref{tab:supp-3dmatch-finetune-nofilter}). This means that, if one only cares about the performance of the learned representation on the test set, then running \nameshort directly on the target test set is sufficient. (ii) Although exchanging the training and test set does not hurt the recall on the test set, it indeed decreases the recall on the training set by more than $10\%$. This suggests that a small training set has the shortcoming of overfitting and the learned representation fails to generalize to a larger dataset. Therefore, if one cares generalization of the learned representation, then a larger training set is still preferred. Nevertheless, this ablation study demonstrates the power of the alternating minimization nature of \nameshort, that is, \nameshort is able to find a sufficiently good local minimum. 

\renewcommand{\arraystretch}{1.2}
\begin{table}[h]
\centering
\begin{tabular}{cc|c|cccccccccc}
& & \fpfh & \multicolumn{10}{c}{\nameshort trained \fcgf (\sfcgf)} \\
 \multicolumn{2}{c|}{Statistics (\%)} & Bootstrap & $1$ & $2$ & $3$ & $4$ & $5$ & $6$ & $7$ & $8$ & $9$ & $10$ \\
 \hline 
 \hline 
 \multirow{11}{*}{\rotatebox{90}{ \hspace{-3mm} \emph{Train} }}  & \plsr & $**$ & $100.0$ & $100.0$ & $100.0$ & $100.0$ & $100.0$ & $100.0$ & $100.0$ & $100.0$ & $100.0$ & $100.0$ \\
 & \plir & $**$ & $73.32$ & $86.20$ & $88.05$ & $89.40$ & $89.96$ & $91.00$ & $90.76$ & $91.37$ & $90.70$ & $90.88$ \\ 
 & Recall & $73.32$ & $86.20$ & $88.05$ & $89.40$ & $89.96$ & $91.00$ & $90.76$ & $\bm{91.37}$ & $90.70$ & $90.88$ & $90.82$ \\ 
 \cline{2-13}
 & Kitchen & $**$ & $94.66$ & $96.84$ & $98.42$ & $98.81$ & $99.21$ & $99.60$ & $99.41$ & $99.01$ & $99.21$ & $99.21$\\ 
 & Home 1  & $**$ & $91.03$ & $89.74$ & $93.59$ & $95.51$ & $95.51$ & $94.87$ & $94.87$ & $95.51$ & $96.15$ & $95.51$\\ 
 & Home 2  & $**$ & $70.67$ & $70.67$ & $71.63$ & $69.23$ & $71.15$ & $70.19$ & $74.04$ & $72.60$ & $72.12$ & $72.60$\\ 
 & Hotel 1 & $**$ & $94.69$ & $96.02$ & $97.35$ & $98.67$ & $98.67$ & $99.12$ & $99.12$ & $99.12$ & $99.12$ & $99.12$\\ 
 & Hotel 2 & $**$ & $77.88$ & $79.81$ & $77.88$ & $80.77$ & $84.62$ & $83.65$ & $86.54$ & $83.65$ & $84.62$ & $83.65$ \\ 
 & Hotel 3 & $**$ & $83.33$ & $85.19$ & $81.48$ & $85.19$ & $88.89$ & $87.04$ & $85.19$ & $83.33$ & $84.19$ & $85.19$\\ 
 & Study & $**$ & $79.79$ & $84.93$ & $87.33$ & $86.64$ & $88.36$ & $87.33$ & $87.67$ & $87.33$ & $87.67$ & $86.99$\\ 
 & MIT & $**$ & $75.32$ & $75.32$ & $75.32$ & $79.22$ & $79.22$ & $80.52$ & $80.52$ & $77.92$ & $76.62$ & $79.22$\\ 
 \hline \hline
 \multicolumn{2}{c|}{ \emph{Test on train set} } & $79.24$ & $\bm{81.94}$ & $81.56$ & $80.72$ & $81.06$ & $80.73$ & $80.87$ & $80.63$ & $80.48$ & $80.54$ & $80.44$
\end{tabular}
\vspace{2mm}
\caption{Train and test statistics of running \nameshort on \threedmatch~\cite{Zeng17cvpr-3dmatch} with {\bf training and test sets exchanged},~\ie,~we train \nameshort on the smaller test set ($1,623$ pairs), but test \sfcgf on the larger training set ($9,856$ pairs). \nameshort setting: \retrain = \false, \filterlabel = \false. Statistics plotted in Fig.~\ref{fig:supp-fig-line-plots}(b). We see \nameshort demonstrates overfitting while training on the smaller test set: \sfcgf achieves equally good ($91.37\%$) recall on the test set, but only achieves below $82\%$ recall on the training set (while in Tables~\ref{tab:supp-3dmatch-finetune}-\ref{tab:supp-3dmatch-finetune-nofilter} \sfcgf has over $92\%$ recall on the training set). Statistics plotted in Fig.~\ref{fig:supp-fig-line-plots}(b).}
\label{tab:supp-3dmatch-train-test-flip}
\end{table}

{\bf Use a non-robust solver as the teacher}. All the experiments so far showed successes of the teacher-student loop, and the robustness of the \nameshort algorithm to imperfections of both the student and the teacher (noisy geometric pseudo-labels). However, we ask another question: \emph{Can we, intentionally, make \nameshort fail?} Our answer is: yes if we try badly. We performed an experiment running \nameshort on \threedmatch, this time replacing \ransactenk with the non-robust Horn's method~\cite{horn87josa}. We remark that Horn's method is a subroutine of \ransac and in practice nobody would use Horn's method alone in the presence of outlier correspondences. Nevertheless, for the purpose of ablation study, we adopted this pessimistic choice. Again, for \nameshort we used \retrain = \false, \filterlabel = \true with a constant overlap ratio threshold $\eta = 10\%$. Fig.~\ref{fig:supp-fig-line-plots}(c) shows the dynamics. We see that the \plir is always below $20\%$, meaning that 8 out of 10 geometric labels passed to \fcgf training are wrong. In this case, the learned \sfcgf representation keeps getting worse, as shown by the decreasing recalls on both the training and test set. Note that for testing, we actually used \ransactenk as the registration solver to be consistent with other experiments we performed on \threedmatch. However, even with \ransactenk, the test recall drops to below $30\%$. Therefore, this ablation study shows the necessity of a robust teacher for \nameshort to work. Fortunately, we have plenty of robust solvers, as discussed in the main text. So we think this is a strength of \nameshort, rather than a weakness.

\end{document}